\documentclass{article}

\usepackage{microtype}
\usepackage{graphicx}
\usepackage{subfigure}
\usepackage{booktabs} % for professional tables

\usepackage{hyperref}

\usepackage[accepted]{icml2024}

\usepackage{amsmath}
\usepackage{amssymb}
\usepackage{mathtools}
\usepackage{amsthm}

\usepackage[capitalize,noabbrev]{cleveref}

\usepackage{amsthm}
\usepackage{thmtools}
\usepackage{hyperref}

\declaretheorem[name=Corollary,numberwithin=section]{corollary}
\declaretheorem[name=Definition,numberwithin=section]{definition}

\usepackage{tikz}

\DeclareMathOperator{\diag}{diag}

\usepackage{wrapfig}
\usepackage{caption}
\usepackage{amsmath}
\newcommand{\mcalX}{\mathcal{X}}
\newcommand{\mcalA}{\mathcal{A}}

\newcommand{\mcalC}{\mathcal{C}}
\newcommand{\mcalK}{\mathcal{K}}
\definecolor{Gray}{gray}{0.92}
\definecolor{LightBlue}{RGB}{173, 216, 230}
\usepackage{colortbl} % To add color to tables
\usepackage{multirow}
\usepackage[inline]{enumitem}
\usepackage{array}

\newcolumntype{?}{!{\vrule width 1.5pt}}
\newcommand{\revision}[1]{{\color{black}#1}}

\usepackage{amsthm}
\usepackage{thmtools, thm-restate}
\usepackage{chngcntr}

\colorlet{LightBlue}{LightBlue!40!white}

\usepackage[textsize=tiny]{todonotes}

\icmltitlerunning{Subgraphormer: Unifying Subgraph GNNs and Graph Transformers via Graph Products}

\begin{document}

\twocolumn[
\icmltitle{Subgraphormer: Unifying Subgraph GNNs and Graph Transformers \\ via Graph Products}



\icmlsetsymbol{equal}{*}

\begin{icmlauthorlist}
\icmlauthor{Guy Bar-Shalom}{cstech}
\icmlauthor{Beatrice Bevilacqua}{purdue}
\icmlauthor{Haggai Maron}{ecetech,nvidia}
\end{icmlauthorlist}

\icmlaffiliation{cstech}{Department of Computer Science, Technion - Israel Institute of Technology}
\icmlaffiliation{purdue}{Department of Computer Science, Purdue University}
\icmlaffiliation{ecetech}{Department of Electrical \& Computer Engineering, Technion - Israel Institute of Technology}
\icmlaffiliation{nvidia}{NVIDIA Research}

\icmlcorrespondingauthor{Guy Bar-Shalom}{guy.b@cs.technion.ac.il}

\icmlkeywords{Machine Learning, ICML}

\vskip 0.3in
]



\printAffiliationsAndNotice{} % otherwise use the standard text.

\begin{abstract}
In the realm of Graph Neural Networks (GNNs), two exciting research directions have recently emerged: Subgraph GNNs and Graph Transformers.
In this paper, we propose an architecture that integrates both approaches, dubbed \texttt{Subgraphormer}, which combines the enhanced expressive power, message-passing mechanisms, and aggregation schemes from Subgraph GNNs with attention and positional encodings, arguably the most important components in Graph Transformers. Our method is based on an intriguing new connection we reveal between Subgraph GNNs and product graphs, suggesting that Subgraph GNNs can be formulated as Message Passing Neural Networks (MPNNs) operating on a product of the graph with itself. We use this formulation to design our architecture: first, we devise an attention mechanism based on the connectivity of the product graph. Following this, we propose a novel and efficient positional encoding scheme for Subgraph GNNs, which we derive as a positional encoding for the product graph. Our experimental results demonstrate significant performance improvements over both Subgraph GNNs and Graph Transformers on a wide range of datasets.
\end{abstract}

\section{Introduction}
% motivation: expressivity of GNNs
Due to their scalability and elegant architectural design, Message Passing Neural Networks (MPNNs) have become the most popular type of Graph Neural Networks (GNNs). Nonetheless, these architectures face significant limitations, particularly in terms of expressive power~\citep{morris2019weisfeiler, xu2018powerful}, resulting in underwhelming performances in certain tasks. Over the past few years, multiple GNN architectures have been proposed to mitigate these problems \cite{morris2021weisfeiler}.

In this paper, we focus on two different enhanced GNN architecture families: Subgraph GNNs and Graph Transformers. In Subgraph GNNs~\cite{zhang2021nested, cotta2021reconstruction, zhao2022from, bevilacqua2021equivariant,frasca2022understanding,zhang2023rethinking}, a GNN is applied to a bag (multiset) of subgraphs, which is generated from the original graph (for example, the multiset of subgraphs obtained by deleting one node in the original graph). Notably, these architectures are provably more expressive than the traditional message passing algorithms applied directly to the original graph. In parallel, Transformers~\cite{vaswani2017attention} have demonstrated outstanding performance across a wide range of applications, including natural language processing~\cite{vaswani2017attention, kalyan2021ammus, kitaev2020reformer}, computer vision~\cite{khan2022transformers, dosovitskiy2020image, li2022mvitv2}, and, more recently, graph-based tasks~\cite{ying2021transformers, zhang2023rethinking, rampavsek2022recipe}. Two of the most important components of Graph Transformers are their attention mechanism and their positional encoding scheme, which captures the graph structure.
Graph Transformers have achieved impressive empirical performance, demonstrated by their state-of-the-art results on molecular   tasks~\cite{li2022kpgt}. 

This work aims to integrate Subgraph GNNs
and Graph Transformers
into a unified architecture, which we call \texttt{Subgraphormer}, to leverage the benefits of both approaches. In order to define this hybrid approach, we develop two main techniques specifically adapted to subgraphs: (1) A subgraph attention mechanism, dubbed Subgraph Attention Block (SAB); and (2) A subgraph positional encoding scheme, which we call product graph PE. We derive these two mechanisms through a key observation: Subgraph GNNs can be naturally interpreted as MPNNs operating on a new product graph. This product graph is defined on 2-tuples of nodes and encodes both the subgraph structure and their inter-subgraph aggregation schemes. 

Our approach builds on the recent maximally expressive (node-based) Subgraph GNN by \citet{zhang2023complete}. The product graph of this architecture has two main types of edges. First, \emph{internal} edges, that connect each vertex $v$ in subgraph $s$ to neighbors of $v$ within $s$. Second, \emph{external} edges, which are similarly constructed to connect $v$ to copies of $v$ itself across different subgraphs $s'$ when the node generating subgraph $s'$ is a neighbor of the node generating subgraph $s$.  Those are illustrated in \Cref{sec: MES-GNN as an MPNN}. These external edges enable the exchange of information between subgraphs, like the aggregation schemes used in existing Subgraph GNNs, which have been shown to improve performance and expressivity~\cite{zhang2023complete}. Our SAB implements both internal and external attention-based aggregation methods, based on the edges defined above, thereby allowing individual nodes to refine their representations by selectively attending to nodes within the same or from different subgraphs. 
Importantly, as Subgraph GNNs are already computationally expensive, we restrict the attention to the product graph connectivity, similarly to sparse transformers~\cite{choromanski2022block,shirzad2022exphormer}.%

Our Positional Encoding (PE) scheme,  product graph PE, is tailored to the new product graph connectivity that emerges from the architecture defined above. Interestingly, this connectivity stems from the Cartesian product of the original graph with itself, which is a well-known type of graph product. We define our positional encoding as the Laplacian eigenvectors \cite{rampavsek2022recipe,dwivedi2023benchmarking,kreuzer2021rethinking} of the adjacency matrix associated with this Cartesian product. Crucially, although the vertex set of the new graph comprises $n^2$ nodes, in contrast to  $n$ nodes in the original graph, the special structure of the Cartesian product graph allows us to compute positional encodings with time complexity equivalent to calculating positional encodings on the original, smaller graph. We demonstrate that those subgraph-positional encodings offer a performance boost for \texttt{Subgraphormer}, particularly in the case of larger graphs where stochastic subgraph sampling is used. We also describe how product graph PE, and \texttt{Subgraphormer} in general, can be generalized to other higher-order GNNs that operate on $k$-tuples of nodes.

An extensive experimental analysis over a variety of eight different datasets confirms that our architecture leads to significant performance improvements over both Subgraph GNNs and Graph Transformers, yielding outstanding results on multiple datasets, including \textsc{ZINC-12k}~\cite{sterling2015zinc, gomez2018automatic, dwivedi2023benchmarking} and the long-range benchmark \textsc{Peptides-struct}~\cite{dwivedi2022long}. Furthermore, we address the potential computational burden of operating on large bags of subgraphs, demonstrating that the performance of stochastic bag sampling improves when using \texttt{Subgraphormer} compared to other Subgraph GNNs. In particular, we show that our positional encodings play a key role in this case.

To summarize, the main contributions of this paper are (1) \texttt{Subgraphormer}, a novel architecture that combines the strengths of both transformer-based and subgraph-based architectures; (2) A novel observation connecting Subgraph GNNs to product graphs; (3) A novel positional encoding scheme tailored to subgraph methods; and (4) An empirical study demonstrating significant improvements of \texttt{Subgraphormer} compared to existing baselines in both full bag and stochastic bag sampling setups.

\section{Related Work}
\label{sec: Previous work and preliminaries}
\begin{figure}[t]
    \centering
    \begin{minipage}{.25\columnwidth}
        \centering
        \includegraphics[scale=0.2]{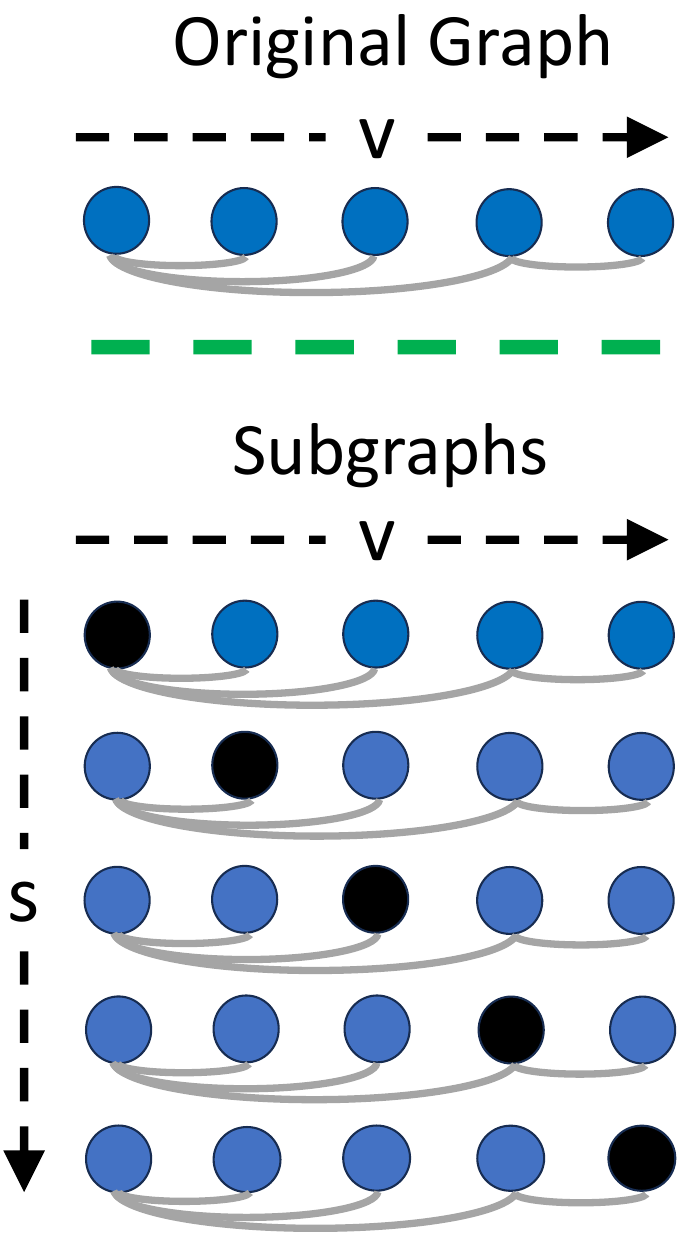}
    \end{minipage}%
    \hfill
    \begin{minipage}{.67\columnwidth}
        \caption{An example of generating subgraphs from the original graph. We denote by $v$ the index used for the node dimension, and by $s$ the one employed for the subgraph dimension. Each subgraph is generated by marking a single node in the original graph, with the marked node represented in black. We refer to the marked node as the root of the corresponding subgraph.
        }
        \label{fig: example of generating subgraphs}
    \end{minipage}
    \vspace{-15pt}
\end{figure}
\textbf{Subgraph GNNs.} 
Subgraph GNNs~\citep{zhang2021nested, cotta2021reconstruction, papp2021dropgnn, bevilacqua2021equivariant, zhao2022from, papp2022theoretical, frasca2022understanding, qian2022ordered, huang2022boosting, zhang2023complete, bevilacqua2023efficient} represent a graph as a collection of subgraphs, obtained by a predefined generation policy. For example, each subgraph can be generated by marking exactly one node in the original graph, an approach commonly referred to as \emph{node marking}~\cite{papp2022theoretical} -- which is also the generation policy on which we focus in this paper. 
Since a node $v$ of the original graph belongs to multiple subgraphs, we use the tuple $(s,v)$ to refer to node $v$ in subgraph $s$. Following the literature, we denote node $s$ in subgraph $s$ as the \emph{root} of the subgraph. An example of this concept is illustrated in~\Cref{fig: example of generating subgraphs}. The success of Subgraph GNNs can be attributed to the enhanced expressive capabilities of these models, which surpass those of MPNNs working directly on the original graph.

\textbf{Graph Transformers.} 
Transformers have achieved notable success in natural language processing~\cite{vaswani2017attention, kalyan2021ammus, kitaev2020reformer} and computer vision~\cite{khan2022transformers, dosovitskiy2020image, li2022mvitv2}, largely owing to their key component, the attention mechanism~\cite{vaswani2017attention}, often coupled with positional encoding schemes. Building on this, the GNN community has introduced Graph Transformers \cite{ying2021transformers,zhang2023rethinking, rampavsek2022recipe, park2022grpe, chen2022structure}, a specialized variant for graph-structured data. Graph Transformers incorporate either global attention~\cite{kreuzer2021rethinking,mialon2021graphit,ying2021transformers}, which operates on fully-connected data, or make use of sparse attention~\cite{shirzad2022exphormer,choromanski2022block,brody2021attentive}, tailored to leverage the graph structure. Given the quadratic nature of Subgraph GNNs, we focus on sparse attention methods.

\section{From Subgraph GNNs to Product Graphs}
\label{sec: From Subgraph GNNs to graph products}
This section serves a dual purpose. Firstly, we demonstrate that the most expressive version of Subgraph GNNs \cite{zhang2023complete}, which serves as a main motivation in our paper, can be implemented by applying an MPNN to a newly constructed graph. 
Second, we explore the relationship between this new graph and the graph Cartesian product operation.
Both of those aspects will be used in the following sections to develop our \texttt{Subgraphormer} architecture. 

\subsection{Subgraph GNNs as MPNNs}
\label{sec: MES-GNN as an MPNN}

\textbf{Notation.} 
Let $G=(A,X)$ denote an undirected graph with node features.\footnote{The consideration of edge features is omitted for simplicity.} The adjacency matrix $A \in \mathbb{R}^{n \times n}$ represents the graph connectivity while the feature matrix, $X \in \mathbb{R}^{n \times d}$, represents the node features. We denote the sets of nodes and edges as $V$ and $E$, respectively, with $\vert V \vert = n$. Additionally, we denote neighbors in $G$ by $\sim_G$, that is $v_1 \sim_G v_2$ denotes that $v_1$ and $v_2$ are neighbors in $G$. We denote the (usually highly sparse) adjacency and feature matrices of all subgraphs of a graph with calligraphic letters as, $\mcalA \in \mathbb{R}^{n^2 \times n^2}$ and $\mcalX \in \mathbb{R}^{n^2 \times d}$, respectively. Notably, we index $\mcalX$ and $\mcalA$ using the tuple $(s,v)$, where $\mcalX(s,v)$ is the feature of node $v$ in subgraph $s$, and $\mcalA( (s,v), (s',v') )$ denotes the edge between node $v$ in subgraph $s$ and node $v'$ in subgraph $s'$. 

\textbf{GNN-SSWL+.} We build upon the construction of the maximally expressive Subgraph GNN proposed in \citet{zhang2023complete}, which employs the \emph{Node Marking} generation policy for generating the subgraphs, and updates the representation of node $(s,v)$ according to the following formula:
\begin{align} 
\label{eq:bohang}
    \mcalX(s,v)^{t+1} &= f^t \Big( \mcalX(s,v)^{t}, \mcalX(v,v)^{t},  \nonumber \\
    & \{\mcalX(s,v')^{t}\}_{v'\sim_G v}, \{\mcalX(s',v)^{t}\}_{s'\sim_G s}, \Big),
\end{align}
for an appropriate parameterized learnable function $f^t$, where the superscript $t$ represents the layer number. Essentially, in each step, a node is updated using its representation, $\mcalX(s,v)^t$, its representation in the subgraph it is root of, $\mcalX(v,v)^{t}$, and two multisets of its neighbors' representations. The first multiset encompasses horizontal neighbors within the same subgraph, represented as $\{\mcalX(s,v')^{t}\}_{v'\sim_G v}$, and the second multiset consists of vertical neighbors between the subgraphs, denoted by $\{\mcalX(s',v)^{t}\}_{s'\sim_G s}$.

\textbf{Product Graph Definition.} In what follows, we construct an adjacency matrix of the form $\mathcal{A}\in \mathbb{R}^{n^2\times n^2}$ to represent each one of the updates in \cref{eq:bohang}; we note that in all cases, this matrix, although in $\mathbb{R}^{n^2\times n^2}$, is extremely sparse. We define the \emph{product graph} to be the heterogeneous graph\footnote{ With a single node type and multiple edge types.} defined by these adjacency matrices, together with a node feature matrix $\mcalX \in \mathbb{R}^{n^2 \times d}$. Each adjacency we construct is visualized for the specific case where the original graph is the one given in~\Cref{fig: example of generating subgraphs}.

\begin{wrapfigure}[6]{l}{0.25\linewidth}
    \vspace{-10pt}
    \centering
    \centering
    \includegraphics[scale=0.2]{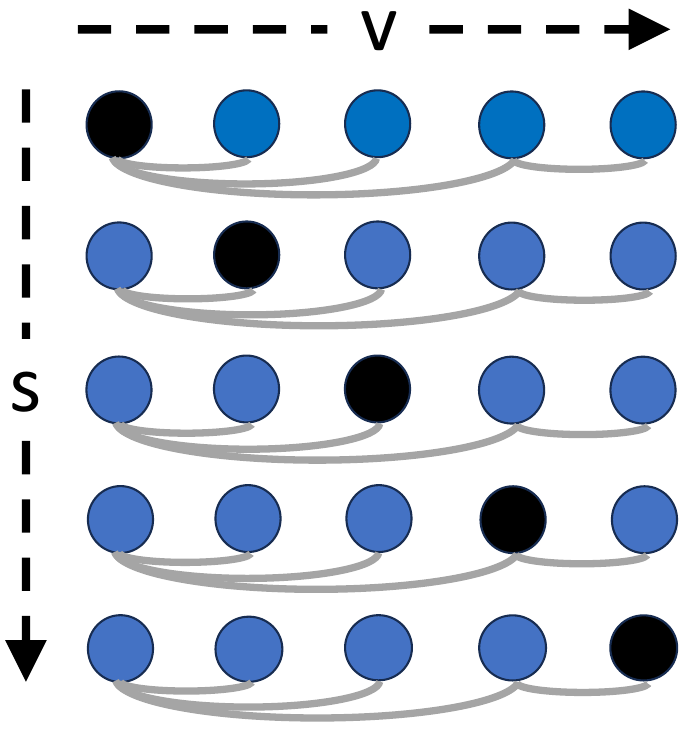} 
    \captionsetup{type=figure}
    %\captionof{figure}{}
    \label{fig: A_G}
\end{wrapfigure}
\textbf{Internal subgraph connectivity.} We start by introducing $\mcalA_G \in \mathbb{R}^{n^2 \times n^2}$ (see inset on the left), which corresponds to $\{\mcalX(s,v')^{t}\}_{v'\sim_G v}$ and maintains the connectivity of each subgraph separately. 
Formally,
\begin{align}
\label{eq: A_G}
\mcalA_G\Big((s, v), (s', v')\Big) = 
    \begin{cases} 
        \delta_{ss'} & \text{if } v \sim_G v'; \\
        0 & \text{otherwise},
    \end{cases} 
\end{align}
where $\delta$ is the Kronecker delta.

\begin{wrapfigure}[7]{l}{0.25\linewidth}
    \vspace{-10pt}
    \centering
    \centering
    \includegraphics[scale=0.2]{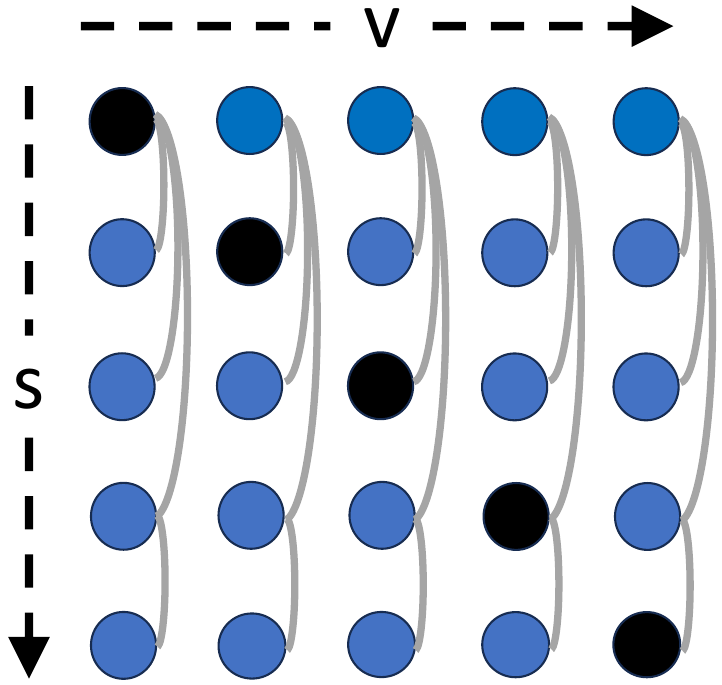} 
    \captionsetup{type=figure}
    \label{fig: A_G^S}
\end{wrapfigure}%
\textbf{External subgraph connectivity.} The adjacency $\mcalA_{G^S} \in \mathbb{R}^{n^2 \times n^2}$ corresponds to $\{\mcalX(s',v)^{t}\}_{s'\sim_G s}$ (see inset) and contains edges across different subgraphs; hence denoted by the superscript $S$ for $G$. In particular, these edges connect the same node in different subgraphs whenever the root nodes of these subgraphs are neighbors in the original graph. 
Formally,
\begin{align}
\label{eq: A_G^S}
 \mcalA_{G^S}\Big((s, v), (s', v')\Big) =
    \begin{cases} 
        \delta_{vv'} & \text{if } s \sim_G s'; \\ 
        0 & \text{otherwise}.
    \end{cases} 
\end{align}

\begin{wrapfigure}[7]{l}{0.25\linewidth}
    \vspace{-10pt}
    \centering
    \centering
    \includegraphics[scale=0.2]{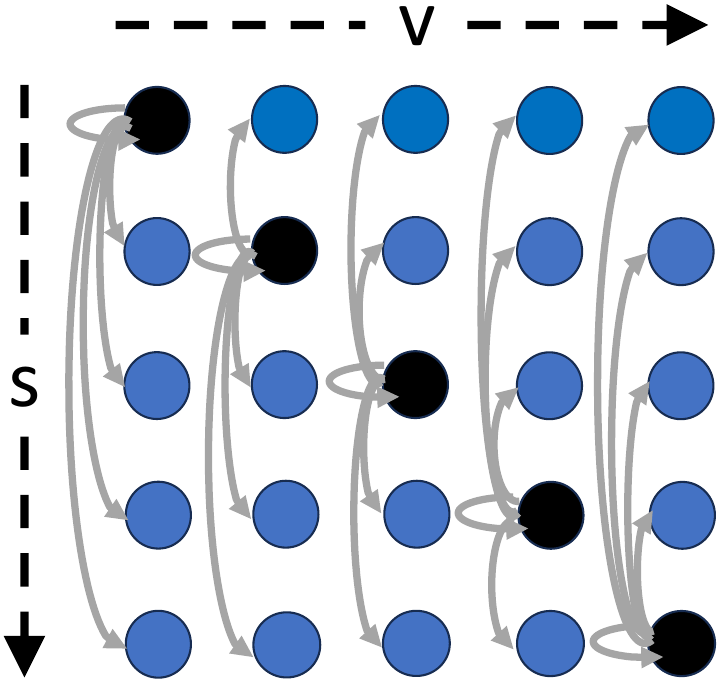} 
    \captionsetup{type=figure}
    %\captionof{figure}{}
    \label{fig: A_point}
\end{wrapfigure}%

\paragraph{Point-wise updates.} The adjacency corresponding to $\mcalX(v,v)^{t}$, denoted as $\mcalA_\text{point}$, allows a node $v$ in subgraph $s$ to receive its root representation.
Those updates are visualized inset by the grey 
arrows pointing from root nodes to other nodes in the same column.
This adjacency matrix can be written as,
\begin{align}
\label{eq: A_point}
     \mcalA_\text{point}\Big((s, v), (s', v')\Big) = 
     \begin{cases} 
        1 & \text{if } s'=v'=v;  \\ % s'=v' \land s'=v;  \\
        0 & \text{otherwise}.
    \end{cases} 
\end{align}
To conclude our discussion on the relationship between GNN-SSWL+ and the product graph we constructed, we state and prove the following proposition showing that GNN-SSWL+ can be simulated by running an MPNN on the product graph:

\begin{restatable}[GNN-SSWL+ as an MPNN on the product graph]{proposition}{SubgraphGNNsasMPNNs}
\label{prop:RGCN}
    Consider a graph $G=(A,X)$. 
    Applying a stacking of RGCN layers \cite{schlichtkrull2018modeling}, interleaved with ReLU activations, on the product graph, as defined via the adjacencies in \Cref{eq: A_G,eq: A_G^S,eq: A_point}, can implement the GNN-SSWL$+$ update in \Cref{eq:bohang}.
\end{restatable}
The proof is in \Cref{app: proofs}.
This idea can be easily extended to other Subgraph GNN types. We note that the general idea of viewing GNNs as MPNNs operating on a new graph was recently proposed by \citet{bause2023maximally,jogl2022weisfeiler,jogl2023expressivity}, and a preliminary discussion on simulating Subgraph GNNs with MPNNs also appeared in \citet{velivckovic2022message}. 

\iffalse
\begin{proof}
This summary provides the essential concept. 

Assuming the input features, $\mcalX \in \mathbb{R}^{n^2 \times d}$, are given as 1-hot vectors, we can uniquely encode the input in \Cref{eq:bohang} by using the adjacencies in \Cref{eq: A_G,eq: A_G^S,eq: A_point}, and multiplying them by the feature matrix $\mcalX$. This is a unique encoding of them given that the feature matrix $\mcalX$ indeed consist of 1-hot vectors. 

Next, given that an RGCN is defined,
    \begin{equation}
            \text{RGCN}^t\Big(\mcalX^t, \{ \mcalA_i \}_{i=1}^M \Big) = \mcalX^t \mathbf{W}_0^t + \sum_{i=1}^M \mcalA_i \mcalX^t \mathbf{W}_i^t,
    \end{equation}
we can use the adjacencies in \Cref{eq: A_G,eq: A_G^S,eq: A_point}, and design the weight matrices as block matrices of dimensions $d \times 4d$, such that they effectively implement the concatenation of those encodings. This concatenation provides a unique descriptor of the input for $f$.

As a final step, we apply two \texttt{MLP}'s, the first implements the parameterized function $f$, Recall \Cref{eq:bohang}, and the second maps back inputs (of a finite dimension) to 1-hot vectors. This is possible due to the Universal Approximation Theorem~\citep{cybenko1989approximation, hornik1991approximation}. We note that those two \texttt{MLP}'s can be combined into a single \texttt{MLP}; this completes the proof.
\end{proof}
\fi

\subsection{Connection to the Graph Cartesian Product}
\label{sec: MES-GNN Main Connectivity: Just a Simple Graph Cartesian Product?}
In the previous section, we described how to build a product graph for simulating a maximally expressive Subgraph GNN. Next we demonstrate that the connectivity of this particular product graph is tightly related to the concept of Graph Cartesian Product \cite{vizing1963cartesian,harary2018graph}.

\textbf{Graph Cartesian Product.} In simple terms, the Cartesian product of two graphs $G_1$ and $G_2$, denoted $G_1 \square G_2$, is a graph whose vertex set is the Cartesian product of $V(G_1)$ and $V(G_2)$, with two vertices $ (u_1, u_2) $ and $ (v_1, v_2) $ being adjacent if either $ u_1 = v_1 $ and $ u_2 $ is adjacent to $ v_2 $ in $ G_2 $, or $ u_2 = v_2 $ and $ u_1 $ is adjacent to $ v_1 $ in $ G_1 $; we denote the adjacency matrix of this new product graph as $\mcalA_{G_1 \square G_2}$. In this paper, we mainly focus on a specific scenario where $G_1 = G_2$, namely, the Cartesian product of a graph with itself.
For a formal definition, we refer the reader to \Cref{def: Graph Cartesian Product} in \Cref{app: Subgraph GNNs as Graph Cartesian Products}. 

Given a graph $G= (A, X)$, the adjacency matrix corresponding to $G \square G$  can be expressed as:
\begin{equation}
\label{eq: A_G_G}
    \mcalA_{G \square G} \triangleq A \otimes I + I \otimes A,
\end{equation}
with $\otimes$ the Kronecker product and $I$ the identity matrix.
We can now establish a direct connection between the product graph we have built for GNN-SSWL+ in the previous subsection and the Cartesian product graph.
\begin{restatable}[Internal and External Connectivities give rise to the Cartesian Product Graph]{proposition}{AdjacenciesAsProducts}
\label{cor: cartesian subgraph equiva}
The edges induced by the internal and external subgraph connectivities, represented by $\mathcal{A}_{G}$ and $\mathcal{A}_{G^S}$ (\Cref{eq: A_G,eq: A_G^S}) represent the connectivity of$\mathcal{A}_{G \square G}$. This implies the relationship:
\begin{equation}
\mathcal{A}_{G \square G} = \mathcal{A}_{G^S} + \mathcal{A}_G.
\end{equation}
In particular, we have $\mathcal{A}_{G^S} = A \otimes I, \mathcal{A}_G = I \otimes A$.
\end{restatable}

The proof is given in \Cref{app: proofs}.

\textbf{General Subgraph GNNs as Product Graphs.}
Different Subgraph GNNs differ in their aggregation schemes. While the aggregations defined by the adjacencies in \Cref{eq: A_G,eq: A_G^S,eq: A_point} are sufficient for maximal expressivity among known subgraph architectures, additional aggregation schemes, that may be important for certain tasks, have been proposed~\citep{frasca2022understanding}. For instance, global updates enable nodes to refine their representations by incorporating information from other nodes, irrespective of their connectivity in the original graph. For example, a node $(s,v)$, can aggregate information from all nodes within the same subgraph, $(s, v')$, for each $v' \in V$. Similarly, it can refine its representation by considering its copies in all subgraphs, $(s', v)$, for each $s' \in V$.
Both of these aggregations can be derived from the adjacency matrices of the Cartesian product $G_c \square G_c$. Here,  $G_c$ denotes the \emph{clique} graph, whose adjacency matrix is $\vec{1}\vec{1}^T - I$. For a comprehensive discussion on the connection between Subgraph GNNs and Cartesian product graphs, refer to \Cref{app: Subgraph GNNs as Graph Cartesian Products}.

\section{\texttt{Subgraphormer}}
\begin{figure*}[t]
    \centering
    \includegraphics[width=0.9\textwidth]{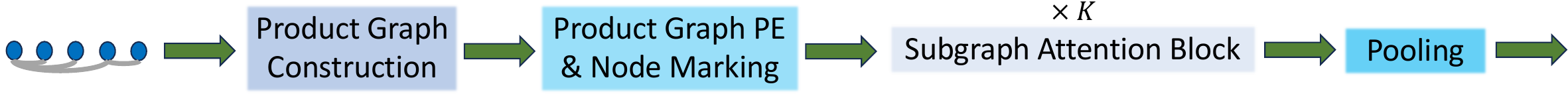}
    % \vspace{-0.5cm}
    \caption{An overview of \texttt{Subgraphormer}. Given an input graph, we first construct the product graph and compute a subgraph-specific positional encoding. Then, we apply a stacking of Subgraph Attention Blocks, followed by a pooling layer to obtain a graph representation. A more comprehensive depiction of this figure can be found in \Cref{fig: architecture_full} of \Cref{app: Subgraphormer Figure}.
    }
    \label{fig: architecture}
    \vspace{-0.2cm}
\end{figure*}

Having established the relationship between Subgraph GNNs and product graphs, we can now use it to define our architecture: \texttt{Subgraphormer}.
\texttt{Subgraphormer} is composed of a subgraph-specific positional encoding layer, node marking, a stacking of Subgraph Attention Blocks (SABs), and a final pooling layer, as depicted in~\Cref{fig: architecture}. In this section, we delve into the specifics of these components. 

\subsection{Subgraph Attention Block}
\label{sec: Subgraph Attention Block}
In this subsection, we introduce the structure of the Subgraph Attention Block (SAB), designed to update the representation of each node $v$ in subgraph $s$, denoted by $\mathcal{X}(s,v)$. In the SAB, the adjacency matrices defined above serve as sparsification functions, effectively biasing the attention mechanism to focus on the neighbors of each node in the product graph.

\textbf{Implementing internal and external subgraph attentions.} We implement two attention mechanisms - dubbed internal and external subgraph attention. These mechanisms are governed by the adjacency matrices $\mathcal{A}_G$ and $\mathcal{A}_{G^S}$, respectively. In general, we employ a sparse attention approach that computes Query ($\mathcal{Q}^t$), Key ($\mathcal{K}^t$), and Value ($\mathcal{V}^t$) transformations from the node features $\mathcal{X}^t \in \mathbb{R}^{n^2 \times d_1}$, according to the following equations:
\begin{equation}
    \mathcal{Q}^t = \mathcal{X}^t\mathbf{W}^t_\mathcal{Q}, \ \mathcal{K}^t = \mathcal{X}^t\mathbf{W}^t_\mathcal{K}, \ \mathcal{V}^t = \mathcal{X}^t\mathbf{W}^t_\mathcal{V},
\end{equation}
where $\mathbf{W}^t_\mathcal{Q}, \mathbf{W}^t_\mathcal{K}, \mathbf{W}^t_\mathcal{V}  \in \mathbb{R}^{d_1 \times d_2}$ are  learned linear projections. 
The attention weights $\alpha^t(\mathcal{Q}^t, \mathcal{K}^t | \mcalA) \in \mathbb{R}^{n^2 \times n^2}$ are computed from $\mathcal{Q}^t$ and $\mathcal{K}^t$ (e.g., by computing their outer product) \emph{only} for the non-zero entries of the adjacency matrix $\mathcal{A} \in \{\mathcal{A}_G, \mathcal{A}_{G^S}\}$.
The resulting node representation are then computed employing the following equations: $$\alpha^t_{\mathcal{A}_G}\!(\mathcal{Q}^t_{\mathcal{A}_G}\!, \mathcal{K}^t_{\mathcal{A}_G}\! | \mathcal{A}_G\!) \mathcal{V}_{\mathcal{A}_G}^t, \alpha^t_{\mathcal{A}_{G^S}}\!(\mathcal{Q}^t_{\mathcal{A}_{G^S}}\!, \mathcal{K}^t_{\mathcal{A}_{G^S}}\! | \mathcal{A}_{G^S}\!) \mathcal{V}_{\mathcal{A}_{G^S}}^t\!,$$
where we use different parameters (denoted by subscripts) to compute the attention coefficients of $\mathcal{A}_{G},\mathcal{A}_{G^S}$. Importantly, our formulation of Subgraph GNNs as MPNNs on a product graph (\Cref{sec: From Subgraph GNNs to graph products}) allows us to leverage existing optimized attention-based MPNNs to implement the \texttt{Subgraphormer} architecture. This avoids the need to build complex custom Subgraph GNN models from scratch. In particular, our implementation follows the Graph Attention Network (GAT) approach proposed by \citet{velivckovic2017graph}, which is implemented efficiently by accounting only for the non-zero entries in $\alpha(Q^t, K^t | \mathcal{A})$. We note that our approach can leverage any sparse attention model. 

\textbf{Implementing pointwise updates.} For pointwise updates we do not use any attention mechanism, since every node has only a single neighbor.
To additionally include the ``self'' (self loop) component from \Cref{eq:bohang} in this update, namely, $\mcalX(s, v)$, we follow~\citet{zhang2023complete} and employ a GIN encoder~\cite{xu2018powerful}, that is,
\begin{equation}
    \label{eq: point agg layer}
    \text{point}^t (\mcalX^t, \mathcal{A}_{\text{Point}}) = \texttt{MLP}^t( (1+\epsilon^t)\mcalX^t + (\mcalA_{\text{point}}\mcalX^t) ).
\end{equation}

\textbf{Final update step.} Putting it all together, a single SAB is defined as an application of an \texttt{MLP} (applied to the feature dimension) to the concatenation of all the updates: 
\begin{align*}
    \text{SAB}^t&(\mathcal{X}^t, \mathcal{A}_G, \mathcal{A}_{G^S}, \mathcal{A}_{\text{Point}}) = \\
     \texttt{MLP}^t &( 
    \alpha^t_{\mathcal{A}_{G}}(\mathcal{Q}^t_{\mathcal{A}_{G}}, \mathcal{K}^t_{\mathcal{A}_{G}} | \mathcal{A}_G) \mathcal{V}^t_{\mathcal{A}_{G}}  \oplus \\
     & \alpha^t_{\mathcal{A}_{G^S}}\!(\mathcal{Q}^t_{\mathcal{A}_{G^S}}\!, \mathcal{K}^t_{\mathcal{A}_{G^S}}\! | \mathcal{A}_{G^S}) \mathcal{V}^t_{\mathcal{A}_{G^S}} \!\!
     \oplus \! \text{point}^t(\mathcal{X}^t\!, \mathcal{A}_{\text{Point}})),
\end{align*}
where $\oplus$ represents feature concatenation.

We remark that one way of looking at a SAB update is by taking inspiration from \Cref{prop:RGCN} and extending the message passing within RGCN to further include different attention weights for each edge type, enabling nodes to attend differently to nodes connected via various edge types.

\textbf{Pooling.} After a stacking of SAB layers, we employ a pooling layer to obtain a graph representation, that is,
\begin{align}
\label{eq: pooling}
    \rho(\mcalX^T) = \texttt{MLP}^T \Big( \sum_{s=1}^n \Big( \sum_{v=1}^n  \mcalX^T(s,v) \Big) \Big),
\end{align}
where $T$ denotes the final layer.

We note that a stacking of SABs followed by an invariant pooling layer as described above, guarantees invariance to node permutations. This is because each block in our model maintains equivariance to node permutations, while the final pooling layer ensures invariance to such permutations. 

\textbf{Complexity.} SAB matches GNN-SSWL+ in time and space complexity, with both at $ \mathcal{O}(\vert V \vert^2 + \vert V \vert \vert E \vert)$ for an original graph of $ \vert V \vert $ nodes and $ \vert E \vert$ edges. See \Cref{app: Complexity} for a detailed complexity analysis of each aggregation.

\subsection{Subgraph Positional Encodings}
\label{sec: Subgraph Positional Encoding}
Positional encodings for graphs have been well studied~\cite{dwivedi2021graph, wang2022equivariant, lim2022sign}, and were shown to provide valuable information for both message-passing based GNNs and graph transformers.

A prominent approach that we follow in this work makes use of the eigendecomposition of the graph Laplacian matrix. More specifically, given a graph $G = (A, X)$, the Laplacian is defined as $L = D - A$, where $D=\text{diag}(A \vec{1})$ is the degree matrix. The positional encoding is calculated by an eigendecomposition of the Laplacian, $L = U^T \Lambda U$, as
\begin{equation}
\label{eq: PE}
\textbf{p}_{v}^{:k} \triangleq [U_{v1}, \ldots, U_{vk}],    
\end{equation}
as the encoding for node $v$, where $k$ is a hyperparameter denoting the number of eigenvectors we consider, $k \leq n$. We note that the eigenvectors are sorted (in ascending order) according to their eigenvalues.

Unfortunately, directly applying this approach to Subgraph GNNs is not straightforward for two reasons: \emph{(1) Unclear adjacency structure:} For the Laplacian to be computed, it is necessary to determine the relevant symmetric adjacency matrix. Specifically, Subgraph GNNs, in general, operate over a collection of subgraphs rather than on a graph structure. It is not clear what type of connectivity we should use for the Laplacian of this structure.
\emph{(2) Efficiency concerns:} The adjacency matrix for Subgraph GNNs is of size ${n^2 \times n^2}$. Consequently, computing its Laplacian's $k$ eigenvectors results in a computational complexity of $\mathcal{O}(k \cdot n^4)$, as opposed to $\mathcal{O}(k \cdot n^2)$\cite{lanczos1950iteration,lehoucq1998arpack,lee2009k} for a graph with $n$ nodes, rendering it impractical in many scenarios.

In the following, we design a positional encoding scheme that addresses these two challenges. 
We start with the adjacency challenge. A natural solution to (1) is to employ the adjacencies $\mathcal{A}_G$, $\mathcal{A}_{G^S}$ in \Cref{eq: A_G,eq: A_G^S}, excluding $\mathcal{A}_{\text{point}}$ since it is asymmetrical and the connectivity it represents is not related to the original graph. These adjacencies correspond to the main connectivity employed by \texttt{Subgraphormer}. Consequently, our proposed positional encoding method will primarily focus on the symmetric connectivity defined by $\mathcal{A}_G$ and $\mathcal{A}_{G^S}$ only. This choice brings us to the challenge (2).

To address point (2), we utilize \Cref{cor: cartesian subgraph equiva} that established a relationship between the adjacency matrices $\mathcal{A}_{G^S}$, $\mathcal{A}_G$, and the Cartesian product graph, expressed as, 
\begin{equation}
\label{eq: A Cartesian product}
\mathcal{A}_{G \square G} =  \mathcal{A}_{G^S} + \mathcal{A}_G = A \otimes I + I \otimes A.
\end{equation}
Surprisingly, \cref{eq: A Cartesian product} simplifies the computation of eigenvectors for the Laplacian $\mathcal{L}_{G \square G}$, by leveraging the eigenvectors and eigenvalues of the original graph $G$, as detailed in the following proposition~\cite{barik2015laplacian}.

\begin{restatable}[Product Graph eigendecomposition]{proposition}{SubgraphormerPE}
\label{prop: main prop}
Consider a graph $G = (A, X)$.\footnote{We assume $A$ has no self loops.}  
The eigenvectors and eigenvalues of the Laplacian matrix of $G \square G$, namely, $\mathcal{L}_{G \square G}$, are $\{(v_i \otimes v_j, \lambda_i + \lambda_j) \}_{i,j = 1}^{n^2}$, where $\{(v_i, \lambda_i) \}_{i=1}^{n}$ are the eigenvectors and eigenvalues of the Laplacian matrix of $G$.
\end{restatable}

\iffalse
\begin{proof} [Proof idea]
As a corollary of the definition of the Cartesian product graph, we have that the adjacency matrix of $G \square G$ is given by $\mcalA_{G \square G} = A \otimes I + I \otimes A$.
Recalling \Cref{eq: A_G,eq: A_G^S,eq: A_{G;G^S}}, we obtain $\mcalA_{G;G^S} = \mcalA_{G \square G}$.

As a result, the graph Laplacian of $G \square G$ is given by 
\begin{equation}
    L_{G \square G} = L \otimes I + I \otimes L,    
\end{equation}
where $L$ is the Laplacian of $G$. This implies that the eigenvectors and eigenvalues of the Laplacian matrix of $G \square G$  are given by $\{(v_i \otimes v_j, \lambda_i + \lambda_j) \}_{i,j = 1}^{n^2}$, where $\{(v_i, \lambda_i) \}_{i=1}^{n}$ are the eigenvectors and eigenvalues of the Laplacian matrix of $G$.
\end{proof}
\fi
We refer to~\Cref{app: proofs} for the proof. 
The implications of this observation are profound. It reveals that, despite the fact that \texttt{Subgraphormer} processes product graphs with $n^2$ nodes, computing the positional encodings only requires an eigendecomposition of the original smaller graph. Specifically, for an undirected graph $G=(A, X)$ with $n$ vertices and its Cartesian product graph $G \square G$, calculating the first $k$ eigenvectors, $k \leq n$, has a time complexity of $\mathcal{O}(k \cdot n^2)$ ~\citep{lanczos1950iteration,lee2009k} -- the same as computing $k$ eigenvectors for the Laplacian of $G$. A proof for this claim  is also given~\Cref{app: proofs}. We refer to our subgraph-positional encodings as \emph{product graph PE}, or simply \emph{PE}.

\textbf{Visualizing product graph PE.} To illustrate the benefits of our product graph PE, we visualize in different colors the entries of the first non-trivial eigenvector of both the original graph (upper) and the product graph (lower). 
\begin{wrapfigure}[10]{l}{0.4\linewidth}
    \vspace{-9pt}
    \centering
        ~\includegraphics[scale=0.3]{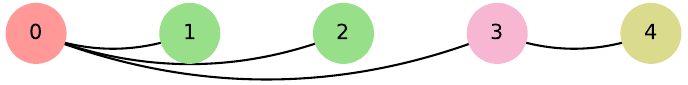}
        \vspace{0.05cm}
        
        \includegraphics[scale=0.33]{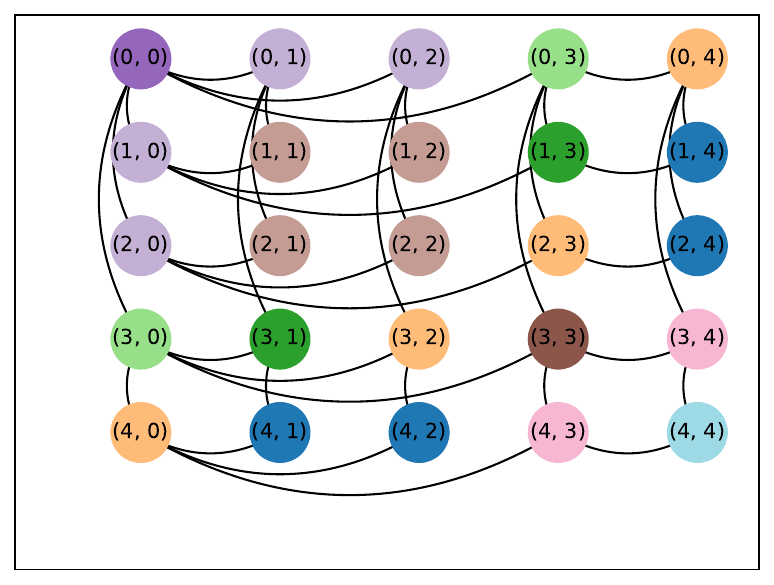}
\end{wrapfigure}
It is readily apparent that the eigenvector on the product graph conveys %significantly
more information than the one on the original graph. This distinction is particularly evident by observing the color diversity, with the product graph featuring ten distinct colors as opposed to the four in the original graph.

\textbf{Concatenation PE.} We briefly introduce another subgraph positional encoding method, \emph{concatenation PE}. Given the Laplacian matrix $L = U^T \Lambda U$, of a graph $G$, the concatenation PE for a node $(s,v)$ is given by the output of an \texttt{MLP} on the concatenated vectors $\mathbf{p}_{s}^{:k} = [U_{s1}, \ldots, U_{sk}]$ and $\mathbf{p}_{v}^{:k} = [U_{v1}, \ldots, U_{vk}]$; we highlight that those two vectors corresponds to the eigenvectors when considering each of the connectivities, $\mcalA_G, \mcalA_{G^S}$, independently.
Notably, concatenation PE can approximate (up to an ordering) the product graph PE. This is because the product graph PE multiplies specific elements from the concatenation PE, a process that can be approximated using a \texttt{MLP} by using the  Universal Approximation Theorem~\cite{hornik1991approximation,cybenko1989approximation}.
However, we use the product graph PE in this study due to its superior performance, as shown in the comparison over the \textsc{ZINC-12k} dataset in \Cref{app: Graph Product PE vs Concatenation PE}.

\textbf{Node Marking.} Motivated by existing Subgraph GNNs, we leverage an additional ``special mark'', defined as \emph{node marking}, which we concatenate to each node's positional encoding. Following the approach of \citet{zhang2023complete}, we hold a lookup table, and assign a learnable embedding to each node, represented as $z_{\text{dist}(s, v)} \in \mathbb{R}^d$.
The embedding assigned to node $(s,v)$ is indexed by the shortest path distance between them in the original graph $G$ -- $\text{dist}(s, v)$.\footnote{We assign a unique mark if the two nodes are unreachable from each other, i.e., $\text{dist}(s, v) = \infty$.}

\subsection{Scaling up \texttt{Subgraphormer}}
\label{sec: Scaling up Subgraphormer}

Given a graph $G$ with $n$ nodes, \texttt{Subgraphormer} operates on a transformed graph with $n^2$ nodes. 
Even though this transformed graph is extremely sparse, this is a well-known drawback of Subgraph GNNs~\citep{qian2022ordered,kong2023mag,bevilacqua2023efficient}
when considering large graphs, since processing the new graph might be infeasible.
\begin{wrapfigure}[13]{l}{0.35\linewidth}
    \vspace{-15pt}
    \centering
    % \vspace{-0.5cm}
    \centering
    \includegraphics[scale=0.2]{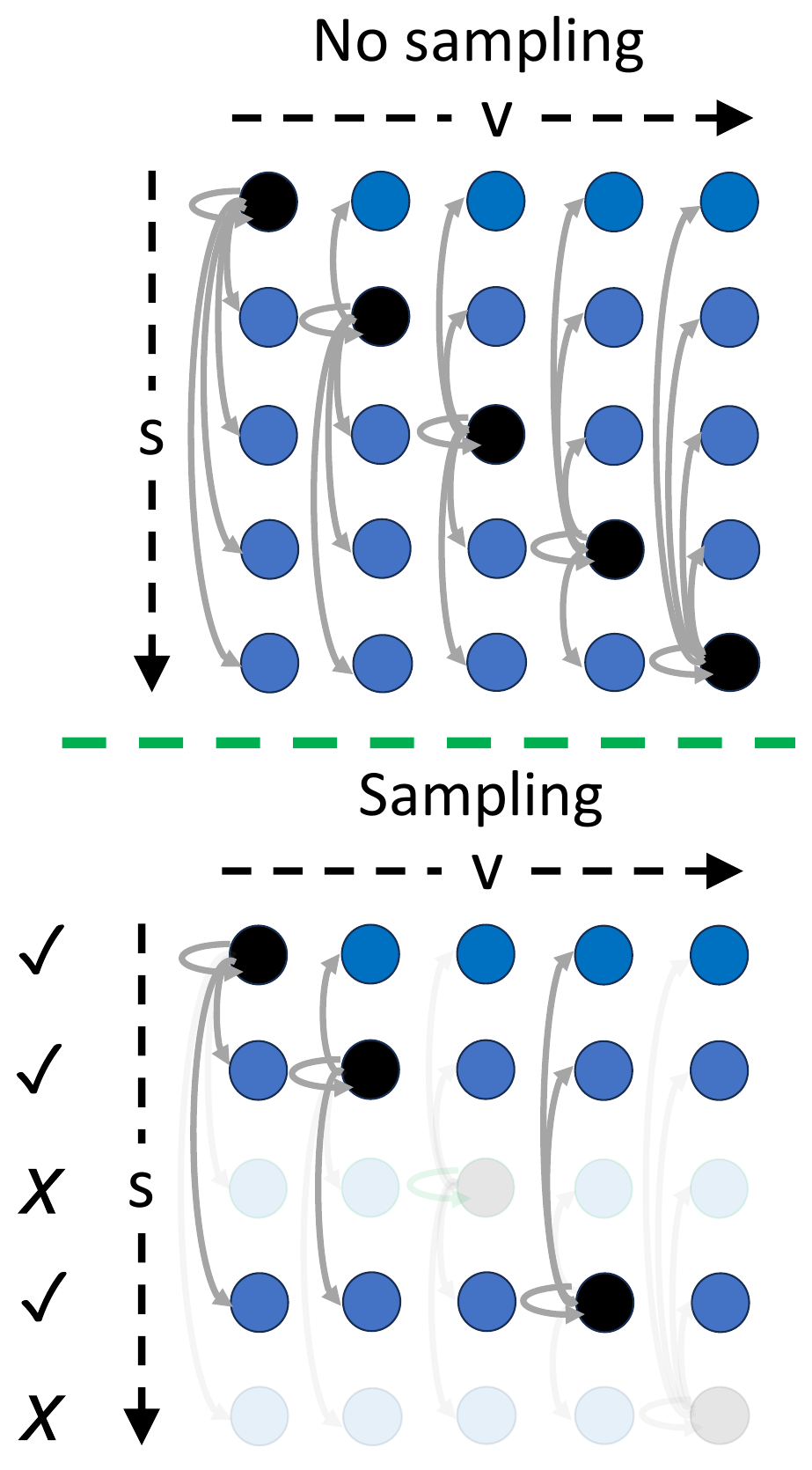} % Adjust the path and the width as needed
    \captionsetup{type=figure}
    % \captionof{figure}{}
    \label{fig: sampling vis}
\end{wrapfigure}%
Thus, in this section we adopt the idea of stochastic sampling~\citep{bevilacqua2021equivariant,cotta2021reconstruction,zhao2022from}. Specifically,  we adapt \texttt{Subgraphormer} to stochastic sampling by adding sparsifications to the adjacencies based on the sampling (see inset). We sample rows (subgraphs) uniformly at random and include an edge if and only if both of its endpoints are part of the sampled graph.

Formally, we introduce a sampling mask, $\mcalA_\text{mask}$, defined as,
\begin{align}
     \mcalA_\text{mask}\Big((s, v), (s', v')\Big) = 
     \begin{cases} 
        1 & \text{if } s~\text{and}~s' ~\text{are sampled}; \nonumber \\
        0 & \text{otherwise}.
    \end{cases}
\end{align}
Given an adjacency matrix $\mcalA$ and a generated sampling mask $\mcalA_{\text{mask}}$, we perform an element-wise logical AND operation between them. Consequently, $\mcalA$ is adjusted to preserve only edges that connect node pairs in sampled subgraphs.

Importantly, our \emph{product graph PE} (\Cref{sec: Subgraph Positional Encoding}) can be computed independently from the sampling process. This feature allows them to retain information from the complete, unsampled graph. As demonstrated in the experimental section, this attribute of \emph{product graph PE} proves advantageous, as it appears to effectively compensate for the information loss incurred due to sampling.

\subsection{Generalization to Higher-Order Graph Networks}
While our current derivation considers product graphs containing 2-tuples of nodes, both the SAB and the product graph PE can be naturally extended to define similar high-order architectures operating on  $k$-tuples of nodes \cite{morris2019weisfeiler,maron2018invariant}. This extension is grounded in our understanding that Subgraph GNNs are connected to the Cartesian product of graphs. Accordingly, this product can be generalized to be applied $k$ times, rather than being limited to $k=2$. This generalization leads to the formation of $k$ distinct adjacency matrices, each representing different ``Internal/External'' aggregations.  To elaborate, the $i$-th adjacency matrix is defined as $\underbrace{I \otimes I \otimes \cdots \otimes I \otimes A \otimes I \otimes \cdots \otimes I \otimes I}_{\text{with } A \text{ exclusively positioned in the } i\text{-th slot}}$. Similarly, the positional encodings can also be extended to any $k$-tuple, with the eigenvectors obtained by the tensor product of $k$ eigenvectors of $G$, and the eigenvalues by the sum of the respective eigenvalues. More specifically, for each sequence of indices $\{i_1, i_2, \ldots, i_k\}$, where $i_j \in \{1, 2, \ldots, n\}$ for each $j$, there exists an eigenvector-eigenvalue pair given by $(v_{i_1} \otimes v_{i_2} \otimes \ldots \otimes v_{i_k}, \lambda_{i_1} + \lambda_{i_2} + \ldots + \lambda_{i_k})$. Here, $\{(v_i, \lambda_i) \}_{i=1}^{n}$ represent the eigenvectors and eigenvalues of the Laplacian matrix of $G$. \Cref{app: Extending Subgraphormer to Any k-tuple} presents a detailed derivation of SABs and PEs for $k$-tuples, for any given $k$. 

Additionally, we note that the high-order positional encoding discussed here may be useful for other high-order GNNs.

\section{Experiments}
We conducted an extensive set of experiments over eight different datasets to answer the following questions:
\begin{enumerate*}[label=(Q\arabic*)]
    \item \emph{Can \texttt{Subgraphormer} outperform Subgraph GNNs and Graph Transformers in real-world benchmarks?}
    \item \emph{Does the attention mechanism prove advantageous?}
    \item \emph{Does our product graph positional encoding scheme boost performance, and in which contexts?} 
    \item \emph{How well does \texttt{Subgraphormer} perform when using stochastic sampling?}
    \item \emph{Can \texttt{Subgraphormer} offer competitive performance on long-range benchmarks when compared to Graph Transformers?}
\end{enumerate*}

All experiments were conducted using PyTorch, and the corresponding code is publicly available for replicating our results\footnote{\url{https://github.com/BarSGuy/Subgraphormer}}.

In the following, we present our main results and refer to \Cref{app: Extended Experimental Section} for additional experiments.

\begin{table}[t]
    \centering
    \scriptsize
    \caption{On the \textsc{ZINC} datasets, \texttt{Subgraphormer} outperforms \colorbox{Gray}{Graph Transformers} and \colorbox{LightBlue}{Subgraph GNNs}. The top three results are reported as \textcolor{blue}{\textbf{First}}, \textcolor{red}{\textbf{Second}}, and \textcolor{orange}{\textbf{Third}}.}
    \label{tab: ZINC}
    \resizebox{\columnwidth}{!}{%
    \addtolength{\tabcolsep}{-0.4em}
    \begin{tabular}{l|c|c|c}
    \toprule
    \multirow{2}*{Model $\downarrow$ / Dataset $\rightarrow$} & \multirow{2}*{Param.}& \textsc{ZINC-12k} & \textsc{ZINC-Full} \\
    & &  (MAE $\downarrow$) &  (MAE $\downarrow$) \\
    \midrule
    GSN~\cite{bouritsas2022improving} & $500k$ & $0.101$\tiny $\pm 0.010$ & - \\
    CIN (small)~\cite{bodnar2021weisfeiler} & $100k$ & $0.094$\tiny $\pm 0.004$& $0.044$\tiny $\pm 0.003$ \\
    GIN~\cite{xu2018powerful} & $500k$ & $0.163$\tiny $\pm 0.004$ & -  \\
    PPGN++(6)~\cite{puny2023equivariant} & $500k$ & $0.071$\tiny $\pm 0.001$& $\textcolor{blue}{\mathbf{0.020}}$\tiny $\pm 0.001$ \\
    \midrule
    \rowcolor{Gray} SAN~\cite{kreuzer2021rethinking} & $509k$ & $0.139$\tiny $\pm 0.006$ & - \\
    \rowcolor{Gray} URPE~\cite{luo2022your} & $492k$ & $0.086$\tiny $\pm 0.007$ & $0.028$\tiny $\pm 0.002$ \\
    \rowcolor{Gray} GPS~\cite{rampavsek2022recipe} &  $424k$ & $\textcolor{orange}{\mathbf{0.070}}$\tiny $\pm 0.004$ & - \\
    \rowcolor{Gray} Graphormer~\cite{ying2021transformers} & $489k$ & $0.122$\tiny $\pm 0.006$ & $0.052$\tiny $\pm 0.005$ \\
    \rowcolor{Gray} Graphormer-GD~\cite{zhang2023rethinking} & $503k$ & $0.081$\tiny $\pm 0.009$ & $0.025$\tiny $\pm 0.004$ \\
    \rowcolor{Gray} K-Subgraph SAT~\cite{chen2022structure} & $523k$ & $0.094$\tiny $\pm 0.008$ & -  \\
    \midrule 
    \rowcolor{LightBlue} NGNN~\cite{zhang2021nested} & $500k$ & $0.111$\tiny $\pm 0.003$ & $0.029$\tiny $\pm 0.001$ \\
    \rowcolor{LightBlue} DS-GNN~\cite{bevilacqua2021equivariant} & $100k$ & $0.116$\tiny $\pm 0.009$ & -  \\
    \rowcolor{LightBlue} DSS-GNN~\cite{bevilacqua2021equivariant} & $100k$ & $0.102$\tiny $\pm 0.003$ & $0.029$\tiny $\pm 0.003$ \\
    \rowcolor{LightBlue} GNN-AK~\cite{zhao2022from} & $500k$ & $0.105$\tiny $\pm 0.010$ & -  \\
    \rowcolor{LightBlue} GNN-AK+~\cite{zhao2022from} & $500k$ & $0.091$\tiny $\pm 0.002$ & -  \\
    \rowcolor{LightBlue} SUN~\cite{frasca2022understanding} & $526k$ & $0.083$\tiny $\pm 0.003$ & $0.024$\tiny $\pm 0.003$  \\
    \rowcolor{LightBlue} OSAN~\cite{qian2022ordered} & $500k$ & $0.154$\tiny $\pm 0.008$ & -  \\
    \rowcolor{LightBlue} DS-GNN~\cite{bevilacqua2023efficient} & $500k$ & $0.087$\tiny $\pm 0.003$ & - \\
    \rowcolor{LightBlue} GNN-SSWL~\cite{zhang2023complete} & $274k$ & $0.082$\tiny $\pm 0.003$ & $0.026$\tiny $\pm 0.001$ \\
    \rowcolor{LightBlue} GNN-SSWL+~\cite{zhang2023complete} & $387k$ & $\textcolor{orange}{\mathbf{0.070}}$\tiny $\pm 0.005$ & $\textcolor{red}{\mathbf{0.022}}$\tiny $\pm 0.001$ \\
    \midrule
    \texttt{Subgraphormer} & $293k$ & $\textcolor{red}{\mathbf{0.067}}$\tiny $\pm 0.007$ & $\textcolor{blue}{\mathbf{0.020}}$\tiny $\pm 0.002$ \\
    \texttt{Subgraphormer + PE} & $293k$ & $\textcolor{blue}{\mathbf{0.063}}$\tiny $\pm 0.001$ & $\textcolor{orange}{\mathbf{0.023}}$\tiny $\pm 0.001$ \\
    \bottomrule
    \end{tabular}
    }
\end{table}

\begin{table}[t!]
\centering
\scriptsize
\caption{On the \textsc{OGB} datasets, \texttt{Subgraphormer} outperforms \colorbox{Gray}{Graph Transformers} and \colorbox{LightBlue}{Subgraph GNNs}. The top three results are reported as \textcolor{blue}{\textbf{First}}, \textcolor{red}{\textbf{Second}}, and \textcolor{orange}{\textbf{Third}}.}
\label{tab: ogbg}
\resizebox{\columnwidth}{!}{%
\addtolength{\tabcolsep}{-0.55em}
\begin{tabular}{l|c|c|c}
\toprule
\multirow{2}*{Model $\downarrow$ / Dataset $\rightarrow$} & \textsc{molhiv} & \textsc{molbace}  & \textsc{molesol}  \\
& (ROC-AUC $\uparrow$) & (ROC-AUC $\uparrow$) & (RMSE $\downarrow$) \\
\midrule
GCN~\cite{kipf2016semi} & $76.06$\tiny $\pm 0.97$ & $\textcolor{orange}{\mathbf{79.15}}$\tiny $\pm 1.44$ & $1.114$\tiny $\pm 0.036$ \\
GIN~\cite{xu2018powerful} & $75.58$\tiny $\pm 1.40$ & $72.97$\tiny $\pm 4.00$ & $1.173$\tiny $\pm 0.057$ \\
PNA~\cite{corso2020principal} & $79.05$\tiny $\pm 1.32$ & - & - \\
GSN~\cite{bouritsas2022improving} & $\textcolor{red}{\mathbf{80.39}}$\tiny $\pm 0.90$ & - & - \\
CIN~\cite{bodnar2021weisfeiler} & $\textcolor{blue}{\mathbf{80.94}}$\tiny $\pm 0.57$ & - & - \\
% ... (Continue for other rows)
\midrule
\rowcolor{Gray} SAN~\cite{kreuzer2021rethinking} & $77.85$\tiny $\pm 0.24$ & - &  - \\
\rowcolor{Gray} GPS~\cite{rampavsek2022recipe} & $78.80$\tiny $\pm 1.01$  & - &  - \\
% ... (Continue for other light blue rows)
\midrule
\rowcolor{LightBlue} NGNN~\cite{zhang2021nested} & $78.34$\tiny $\pm 1.86$ & - &  - \\
%\rowcolor{LightBlue} Policy-Learn~\cite{bevilacqua2023efficient} & $79.13$\tiny $\pm 0.60$ & $78.41$\tiny $\pm 1.94$ &  $\textcolor{red}{\mathbf{0.847$\tiny $\pm 0.015}}$ \\
\rowcolor{LightBlue} Recon. GNN~\cite{cotta2021reconstruction} & $76.32$\tiny $\pm 1.40$ & $1.026$\tiny $\pm 0.033$ &  - \\
\rowcolor{LightBlue} DS-GNN~\cite{bevilacqua2021equivariant} & $77.40$\tiny $\pm 2.19$ & - &  - \\
\rowcolor{LightBlue} DSS-GNN~\cite{bevilacqua2021equivariant} & $76.78$\tiny $\pm 1.66$ & - &  - \\
\rowcolor{LightBlue} GNN-AK+~\cite{zhao2022from} & $79.61$\tiny $\pm 1.19$ & - &  - \\
\rowcolor{LightBlue} SUN~\cite{frasca2022understanding} & $80.03$\tiny $\pm 0.55$ & - &  - \\
\rowcolor{LightBlue} OSAN~\cite{qian2022ordered} & - & $72.30$\tiny $\pm 6.60$ &  $0.980$\tiny $\pm 0.086$ \\
\rowcolor{LightBlue} DS-GNN~\cite{bevilacqua2023efficient} & $76.54$\tiny $\pm 1.37$ & $78.41$\tiny $\pm 1.94$ & $0.847$\tiny $\pm 0.015$ \\
% ... (Continue for other gray rows)
\rowcolor{LightBlue} GNN-SSWL+~\cite{zhang2023complete} & $79.58$\tiny $\pm 0.35$ & $\textcolor{red}{\mathbf{82.70}}$\tiny $\pm 1.80$ &  $\textcolor{orange}{\mathbf{0.837}}$\tiny $\pm 0.019$ \\
\midrule
\texttt{Subgraphormer} & $\textcolor{orange}{\mathbf{80.38}}$\tiny $\pm 1.92$ & $\textcolor{orange}{\mathbf{81.62}}$\tiny $\pm 3.55$ &  $\textcolor{red}{\mathbf{0.832}}$\tiny $\pm 0.043$ \\
\texttt{Subgraphormer + PE} &  $79.48$\tiny $\pm 1.28$ & $\textcolor{blue}{\mathbf{84.35}}$\tiny $\pm 0.65$ &  $\textcolor{blue}{\mathbf{0.826}}$\tiny $\pm 0.010$ \\
\bottomrule
\end{tabular}
}
\end{table}

\begin{table*}[t]
\centering
\scriptsize
\caption{Ablation Study over our product graph PE. For each dataset, the best result for each sampling ratio is in \textbf{Bold}. }
\label{tab:transposed_sampling_ablation_study}
% \resizebox{\textwidth}{!}{%
\addtolength{\tabcolsep}{-0.4em}
\begin{tabular}{lr|cc|cc|cc|cc}
\toprule
\multicolumn{2}{c|}{\multirow{2}{*}{Dataset  $\downarrow$ / Sampling Ratio  $\rightarrow$}} & \multicolumn{8}{c}{\texttt{Subgraphormer} / \texttt{Subgraphormer + PE}} \\ 
\cmidrule(lr){3-10}
& & \multicolumn{2}{c|}{100\%} & \multicolumn{2}{c|}{50\%} & \multicolumn{2}{c|}{20\%} & \multicolumn{2}{c}{5\%}  \\ 
\midrule
\textsc{ZINC-12k} & (MAE $\downarrow$)       & $0.067$\tiny $\pm0.007$ & $\mathbf{0.063}$\tiny $\pm0.001$  & $0.084$\tiny $\pm0.002$ & $\mathbf{0.084}$\tiny $\pm0.002$                    & $0.121$\tiny $\pm0.007$ & $\mathbf{0.120}$\tiny $\pm0.002$ & $0.200$\tiny $\pm0.017$ & $\mathbf{0.175}$\tiny $\pm0.006$ \\
% \revision{\textsc{ZINC-Full}} & \revision{(MAE $\downarrow$)} \\
\textsc{molhiv} & (ROC-AUC $\uparrow$)  & $\mathbf{80.38}$\tiny $\pm1.92\phantom{0}$ & $79.48$\tiny $\pm1.28$ & $\mathbf{79.66}$\tiny $\pm0.79\phantom{0}$ & $79.61$\tiny $\pm1.30$           & $76.48$\tiny $\pm2.38\phantom{0}$ & $\mathbf{76.68}$\tiny $\pm1.07\phantom{0}$ & $70.87$\tiny $\pm0.90\phantom{0}$ & $\mathbf{71.02}$\tiny $\pm0.79\phantom{0}$ \\
\textsc{molbace} & (ROC-AUC $\uparrow$) & $81.62$\tiny $\pm3.55\phantom{0} $ & $\mathbf{84.35}$\tiny $\pm0.65\phantom{0}$ & $79.49$\tiny $\pm2.38\phantom{0}$ & $\mathbf{83.82}$\tiny $\pm2.62\phantom{0}$  & $75.27$\tiny $\pm5.63\phantom{0}$ & $\mathbf{78.77}$\tiny $\pm4.10\phantom{0}$ & $63.05$\tiny $\pm8.70\phantom{0}$ & $\mathbf{67.73}$\tiny $\pm5.50\phantom{0}$ \\ 
\textsc{molesol} & (RMSE $\downarrow$) & $0.832$\tiny $\pm0.043$ & $\mathbf{0.826}$\tiny $\pm0.010$ & $0.829$\tiny $\pm0.013$ & $\mathbf{0.812}$\tiny $\pm0.001$  & $1.093$\tiny $\pm0.009$ & $\mathbf{1.041}$\tiny $\pm0.030$ & $\mathbf{1.266}$\tiny $\pm0.019$ & $1.270$\tiny $\pm0.007$ \\
\bottomrule
\end{tabular}
%}
\end{table*}

\begin{table}[ht]
    \centering
    \caption{Results on the \textsc{Peptides} datasets demonstrate the effectiveness of the stochastic variant of \texttt{Subgraphormer} in tasks where Graph Transformers excel. 
    \colorbox{Gray}{Graph Transformers} and \colorbox{LightBlue}{Subgraph GNNs} are highlighted in gray and light blue. The top three results are reported as  \textcolor{blue}{\textbf{First}}, \textcolor{red}{\textbf{Second}}, and \textcolor{orange}{\textbf{Third}}.}
    \label{tab: LRGB}
    \scriptsize
    \resizebox{\columnwidth}{!}{%
    \addtolength{\tabcolsep}{-0.55em}
    \begin{tabular}{l r |c|c} 
    \toprule
    \multicolumn{2}{l|}{\multirow{2}*{Model $\downarrow$ / Dataset $\rightarrow$}} & \textsc{Peptides-func} & \textsc{Peptides-struct} \\
    & &(AP $\uparrow$) & (MAE $\downarrow$) \\ 
    \midrule
   \multicolumn{2}{l|}{GCN~\cite{kipf2016semi}} & $0.5930$\tiny $\pm0.0023$ & $0.3496$\tiny $\pm0.0013$ \\
    \multicolumn{2}{l|}{GIN~\cite{xu2018powerful}} & $0.5498$\tiny $\pm0.0079$ & $0.3547$\tiny $\pm0.0045$ \\
    \multicolumn{2}{l|}{GatedGCN~\cite{bresson2017residual}} & $0.5864$\tiny $\pm0.0077$ & $0.3420$\tiny $\pm0.0013$ \\
    \multicolumn{2}{l|}{GatedGCN+RWSE~\cite{dwivedi2022long}} & $0.6069$\tiny $\pm0.0035$ & $0.3357$\tiny $\pm0.0006$ \\
    \midrule
    \rowcolor{Gray} \multicolumn{2}{l|}{Transf.+LapPE~\cite{dwivedi2022long}} & $0.6326$\tiny $\pm0.0126$ & $0.2529$\tiny $\pm0.0016$ \\
    \rowcolor{Gray} \multicolumn{2}{l|}{SAN+LapPE~\cite{kreuzer2021rethinking}} & $0.6384$\tiny $\pm0.0121$ & $0.2683$\tiny $\pm0.0043$ \\
    \rowcolor{Gray} \multicolumn{2}{l|}{SAN+RWSE~\cite{kreuzer2021rethinking}} & $\textcolor{red}{\mathbf{0.6439}}$\tiny $\pm0.0075$ & $0.2545$\tiny $\pm0.0012$ \\
    \rowcolor{Gray} \multicolumn{2}{l|}{GPS~\cite{rampavsek2022recipe}} & $\textcolor{blue}{\mathbf{0.6535}}$\tiny $\pm0.0041$ & $\textcolor{orange}{\mathbf{0.2500}}$\tiny $\pm0.0005$ \\
    \midrule
    \rowcolor{LightBlue} GNN-SSWL+~\cite{zhang2023complete} & 30\%& $0.5847$\tiny $\pm0.0050$ & $0.2570$\tiny $\pm0.0006$ \\
    \midrule
    \texttt{Subgraphormer} & 30\%& $\textcolor{orange}{\mathbf{0.6415}}$\tiny $\pm0.0052$ & $\textcolor{red}{\mathbf{0.2494}}$\tiny $\pm0.0020$ \\
    \texttt{Subgraphormer + PE} &30\%& $0.6373$\tiny $\pm0.0110$ & $\textcolor{blue}{\mathbf{0.2475}}$\tiny $\pm0.0007$ \\
    \bottomrule
    \end{tabular}
    }
\end{table}
\paragraph{Real-world data.} On the \textsc{ZINC-12k} and \textsc{ZINC-Full} datasets~\citep{sterling2015zinc, gomez2018automatic, dwivedi2023benchmarking}, \texttt{Subgraphormer} outperforms both Subgraph GNNs and Graph Transformers by a significant margin, while using a smaller number of parameters (\Cref{tab: ZINC}). 
Importantly, the attention mechanism proves advantageous, as \texttt{Subgraphormer} always outperforms GNN-SSWL+~\cite{zhang2023complete}, on which it builds on by integrating attention blocks.
Remarkably, our product graph PE boosts the performance of \texttt{Subgraphormer} on \textsc{ZINC-12k}, where it is the top-performing method. A similar trend is observed on the OGB datasets~\citep{hu2020open} in \Cref{tab: ogbg}. Notably, on the \textsc{molbace} dataset, the inclusion of the PE improves the performance by $3\%$. Additional results on \textsc{Alchemy-12k} are reported in \Cref{app: Extended Results}.

\paragraph{Efficacy of stochastic sampling on long-range data.} We experimented on the \textsc{Peptides-func} and \textsc{Peptides-struct} datasets~\citep{dwivedi2022long}, evaluating the ability of \texttt{Subgraphormer} to scale to larger graphs and to capture long-range dependencies, a notoriously hard task for MPNNs~\citep{dwivedi2022long}. We considered a sampling ratio of $30\%$, which allows us to run \texttt{Subgraphormer} though typical full-bag Subgraph GNNs cannot be applied. The results are summarized in \Cref{tab: LRGB}. Despite the reduced number of subgraphs, \texttt{Subgraphormer} achieves the best performance on \textsc{Peptides-struct}, outperforming all Graph Transformers, which excel on tasks of this kind, and offer comparable results over \textsc{Peptides-func}.

\paragraph{Ablation Study: Product Graph PE.}
We assess the impact of our product graph PE on the performance of the stochastic variant of \texttt{Subgraphormer} under common sampling rates~\cite{bevilacqua2021equivariant} ($100\%$, $50\%$, $20\%$, $5\%$) across four datasets. \Cref{tab:transposed_sampling_ablation_study} shows that the inclusion of our product graph PE improves the results of \texttt{Subgraphormer} in 13 out of 16 dataset-sampling combinations.
Its effectiveness is especially noticeable when combined with $20\%$ and $5\%$ ratios, indicating its role in compensating information loss due to low sampling rates.

\paragraph{Discussion.}
In what follows, we address research questions \textbf{Q1} to \textbf{Q5}.

\textbf{A1.} \Cref{tab: ZINC,tab: ogbg} show \texttt{Subgraphormer} outperforming all transformer-based and subgraph-based baselines. A similar trend is observed on the \textsc{Alchemy-12k} dataset in \Cref{tab: alchemy} in \Cref{app: Extended Results}.

\textbf{A2.} \Cref{tab: ZINC,tab: LRGB} clearly demonstrate the importance of the attention, as \texttt{Subgraphormer} always outperforms GNN-SSWL+, from which it is built by adding SABs. On the OGB (\Cref{tab: ogbg}) the attention mechanism proves particularly beneficial when coupled with the product graph PE. 

\textbf{A3.} \Cref{tab:transposed_sampling_ablation_study} demonstrates the effectiveness of the \texttt{PE}, especially when considering lower sampling regimes.

\textbf{A4.} The stochastic variant of \texttt{Subgraphormer} proves particularly advantageous in settings where full-bag Subgraph GNNs cannot be otherwise applied, such as in \cref{tab: LRGB}. Additional results can be found in \Cref{tab: zinc-sampling,tab: molbace-sampling,tab: molesol-sampling} in \Cref{app: Extended Experimental Section}.

\textbf{A5.} \Cref{tab: LRGB} demonstrates the ability of \texttt{Subgraphormer} to capture long-range dependencies, as it outperforms all baselines on the \textsc{Peptides-struct} dataset, and performs comparably over \textsc{Peptides-func}.

\section{Conclusions}
In this work, we introduce \texttt{Subgraphormer}, a novel architecture that merges the capabilities of Subgraph GNNs and Graph Transformers. Building upon our observation that views Subgraph GNNs as MPNNs operating on product graphs, we propose: (1) a subgraph-based attention mechanism, which can be implemented by leveraging optimized sparse attention blocks operating on the product graph; and (2) a positional encoding scheme capturing the connectivity of the product graph computed in a time complexity equivalent to computing positional encodings on the original, small graph. Empirically, we demonstrate that \texttt{Subgraphormer} outperforms Subgraph GNNs and Graph Transformers over a wide range of datasets. We also investigated the capabilities of the stochastic variant of our approach, demonstrating impressive performance on long-range datasets, where Graph Transformers excel, and full-bag Subgraph GNNs cannot be applied.

\section*{Acknowledgements}
HM is the Robert J. Shillman Fellow and is supported by the Israel Science Foundation through a personal grant (ISF 264/23) and an equipment grant (ISF 532/23).

\section*{Impact Statement} This paper presents work whose goal is to advance the field of Machine Learning. There are many potential societal consequences of our work, none of which we feel must be specifically highlighted here.

%\clearpage
\bibliography{icml_2024}

\begin{thebibliography}{73}
\providecommand{\natexlab}[1]{#1}
\providecommand{\url}[1]{\texttt{#1}}
\expandafter\ifx\csname urlstyle\endcsname\relax
  \providecommand{\doi}[1]{doi: #1}\else
  \providecommand{\doi}{doi: \begingroup \urlstyle{rm}\Url}\fi

\bibitem[Barik et~al.(2015)Barik, Bapat, and Pati]{barik2015laplacian}
Sasmita Barik, Ravindra~B Bapat, and Sukanta Pati.
\newblock On the laplacian spectra of product graphs.
\newblock \emph{Applicable Analysis and Discrete Mathematics}, pages 39--58,
  2015.

\bibitem[Bause et~al.(2023)Bause, Jogl, Indri, Drucks, Penz, Kriege,
  G{\"a}rtner, Welke, and Thiessen]{bause2023maximally}
Franka Bause, Fabian Jogl, Patrick Indri, Tamara Drucks, David Penz, Nils
  Kriege, Thomas G{\"a}rtner, Pascal Welke, and Maximilian Thiessen.
\newblock Maximally expressive gnns for outerplanar graphs.
\newblock In \emph{NeurIPS 2023 Workshop: New Frontiers in Graph Learning},
  2023.

\bibitem[Bevilacqua et~al.(2022)Bevilacqua, Frasca, Lim, Srinivasan, Cai,
  Balamurugan, Bronstein, and Maron]{bevilacqua2021equivariant}
Beatrice Bevilacqua, Fabrizio Frasca, Derek Lim, Balasubramaniam Srinivasan,
  Chen Cai, Gopinath Balamurugan, Michael~M Bronstein, and Haggai Maron.
\newblock Equivariant subgraph aggregation networks.
\newblock \emph{International Conference on Learning Representations}, 2022.

\bibitem[Bevilacqua et~al.(2023)Bevilacqua, Eliasof, Meirom, Ribeiro, and
  Maron]{bevilacqua2023efficient}
Beatrice Bevilacqua, Moshe Eliasof, Eli Meirom, Bruno Ribeiro, and Haggai
  Maron.
\newblock Efficient subgraph gnns by learning effective selection policies.
\newblock \emph{International Conference on Learning Representations}, 2023.

\bibitem[Biewald(2020)]{wandb}
Lukas Biewald.
\newblock Experiment tracking with weights and biases, 2020.
\newblock URL \url{https://www.wandb.com/}.
\newblock Software available from wandb.com.

\bibitem[Bodnar et~al.(2021)Bodnar, Frasca, Otter, Wang, Lio, Montufar, and
  Bronstein]{bodnar2021weisfeiler}
Cristian Bodnar, Fabrizio Frasca, Nina Otter, Yuguang Wang, Pietro Lio, Guido~F
  Montufar, and Michael Bronstein.
\newblock Weisfeiler and lehman go cellular: Cw networks.
\newblock \emph{Advances in Neural Information Processing Systems},
  34:\penalty0 2625--2640, 2021.

\bibitem[Bouritsas et~al.(2022)Bouritsas, Frasca, Zafeiriou, and
  Bronstein]{bouritsas2022improving}
Giorgos Bouritsas, Fabrizio Frasca, Stefanos Zafeiriou, and Michael~M
  Bronstein.
\newblock Improving graph neural network expressivity via subgraph isomorphism
  counting.
\newblock \emph{IEEE Transactions on Pattern Analysis and Machine
  Intelligence}, 45\penalty0 (1):\penalty0 657--668, 2022.

\bibitem[Bresson and Laurent(2017)]{bresson2017residual}
Xavier Bresson and Thomas Laurent.
\newblock Residual gated graph convnets.
\newblock \emph{arXiv preprint arXiv:1711.07553}, 2017.

\bibitem[Brody et~al.(2021)Brody, Alon, and Yahav]{brody2021attentive}
Shaked Brody, Uri Alon, and Eran Yahav.
\newblock How attentive are graph attention networks?
\newblock \emph{arXiv preprint arXiv:2105.14491}, 2021.

\bibitem[Chen et~al.(2022)Chen, O’Bray, and Borgwardt]{chen2022structure}
Dexiong Chen, Leslie O’Bray, and Karsten Borgwardt.
\newblock Structure-aware transformer for graph representation learning.
\newblock In \emph{International Conference on Machine Learning}, pages
  3469--3489. PMLR, 2022.

\bibitem[Chen et~al.(2019)Chen, Chen, Hsieh, Lee, Liao, Liao, Liu, Qiu, Sun,
  Tang, et~al.]{chen2019alchemy}
Guangyong Chen, Pengfei Chen, Chang-Yu Hsieh, Chee-Kong Lee, Benben Liao,
  Renjie Liao, Weiwen Liu, Jiezhong Qiu, Qiming Sun, Jie Tang, et~al.
\newblock Alchemy: A quantum chemistry dataset for benchmarking ai models.
\newblock \emph{arXiv preprint arXiv:1906.09427}, 2019.

\bibitem[Choromanski et~al.(2022)Choromanski, Lin, Chen, Zhang, Sehanobish,
  Likhosherstov, Parker-Holder, Sarlos, Weller, and
  Weingarten]{choromanski2022block}
Krzysztof Choromanski, Han Lin, Haoxian Chen, Tianyi Zhang, Arijit Sehanobish,
  Valerii Likhosherstov, Jack Parker-Holder, Tamas Sarlos, Adrian Weller, and
  Thomas Weingarten.
\newblock From block-toeplitz matrices to differential equations on graphs:
  towards a general theory for scalable masked transformers.
\newblock In \emph{International Conference on Machine Learning}, pages
  3962--3983. PMLR, 2022.

\bibitem[Corso et~al.(2020)Corso, Cavalleri, Beaini, Li{\`o}, and
  Veli{\v{c}}kovi{\'c}]{corso2020principal}
Gabriele Corso, Luca Cavalleri, Dominique Beaini, Pietro Li{\`o}, and Petar
  Veli{\v{c}}kovi{\'c}.
\newblock Principal neighbourhood aggregation for graph nets.
\newblock \emph{Advances in Neural Information Processing Systems},
  33:\penalty0 13260--13271, 2020.

\bibitem[Cotta et~al.(2021)Cotta, Morris, and Ribeiro]{cotta2021reconstruction}
Leonardo Cotta, Christopher Morris, and Bruno Ribeiro.
\newblock Reconstruction for powerful graph representations.
\newblock In \emph{Advances in Neural Information Processing Systems},
  volume~34, 2021.

\bibitem[Cybenko(1989)]{cybenko1989approximation}
George Cybenko.
\newblock Approximation by superpositions of a sigmoidal function.
\newblock \emph{Mathematics of control, signals and systems}, 2\penalty0
  (4):\penalty0 303--314, 1989.

\bibitem[Dosovitskiy et~al.(2020)Dosovitskiy, Beyer, Kolesnikov, Weissenborn,
  Zhai, Unterthiner, Dehghani, Minderer, Heigold, Gelly,
  et~al.]{dosovitskiy2020image}
Alexey Dosovitskiy, Lucas Beyer, Alexander Kolesnikov, Dirk Weissenborn,
  Xiaohua Zhai, Thomas Unterthiner, Mostafa Dehghani, Matthias Minderer, Georg
  Heigold, Sylvain Gelly, et~al.
\newblock An image is worth 16x16 words: Transformers for image recognition at
  scale.
\newblock \emph{arXiv preprint arXiv:2010.11929}, 2020.

\bibitem[Dupty and Lee(2022)]{dupty2022graph}
Mohammed~Haroon Dupty and Wee~Sun Lee.
\newblock Graph representation learning with individualization and refinement.
\newblock \emph{arXiv preprint arXiv:2203.09141}, 2022.

\bibitem[Dupty et~al.(2021)Dupty, Dong, and Lee]{dupty2021pf}
Mohammed~Haroon Dupty, Yanfei Dong, and Wee~Sun Lee.
\newblock Pf-gnn: Differentiable particle filtering based approximation of
  universal graph representations.
\newblock In \emph{International Conference on Learning Representations}, 2021.

\bibitem[Dwivedi et~al.(2021)Dwivedi, Luu, Laurent, Bengio, and
  Bresson]{dwivedi2021graph}
Vijay~Prakash Dwivedi, Anh~Tuan Luu, Thomas Laurent, Yoshua Bengio, and Xavier
  Bresson.
\newblock Graph neural networks with learnable structural and positional
  representations.
\newblock \emph{International Conference on Learning Representations}, 2021.

\bibitem[Dwivedi et~al.(2022)Dwivedi, Ramp{\'a}{\v{s}}ek, Galkin, Parviz, Wolf,
  Luu, and Beaini]{dwivedi2022long}
Vijay~Prakash Dwivedi, Ladislav Ramp{\'a}{\v{s}}ek, Michael Galkin, Ali Parviz,
  Guy Wolf, Anh~Tuan Luu, and Dominique Beaini.
\newblock Long range graph benchmark.
\newblock \emph{Advances in Neural Information Processing Systems},
  35:\penalty0 22326--22340, 2022.

\bibitem[Dwivedi et~al.(2023)Dwivedi, Joshi, Luu, Laurent, Bengio, and
  Bresson]{dwivedi2023benchmarking}
Vijay~Prakash Dwivedi, Chaitanya~K Joshi, Anh~Tuan Luu, Thomas Laurent, Yoshua
  Bengio, and Xavier Bresson.
\newblock Benchmarking graph neural networks.
\newblock \emph{Journal of Machine Learning Research}, 24\penalty0
  (43):\penalty0 1--48, 2023.

\bibitem[Fey and Lenssen(2019)]{fey2019fast}
Matthias Fey and Jan~Eric Lenssen.
\newblock Fast graph representation learning with pytorch geometric.
\newblock \emph{arXiv preprint arXiv:1903.02428}, 2019.

\bibitem[Frasca et~al.(2022)Frasca, Bevilacqua, Bronstein, and
  Maron]{frasca2022understanding}
Fabrizio Frasca, Beatrice Bevilacqua, Michael Bronstein, and Haggai Maron.
\newblock Understanding and extending subgraph gnns by rethinking their
  symmetries.
\newblock \emph{Advances in Neural Information Processing Systems},
  35:\penalty0 31376--31390, 2022.

\bibitem[G{\'o}mez-Bombarelli et~al.(2018)G{\'o}mez-Bombarelli, Wei, Duvenaud,
  Hern{\'a}ndez-Lobato, S{\'a}nchez-Lengeling, Sheberla, Aguilera-Iparraguirre,
  Hirzel, Adams, and Aspuru-Guzik]{gomez2018automatic}
Rafael G{\'o}mez-Bombarelli, Jennifer~N Wei, David Duvenaud, Jos{\'e}~Miguel
  Hern{\'a}ndez-Lobato, Benjam{\'\i}n S{\'a}nchez-Lengeling, Dennis Sheberla,
  Jorge Aguilera-Iparraguirre, Timothy~D Hirzel, Ryan~P Adams, and Al{\'a}n
  Aspuru-Guzik.
\newblock Automatic chemical design using a data-driven continuous
  representation of molecules.
\newblock \emph{ACS central science}, 4\penalty0 (2):\penalty0 268--276, 2018.

\bibitem[Harary(2018)]{harary2018graph}
Frank Harary.
\newblock \emph{Graph Theory (on Demand Printing Of 02787)}.
\newblock CRC Press, 2018.

\bibitem[Hornik(1991)]{hornik1991approximation}
Kurt Hornik.
\newblock Approximation capabilities of multilayer feedforward networks.
\newblock \emph{Neural networks}, 4\penalty0 (2):\penalty0 251--257, 1991.

\bibitem[Hu et~al.(2020)Hu, Fey, Zitnik, Dong, Ren, Liu, Catasta, and
  Leskovec]{hu2020open}
Weihua Hu, Matthias Fey, Marinka Zitnik, Yuxiao Dong, Hongyu Ren, Bowen Liu,
  Michele Catasta, and Jure Leskovec.
\newblock Open graph benchmark: Datasets for machine learning on graphs.
\newblock \emph{Advances in neural information processing systems},
  33:\penalty0 22118--22133, 2020.

\bibitem[Huang et~al.(2022)Huang, Peng, Ma, and Zhang]{huang2022boosting}
Yinan Huang, Xingang Peng, Jianzhu Ma, and Muhan Zhang.
\newblock Boosting the cycle counting power of graph neural networks with
  i$^{2}$-gnns.
\newblock In \emph{The Eleventh International Conference on Learning
  Representations}, 2022.

\bibitem[Jogl et~al.(2022)Jogl, Thiessen, and G{\"a}rtner]{jogl2022weisfeiler}
Fabian Jogl, Maximilian Thiessen, and Thomas G{\"a}rtner.
\newblock Weisfeiler and leman return with graph transformations.
\newblock In \emph{18th International Workshop on Mining and Learning with
  Graphs}, 2022.

\bibitem[Jogl et~al.(2023)Jogl, Thiessen, and
  G{\"a}rtner]{jogl2023expressivity}
Fabian Jogl, Maximilian Thiessen, and Thomas G{\"a}rtner.
\newblock Expressivity-preserving gnn simulation.
\newblock In \emph{Thirty-seventh Conference on Neural Information Processing
  Systems}, 2023.

\bibitem[Kalyan et~al.(2021)Kalyan, Rajasekharan, and
  Sangeetha]{kalyan2021ammus}
Katikapalli~Subramanyam Kalyan, Ajit Rajasekharan, and Sivanesan Sangeetha.
\newblock Ammus: A survey of transformer-based pretrained models in natural
  language processing.
\newblock \emph{arXiv preprint arXiv:2108.05542}, 2021.

\bibitem[Khan et~al.(2022)Khan, Naseer, Hayat, Zamir, Khan, and
  Shah]{khan2022transformers}
Salman Khan, Muzammal Naseer, Munawar Hayat, Syed~Waqas Zamir, Fahad~Shahbaz
  Khan, and Mubarak Shah.
\newblock Transformers in vision: A survey.
\newblock \emph{ACM computing surveys (CSUR)}, 54\penalty0 (10s):\penalty0
  1--41, 2022.

\bibitem[Kipf and Welling(2016)]{kipf2016semi}
Thomas~N Kipf and Max Welling.
\newblock Semi-supervised classification with graph convolutional networks.
\newblock \emph{International Conference on Learning Representations}, 2016.

\bibitem[Kitaev et~al.(2020)Kitaev, Kaiser, and Levskaya]{kitaev2020reformer}
Nikita Kitaev, {\L}ukasz Kaiser, and Anselm Levskaya.
\newblock Reformer: The efficient transformer.
\newblock \emph{arXiv preprint arXiv:2001.04451}, 2020.

\bibitem[Kong et~al.(2023)Kong, Feng, Liu, Tao, Chen, and Zhang]{kong2023mag}
Lecheng Kong, Jiarui Feng, Hao Liu, Dacheng Tao, Yixin Chen, and Muhan Zhang.
\newblock Mag-gnn: Reinforcement learning boosted graph neural network.
\newblock In \emph{Thirty-seventh Conference on Neural Information Processing
  Systems}, 2023.

\bibitem[Kreuzer et~al.(2021)Kreuzer, Beaini, Hamilton, L{\'e}tourneau, and
  Tossou]{kreuzer2021rethinking}
Devin Kreuzer, Dominique Beaini, Will Hamilton, Vincent L{\'e}tourneau, and
  Prudencio Tossou.
\newblock Rethinking graph transformers with spectral attention.
\newblock \emph{Advances in Neural Information Processing Systems},
  34:\penalty0 21618--21629, 2021.

\bibitem[Kwon et~al.(2021)Kwon, Kim, Park, and Choi]{kwon2021asam}
Jungmin Kwon, Jeongseop Kim, Hyunseo Park, and In~Kwon Choi.
\newblock Asam: Adaptive sharpness-aware minimization for scale-invariant
  learning of deep neural networks.
\newblock In \emph{International Conference on Machine Learning}, 2021.

\bibitem[Lanczos(1950)]{lanczos1950iteration}
Cornelius Lanczos.
\newblock An iteration method for the solution of the eigenvalue problem of
  linear differential and integral operators1.
\newblock \emph{Journal of Research of the National Bureau of Standards},
  45\penalty0 (4), 1950.

\bibitem[Lee et~al.(2009)Lee, Balakrishnan, Koh, and Jiao]{lee2009k}
Jongwon Lee, Venkataramanan Balakrishnan, Cheng-Kok Koh, and Dan Jiao.
\newblock From o (k 2 n) to o (n): A fast complex-valued eigenvalue solver for
  large-scale on-chip interconnect analysis.
\newblock In \emph{2009 IEEE MTT-S International Microwave Symposium Digest},
  pages 181--184. IEEE, 2009.

\bibitem[Lehoucq et~al.(1998)Lehoucq, Sorensen, and Yang]{lehoucq1998arpack}
Richard~B Lehoucq, Danny~C Sorensen, and Chao Yang.
\newblock \emph{ARPACK users' guide: solution of large-scale eigenvalue
  problems with implicitly restarted Arnoldi methods}.
\newblock SIAM, 1998.

\bibitem[Li et~al.(2022{\natexlab{a}})Li, Zhao, and Zeng]{li2022kpgt}
Han Li, Dan Zhao, and Jianyang Zeng.
\newblock Kpgt: knowledge-guided pre-training of graph transformer for
  molecular property prediction.
\newblock In \emph{Proceedings of the 28th ACM SIGKDD Conference on Knowledge
  Discovery and Data Mining}, pages 857--867, 2022{\natexlab{a}}.

\bibitem[Li et~al.(2022{\natexlab{b}})Li, Wu, Fan, Mangalam, Xiong, Malik, and
  Feichtenhofer]{li2022mvitv2}
Yanghao Li, Chao-Yuan Wu, Haoqi Fan, Karttikeya Mangalam, Bo~Xiong, Jitendra
  Malik, and Christoph Feichtenhofer.
\newblock Mvitv2: Improved multiscale vision transformers for classification
  and detection.
\newblock In \emph{Proceedings of the IEEE/CVF Conference on Computer Vision
  and Pattern Recognition}, pages 4804--4814, 2022{\natexlab{b}}.

\bibitem[Lim et~al.(2022)Lim, Robinson, Zhao, Smidt, Sra, Maron, and
  Jegelka]{lim2022sign}
Derek Lim, Joshua Robinson, Lingxiao Zhao, Tess Smidt, Suvrit Sra, Haggai
  Maron, and Stefanie Jegelka.
\newblock Sign and basis invariant networks for spectral graph representation
  learning.
\newblock \emph{The Eleventh International Conference on Learning
  Representations}, 2022.

\bibitem[Luo et~al.(2022)Luo, Li, Zheng, Liu, Wang, and He]{luo2022your}
Shengjie Luo, Shanda Li, Shuxin Zheng, Tie-Yan Liu, Liwei Wang, and Di~He.
\newblock Your transformer may not be as powerful as you expect.
\newblock \emph{Advances in Neural Information Processing Systems},
  35:\penalty0 4301--4315, 2022.

\bibitem[Maron et~al.(2018)Maron, Ben-Hamu, Shamir, and
  Lipman]{maron2018invariant}
Haggai Maron, Heli Ben-Hamu, Nadav Shamir, and Yaron Lipman.
\newblock Invariant and equivariant graph networks.
\newblock \emph{arXiv preprint arXiv:1812.09902}, 2018.

\bibitem[Mialon et~al.(2021)Mialon, Chen, Selosse, and
  Mairal]{mialon2021graphit}
Gr{\'e}goire Mialon, Dexiong Chen, Margot Selosse, and Julien Mairal.
\newblock Graphit: Encoding graph structure in transformers.
\newblock \emph{arXiv preprint arXiv:2106.05667}, 2021.

\bibitem[Morris et~al.(2019)Morris, Ritzert, Fey, Hamilton, Lenssen, Rattan,
  and Grohe]{morris2019weisfeiler}
Christopher Morris, Martin Ritzert, Matthias Fey, William~L Hamilton, Jan~Eric
  Lenssen, Gaurav Rattan, and Martin Grohe.
\newblock Weisfeiler and leman go neural: Higher-order graph neural networks.
\newblock In \emph{Proceedings of the AAAI conference on artificial
  intelligence}, volume~33, pages 4602--4609, 2019.

\bibitem[Morris et~al.(2020)Morris, Rattan, and Mutzel]{morris2020weisfeiler}
Christopher Morris, Gaurav Rattan, and Petra Mutzel.
\newblock Weisfeiler and leman go sparse: Towards scalable higher-order graph
  embeddings.
\newblock \emph{Advances in Neural Information Processing Systems},
  33:\penalty0 21824--21840, 2020.

\bibitem[Morris et~al.(2021)Morris, Lipman, Maron, Rieck, Kriege, Grohe, Fey,
  and Borgwardt]{morris2021weisfeiler}
Christopher Morris, Yaron Lipman, Haggai Maron, Bastian Rieck, Nils~M Kriege,
  Martin Grohe, Matthias Fey, and Karsten Borgwardt.
\newblock Weisfeiler and leman go machine learning: The story so far.
\newblock \emph{arXiv preprint arXiv:2112.09992}, 2021.

\bibitem[Morris et~al.(2022)Morris, Rattan, Kiefer, and
  Ravanbakhsh]{morris2022speqnets}
Christopher Morris, Gaurav Rattan, Sandra Kiefer, and Siamak Ravanbakhsh.
\newblock Speqnets: Sparsity-aware permutation-equivariant graph networks.
\newblock In \emph{International Conference on Machine Learning}, pages
  16017--16042. PMLR, 2022.

\bibitem[Papp and Wattenhofer(2022)]{papp2022theoretical}
P{\'a}l~Andr{\'a}s Papp and Roger Wattenhofer.
\newblock A theoretical comparison of graph neural network extensions.
\newblock In \emph{International Conference on Machine Learning}, pages
  17323--17345. PMLR, 2022.

\bibitem[Papp et~al.(2021)Papp, Martinkus, Faber, and
  Wattenhofer]{papp2021dropgnn}
P{\'a}l~Andr{\'a}s Papp, Karolis Martinkus, Lukas Faber, and Roger Wattenhofer.
\newblock Dropgnn: Random dropouts increase the expressiveness of graph neural
  networks.
\newblock \emph{Advances in Neural Information Processing Systems},
  34:\penalty0 21997--22009, 2021.

\bibitem[Park et~al.(2022)Park, Chang, Lee, Kim, and Hwang]{park2022grpe}
Wonpyo Park, Woonggi Chang, Donggeon Lee, Juntae Kim, and Seung-won Hwang.
\newblock Grpe: Relative positional encoding for graph transformer.
\newblock \emph{arXiv preprint arXiv:2201.12787}, 2022.

\bibitem[Paszke et~al.(2019)Paszke, Gross, Massa, Lerer, Bradbury, Chanan,
  Killeen, Lin, Gimelshein, Antiga, et~al.]{paszke2019pytorch}
Adam Paszke, Sam Gross, Francisco Massa, Adam Lerer, James Bradbury, Gregory
  Chanan, Trevor Killeen, Zeming Lin, Natalia Gimelshein, Luca Antiga, et~al.
\newblock Pytorch: An imperative style, high-performance deep learning library.
\newblock \emph{Advances in neural information processing systems}, 32, 2019.

\bibitem[Puny et~al.(2023)Puny, Lim, Kiani, Maron, and
  Lipman]{puny2023equivariant}
Omri Puny, Derek Lim, Bobak Kiani, Haggai Maron, and Yaron Lipman.
\newblock Equivariant polynomials for graph neural networks.
\newblock In \emph{International Conference on Machine Learning}, pages
  28191--28222. PMLR, 2023.

\bibitem[Qian et~al.(2022)Qian, Rattan, Geerts, Niepert, and
  Morris]{qian2022ordered}
Chendi Qian, Gaurav Rattan, Floris Geerts, Mathias Niepert, and Christopher
  Morris.
\newblock Ordered subgraph aggregation networks.
\newblock \emph{Advances in Neural Information Processing Systems},
  35:\penalty0 21030--21045, 2022.

\bibitem[Ramp{\'a}{\v{s}}ek et~al.(2022)Ramp{\'a}{\v{s}}ek, Galkin, Dwivedi,
  Luu, Wolf, and Beaini]{rampavsek2022recipe}
Ladislav Ramp{\'a}{\v{s}}ek, Michael Galkin, Vijay~Prakash Dwivedi, Anh~Tuan
  Luu, Guy Wolf, and Dominique Beaini.
\newblock Recipe for a general, powerful, scalable graph transformer.
\newblock \emph{Advances in Neural Information Processing Systems},
  35:\penalty0 14501--14515, 2022.

\bibitem[Schlichtkrull et~al.(2018)Schlichtkrull, Kipf, Bloem, Van Den~Berg,
  Titov, and Welling]{schlichtkrull2018modeling}
Michael Schlichtkrull, Thomas~N Kipf, Peter Bloem, Rianne Van Den~Berg, Ivan
  Titov, and Max Welling.
\newblock Modeling relational data with graph convolutional networks.
\newblock In \emph{The Semantic Web: 15th International Conference, ESWC 2018,
  Heraklion, Crete, Greece, June 3--7, 2018, Proceedings 15}, pages 593--607.
  Springer, 2018.

\bibitem[Shirzad et~al.(2023)Shirzad, Velingker, Venkatachalam, Sutherland, and
  Sinop]{shirzad2022exphormer}
Hamed Shirzad, Ameya Velingker, Balaji Venkatachalam, Danica~J Sutherland, and
  Ali~Kemal Sinop.
\newblock Exphormer: Scaling graph transformers with expander graphs.
\newblock \emph{International Conference on Machine Learning}, 2023.

\bibitem[Sterling and Irwin(2015)]{sterling2015zinc}
Teague Sterling and John~J Irwin.
\newblock Zinc 15--ligand discovery for everyone.
\newblock \emph{Journal of chemical information and modeling}, 55\penalty0
  (11):\penalty0 2324--2337, 2015.

\bibitem[Vaswani et~al.(2017)Vaswani, Shazeer, Parmar, Uszkoreit, Jones, Gomez,
  Kaiser, and Polosukhin]{vaswani2017attention}
Ashish Vaswani, Noam Shazeer, Niki Parmar, Jakob Uszkoreit, Llion Jones,
  Aidan~N Gomez, {\L}ukasz Kaiser, and Illia Polosukhin.
\newblock Attention is all you need.
\newblock \emph{Advances in neural information processing systems}, 30, 2017.

\bibitem[Veli{\v{c}}kovi{\'c}(2022)]{velivckovic2022message}
Petar Veli{\v{c}}kovi{\'c}.
\newblock Message passing all the way up.
\newblock \emph{arXiv preprint arXiv:2202.11097}, 2022.

\bibitem[Veli{\v{c}}kovi{\'c} et~al.(2017)Veli{\v{c}}kovi{\'c}, Cucurull,
  Casanova, Romero, Lio, and Bengio]{velivckovic2017graph}
Petar Veli{\v{c}}kovi{\'c}, Guillem Cucurull, Arantxa Casanova, Adriana Romero,
  Pietro Lio, and Yoshua Bengio.
\newblock Graph attention networks.
\newblock \emph{International Conference on Learning Representations}, 2017.

\bibitem[Virtanen et~al.(2020)Virtanen, Gommers, Oliphant, Haberland, Reddy,
  Cournapeau, Burovski, Peterson, Weckesser, Bright, {van der Walt}, Brett,
  Wilson, Millman, Mayorov, Nelson, Jones, Kern, Larson, Carey, Polat, Feng,
  Moore, {VanderPlas}, Laxalde, Perktold, Cimrman, Henriksen, Quintero, Harris,
  Archibald, Ribeiro, Pedregosa, {van Mulbregt}, and {SciPy 1.0
  Contributors}]{2020SciPy-NMeth}
Pauli Virtanen, Ralf Gommers, Travis~E. Oliphant, Matt Haberland, Tyler Reddy,
  David Cournapeau, Evgeni Burovski, Pearu Peterson, Warren Weckesser, Jonathan
  Bright, St{\'e}fan~J. {van der Walt}, Matthew Brett, Joshua Wilson, K.~Jarrod
  Millman, Nikolay Mayorov, Andrew R.~J. Nelson, Eric Jones, Robert Kern, Eric
  Larson, C~J Carey, {\.I}lhan Polat, Yu~Feng, Eric~W. Moore, Jake
  {VanderPlas}, Denis Laxalde, Josef Perktold, Robert Cimrman, Ian Henriksen,
  E.~A. Quintero, Charles~R. Harris, Anne~M. Archibald, Ant{\^o}nio~H. Ribeiro,
  Fabian Pedregosa, Paul {van Mulbregt}, and {SciPy 1.0 Contributors}.
\newblock {{SciPy} 1.0: Fundamental Algorithms for Scientific Computing in
  Python}.
\newblock \emph{Nature Methods}, 17:\penalty0 261--272, 2020.
\newblock \doi{10.1038/s41592-019-0686-2}.

\bibitem[Vizing(1963)]{vizing1963cartesian}
Vladim~G Vizing.
\newblock The cartesian product of graphs.
\newblock \emph{Vycisl. Sistemy}, 9\penalty0 (30-43):\penalty0 33, 1963.

\bibitem[Wang et~al.(2022)Wang, Yin, Zhang, and Li]{wang2022equivariant}
Haorui Wang, Haoteng Yin, Muhan Zhang, and Pan Li.
\newblock Equivariant and stable positional encoding for more powerful graph
  neural networks.
\newblock \emph{International Conference on Learning Representations}, 2022.

\bibitem[Xu et~al.(2018)Xu, Hu, Leskovec, and Jegelka]{xu2018powerful}
Keyulu Xu, Weihua Hu, Jure Leskovec, and Stefanie Jegelka.
\newblock How powerful are graph neural networks?
\newblock \emph{International Conference on Learning Representations}, 2018.

\bibitem[Ying et~al.(2021)Ying, Cai, Luo, Zheng, Ke, He, Shen, and
  Liu]{ying2021transformers}
Chengxuan Ying, Tianle Cai, Shengjie Luo, Shuxin Zheng, Guolin Ke, Di~He,
  Yanming Shen, and Tie-Yan Liu.
\newblock Do transformers really perform badly for graph representation?
\newblock \emph{Advances in Neural Information Processing Systems},
  34:\penalty0 28877--28888, 2021.

\bibitem[Yun et~al.(2019)Yun, Sra, and Jadbabaie]{yun2019small}
Chulhee Yun, Suvrit Sra, and Ali Jadbabaie.
\newblock Small relu networks are powerful memorizers: a tight analysis of
  memorization capacity.
\newblock \emph{Advances in Neural Information Processing Systems}, 32, 2019.

\bibitem[Zhang et~al.(2023{\natexlab{a}})Zhang, Feng, Du, He, and
  Wang]{zhang2023complete}
Bohang Zhang, Guhao Feng, Yiheng Du, Di~He, and Liwei Wang.
\newblock A complete expressiveness hierarchy for subgraph gnns via subgraph
  weisfeiler-lehman tests.
\newblock \emph{International Conference on Machine Learning},
  2023{\natexlab{a}}.

\bibitem[Zhang et~al.(2023{\natexlab{b}})Zhang, Luo, Wang, and
  He]{zhang2023rethinking}
Bohang Zhang, Shengjie Luo, Liwei Wang, and Di~He.
\newblock Rethinking the expressive power of gnns via graph biconnectivity.
\newblock \emph{International Conference on Learning Representations},
  2023{\natexlab{b}}.

\bibitem[Zhang and Li(2021)]{zhang2021nested}
Muhan Zhang and Pan Li.
\newblock Nested graph neural networks.
\newblock In \emph{Advances in Neural Information Processing Systems},
  volume~34, 2021.

\bibitem[Zhao et~al.(2022)Zhao, Jin, Akoglu, and Shah]{zhao2022from}
Lingxiao Zhao, Wei Jin, Leman Akoglu, and Neil Shah.
\newblock From stars to subgraphs: Uplifting any {GNN} with local structure
  awareness.
\newblock In \emph{International Conference on Learning Representations}, 2022.

\end{thebibliography}
\bibliographystyle{plainnat}

\newpage
\appendix
\onecolumn

\section{Subgraph GNNs as Cartesian Product Graphs}
\label{app: Subgraph GNNs as Graph Cartesian Products}

In this section, we delve into the relationship between the Cartesian product graph and \texttt{Subgraphormer} (Appendix \ref{app: The Graph Cartesian Product}). Additionally, we explore the broader connection between Cartesian product graphs and Subgraph GNNs in general (detailed in \Cref{sec: Establishing the General Relationship Between Cartesian Product Graphs and Subgraph GNNs}).

\subsection{The Cartesian Product Graph and its Application for \texttt{Subgraphormer}}
\label{app: The Graph Cartesian Product}

\begin{restatable}[Cartesian Product Graph]{definition}{CartesianProductGraphDef}
\label{def: Graph Cartesian Product}
Given two graphs $G_1$ and $G_2$, their Cartesian product $G_1 \square G_2$ is defined as:
\begin{itemize}
    \item The vertex set $ V(G_1 \square G_2) = V(G_1) \times V(G_2) $.
    \item Vertices $ (u_1, u_2) $ and $ (v_1, v_2) $ in $ G_1 \square G_2 $ are adjacent if:
    \begin{itemize}
        \item $ u_1 = v_1 $ and $ u_2 $ is adjacent to $ v_2 $ in $ G_2 $, or
        \item $ u_2 = v_2 $ and $ u_1 $ is adjacent to $ v_1 $ in $ G_1 $.
    \end{itemize}
\end{itemize}
\end{restatable}

Thus, we observe the following,
\begin{corollary}
\label{cor: adjacency of cartesian product}
    Let $G$ be a given graph, defined by an adjacency $A$, and assume $A$ doesn't include self loops. The adjacency matrix of the Cartesian product of $G$ with itself; namely, $G \square G$, is given by 
    \begin{equation}
    \label{eq: A_GG}
        \mcalA_{G \square G} \triangleq A \otimes I + I \otimes A,
    \end{equation}
    where by $\otimes$ we denote the Kronecker product.
\end{corollary}
Therefore, we claim,
\AdjacenciesAsProducts*
The proof is available in \Cref{app: proofs}.

Thus, given~\Cref{cor: cartesian subgraph equiva}, our objective is to diagonalize the Laplacian of $G \square G$, which is defined by,
\begin{align}
    \mathcal{L}_{G \square G} &= \diag\Big(\mcalA_{G \square G}\vec{1}_{n^2}\Big) - \mcalA_{G \square G} \\
    &=  \diag\Big((A \otimes I + I \otimes A)\vec{1}_{n^2}\Big) - A \otimes I + I \otimes A, \nonumber 
\end{align}
where by $\vec{1}_{n^2}$ we refer to a vector of $n^2$ ones.

The eigendecomposition of $\mathcal{L}_{G \square G}$, is given in the following proposition,
\SubgraphormerPE*

The proof is available in \Cref{app: proofs}.

We conclude with the following result, stating the complexity of our product graph PE.
\begin{restatable}[Product Graph PE Complexity]{proposition}{ComplexityPE}
\label{thm: complexity}
    Let $G=(A, X)$ be an undirected graph with $n$ vertices, and consider its Cartesian product graph $G \square G$.
    \begin{enumerate}
        \item The time complexity for diagonalizing the Laplacian matrix of $G \square G$ is $\mathcal{O}(n^4)$.

        \item For calculating $k$ eigenvectors, where $k \leq n$, the time complexity is $\mathcal{O}(k \cdot n^2)$. This is equivalent to the complexity of computing $k$ eigenvectors for the original graph $G$.
    \end{enumerate}
\end{restatable}

The proof is available in \Cref{app: proofs}.

\subsection{Establishing the General Relationship Between Cartesian Product Graphs and Subgraph GNNs}
\label{sec: Establishing the General Relationship Between Cartesian Product Graphs and Subgraph GNNs}
In this section we develop the connection between Subgraph GNNs and the Graph Cartesian Product (GCP); the results are summarized in \Cref{tab: Subgraph GNNs as Graph Cartesian Products}. We focus specifically on the node-based (Node-Marking) generation policy, over an original graph $G=(A, X)$, and specify the formulation of the following four main connectivities: Internal Subgraph connectivity, External Subgraph connectivity, Global Internal Subgraph connectivity, Global External Subgraph connectivity. 

Before delving into the derivation details, we define two graphs, $G_s$ and $G_c$, corresponding to the ``set graph'' and the ``clique graph'', respectively. The ``set graph'' corresponds to a graph with no edges, thus:
\begin{align}
\label{eq: A set}
    A_{G_s} = \vec{0}\vec{0}^T,
\end{align}
and the ``clique graph'' corresponds to a graph where every pair of nodes is connected by an edge, hence:
\begin{align}
\label{eq: A clique}
    A_{G_c} = \vec{1}\vec{1}^T - I.
\end{align}
Recalling the definition of the Cartesian product graph (\Cref{{def: Graph Cartesian Product}}), we introduce a relevant corollary:

\begin{restatable}[Adjacency of Cartesian Product Graph]{corollary}{CartesianProductGraphAdj}
    Consider two distinct graphs, $G_1 = (A_1, X_1)$ and $G_2 = (A_2, X_2)$, with $n$ and $m$ vertices, respectively. The adjacency matrix of their Cartesian product, denoted as $G_1 \square G_2$, is given by:
    \begin{equation}
    \label{eq: A G1 G2}
        \mathcal{A}_{G_1 \square G_2} = A_{G_1} \otimes I_m + I_n \otimes A_{G_2},
    \end{equation}
    where $\otimes$ denotes the Kronecker product, and $I_n$ and $I_m$ are identity matrices in $\mathbb{R}^{n \times n}$ and $\mathbb{R}^{m \times m}$, respectively.
\end{restatable}
In this subsection, as our focus is on graphs each containing $n$ nodes, we will omit the subscripts $n/m$ from the identity matrix for simplicity.

\textbf{Internal Subgraph Connectivity.} The adjacency matrix corresponding to the internal subgraph connectivity, $\mathcal{A}_G$, arises from the Cartesian product of the graphs $G_s$ and $G$, that is,
\begin{align}
    \mathcal{A}_G \triangleq \mathcal{A}_{G_s \square G} = A_{G_s} \otimes I + I \otimes A.
\end{align}
By substituting~\Cref{eq: A set}, we obtain:
\begin{align}
    \mathcal{A}_{G_s \square G} = I \otimes A.
\end{align}
This results is in line with \Cref{cor: cartesian subgraph equiva}.

\textbf{External Subgraph Connectivity.} Similarly, the adjacency matrix for external subgraph connectivity, $\mathcal{A}_{G^S}$, results from the Cartesian product of the graphs $G$ and $G_s$, that is,
\begin{align}
    \mathcal{A}_{G^S} \triangleq \mathcal{A}_{G \square G_s} = A \otimes I + I \otimes A_{G_s}.
\end{align}
After substituting~\Cref{eq: A set}, we find:
\begin{align}
    \mathcal{A}_{G \square G_s} = A \otimes I.
\end{align}
We note that this is also in line with \Cref{cor: cartesian subgraph equiva}. 

As derived in \Cref{sec: MES-GNN Main Connectivity: Just a Simple Graph Cartesian Product?} in the main paper, the connectivity that combines both Internal and External updates can be given by the Cartesian product of $G$ with itself, as given in \Cref{eq: A_G_G}.

\textbf{Global Internal Subgraph Connectivity.} Following the literature~\cite{zhang2023complete,frasca2022understanding}, assuming no self loops, the adjacency matrix corresponding to the global internal subgraph connectivity is defined as follows:
\begin{align}
 \mathcal{A}^{\text{Global}}_{G}\Big((s, v), (s', v')\Big) = \delta_{ss'},
\end{align}
where $\delta$ denotes the Kronecker delta. 

This matrix results from the Cartesian product of $G_s$ and $G_c$:
\begin{align}
     \mathcal{A}^{\text{Global}}_{G} \triangleq \mathcal{A}_{G_s \square G_c} = A_{G_s} \otimes I + I \otimes A_{G_c}.
\end{align}
Substituting~\Cref{eq: A clique,eq: A set}, we obtain:
\begin{align}
    \mathcal{A}_{G_s \square G_c} =  \vec{0}\vec{0}^T  \otimes I  + I \otimes (\vec{1}\vec{1}^T - I) = I \otimes (\vec{1}\vec{1}^T - I).
\end{align}

\textbf{Global External Subgraph Connectivity.} Again, assuming no self loops, the adjacency matrix for global external subgraph connectivity is defined by:
\begin{align}
\mathcal{A}^{\text{Global}}_{G^S}\Big((s, v), (s', v')\Big) = \delta_{vv'}.
\end{align}

This matrix can be obtained from the Cartesian product of $G_c$ and $G_s$:
\begin{align}
     \mathcal{A}^{\text{Global}}_{G^S} \triangleq \mathcal{A}_{G_c \square G_s} = A_{G_c} \otimes I + I \otimes A_{G_s}.
\end{align}
Upon substituting~\Cref{eq: A clique,eq: A set}, we get:
\begin{align}
    \mathcal{A}_{G_s \square G_c} = (\vec{1}\vec{1}^T - I) \otimes I + I \otimes \vec{0}\vec{0}^T =  (\vec{1}\vec{1}^T - I) \otimes I.
\end{align}

Analogously to the adjacency that unifies the internal and external aggregations, we also show that the following Cartesian Graph product results in a unifying adjacency for the global internal connectivity and the global external connectivity. This Cartesian product graph is given by $G_c \square G_c$, and therefore, its corresponding adjacency is,
\begin{align}
    \mcalA_{G_c \square G_c} =(\vec{1}\vec{1}^T - I) \otimes I + I \otimes (\vec{1}\vec{1}^T - I).
\end{align}

\begin{table}[ht]
    \centering
    \scriptsize
    \caption{Table summarizing the connection between Cartesian Product Graphs and Subgraph GNNs.}
    \label{tab: Subgraph GNNs as Graph Cartesian Products}
    % \addtolength{\tabcolsep}{-0.4em}
    \resizebox{\textwidth}{!}{%
    \begin{tabular}{l?c|c|>{\raggedright\arraybackslash}m{2.5cm}?c|c|c}
    \toprule
    Connectivity type  & GCP & Adjacency & Visualization & GCP -- unified & Adjacency -- unified & Visualization -- unified \\
    \midrule
    Internal Subgraph Connectivity   & $G_s \square G$ & $\mcalA_{G_s \square G} \triangleq \mcalA_G = I \otimes A$ & \includegraphics[scale=0.2]{Figures/A_G.pdf} & \multirow{10}{*}{$G \square G$}  & \multirow{10}{*}{$\mcalA_{G \square G} = A \otimes I + I \otimes A$} & \multirow{2}{*}{\includegraphics[scale=0.2]{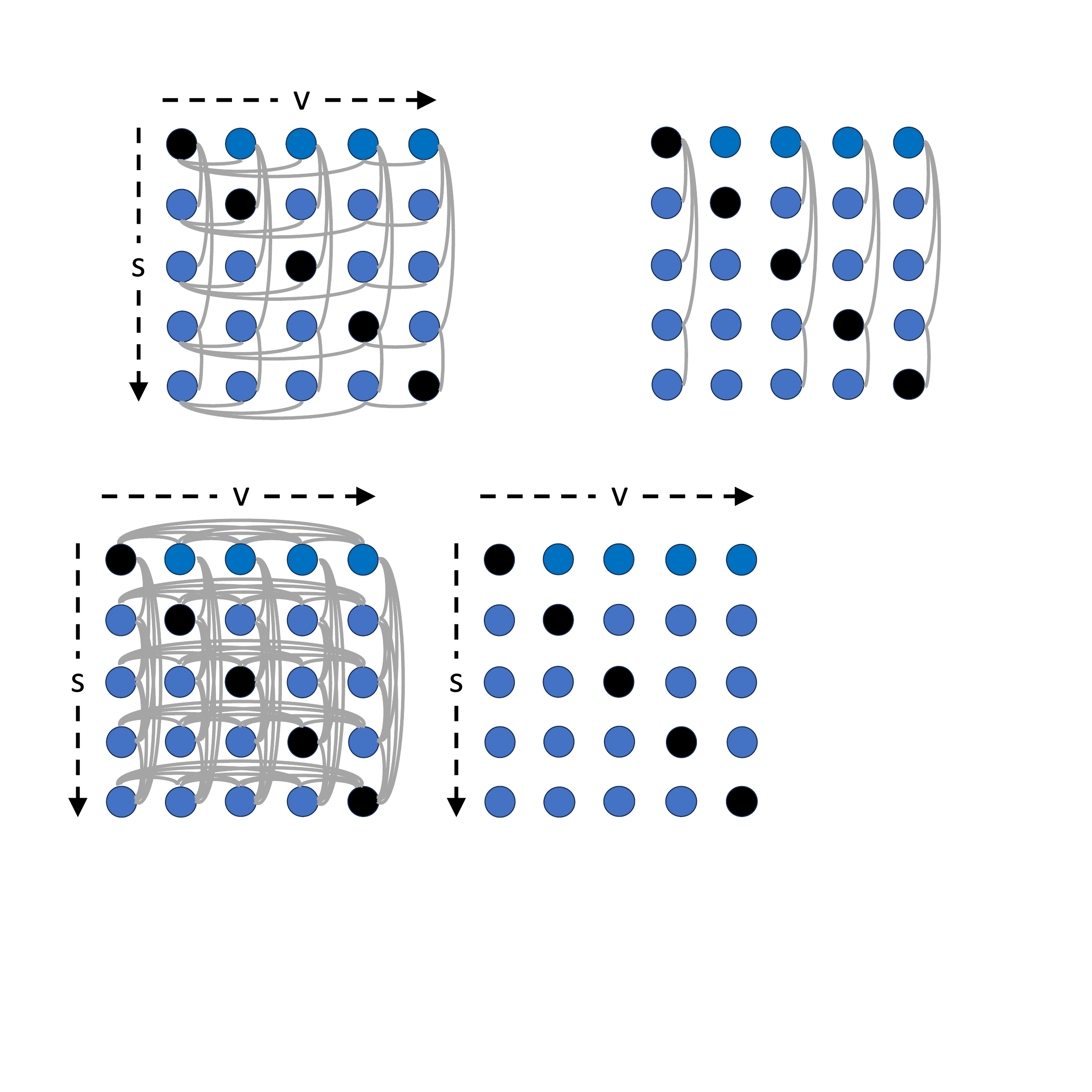}}\\
    % \cline{1-4}
    External Subgraph Connectivity & $G \square G_s$ & $\mcalA_{G \square G_s} \triangleq \mcalA_{G^S} = A \otimes I$ & \includegraphics[scale=0.2]{Figures/A_G_T.pdf} & & \\
    \specialrule{.2em}{.1em}{.1em} 
    Global Internal Subgraph Connectivity  & $G_s \square G_c$ & $\mcalA_{G_s \square G_c} = I \otimes (\vec{1}\vec{1}^T - I)$ & \includegraphics[scale=0.2]{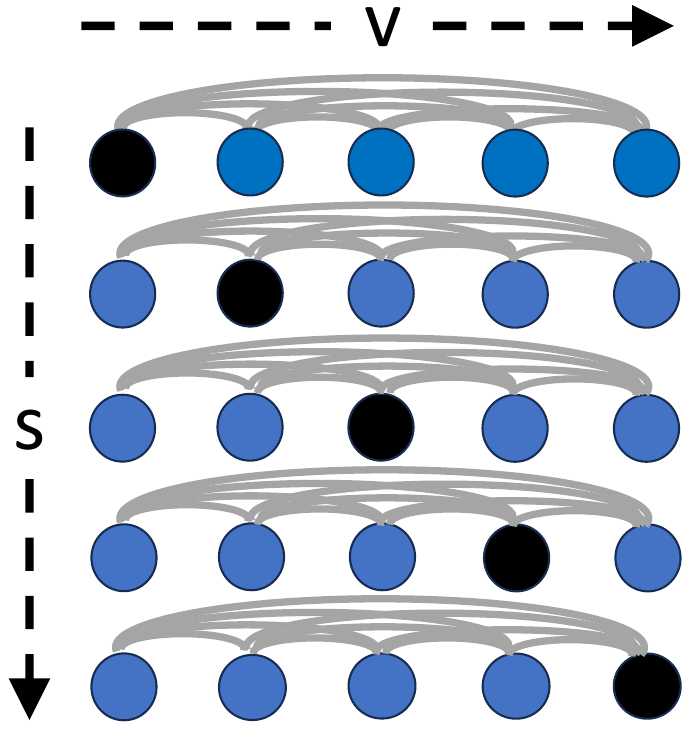}  & \multirow{10}{*}{$G_c \square G_c$}  & \multirow{10}{*}{$\mcalA_{G_c \square G_c} = (\vec{1}\vec{1}^T - I) \otimes I + I \otimes (\vec{1}\vec{1}^T - I)$} & \multirow{2}{*}{\includegraphics[scale=0.2]{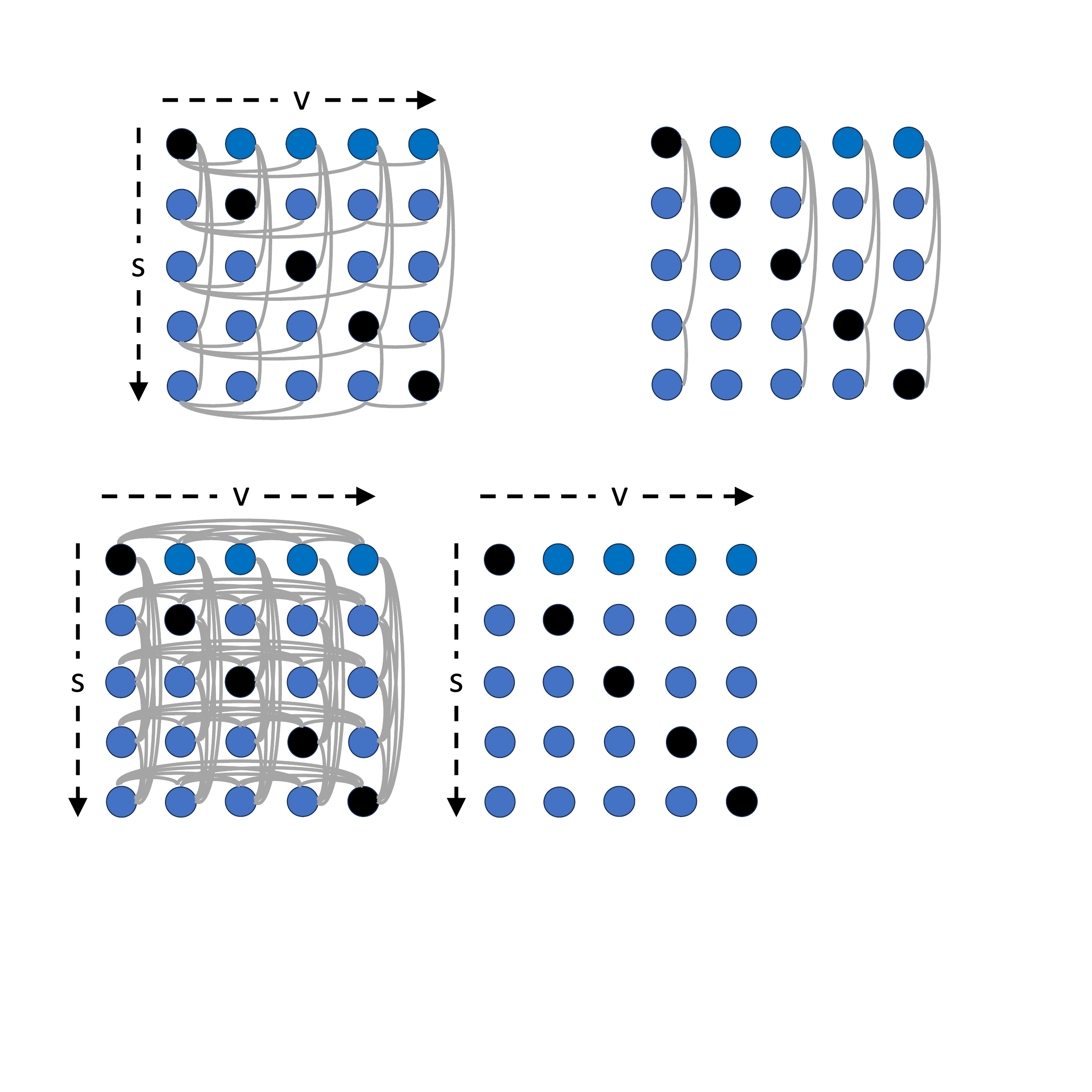}} \\
    % \cline{1-4}
    Global External Subgraph Connectivity &  $G_c \square G_s$ & $\mcalA_{G_c \square G_s} = (\vec{1}\vec{1}^T - I) \otimes I$ & \includegraphics[scale=0.2]{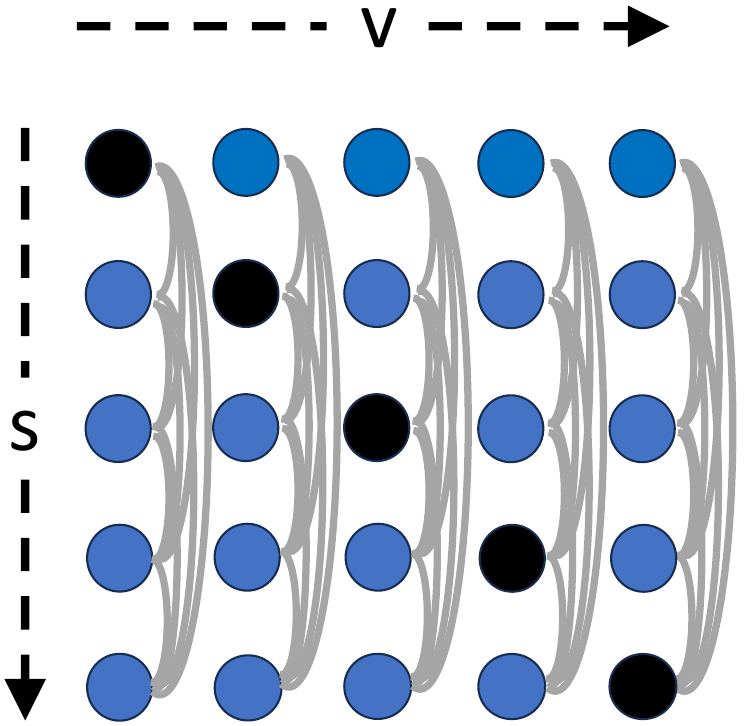}  & &  \\
    \bottomrule
    \end{tabular}
    }
\end{table}

% \begin{table}[ht]
%     \centering
%     \scriptsize
%     % \caption{On the \textsc{ZINC} datasets, \texttt{Subgraphormer} outperforms \colorbox{Gray}{Graph Transformers} and \colorbox{LightBlue}{Subgraph GNNs}. The top three results are reported as \textcolor{blue}{\textbf{First}}, \textcolor{red}{\textbf{Second}}, and \textcolor{orange}{\textbf{Third}}.}
%     \label{tab: Subgraph GNNs as Graph Cartesian Products}
%     %\resizebox{\columnwidth}{!}{%
%     \addtolength{\tabcolsep}{-0.4em}
%     \begin{tabular}{l|c|c|c|c}
%     \toprule
%     Connectivity type & Visualization & Corresponding Graph Cartesian Product & Adjacency & Adjacency unified \\
%     \midrule
%     Internal Subgraph Connectivity & - & $G \square G^s$ & $\mcalA_{G \square G^s} \triangleq \mcal_G = A \otimes I$  & asd \\
%     \midrule
%     External Subgraph Connectivity & - & $G^s \square G$ & $\mcalA_{G^s \square G} \triangleq \mcal_G = I \otimes A$  & asd \\
%         \midrule
%     Global Internal Subgraph Connectivity & - & $G \square G^q$ & $\mcalA_{G \square G^q} = (\vec{1}\vec{1}^T - I) \otimes I$  & asd \\
%         \midrule
%     Global External Subgraph Connectivity & - & $G^q \square G$ & $\mcalA_{G^q \square G} = I \otimes (\vec{1}\vec{1}^T - I) $  & asd \\
%     \bottomrule
%     \end{tabular}
%     %}
% \end{table}

\clearpage

\section{Extending Subgraphormer to any k-tuple}
\label{app: Extending Subgraphormer to Any k-tuple}

In this section, we extend the formulation of \texttt{Subgraphormer}, which was given for 2-tuples of nodes, to any $k$-tuple. We start in \Cref{app: k-tuple Product Graph PE}, by formulating the product graph PE for any $k$-tuple. This sets the required exposition for \Cref{app: k-tuple SAB}, where we formulate a SAB update to by applied to $k$-tuples.

\subsection{k-tuple Product Graph PE}
\label{app: k-tuple Product Graph PE}

In this section, we extend the concept of product graph PE to encompass any $k$-tuple for arbitrary values of $k$. This generalization stems from an essential observation that arises from the definition of the Cartesian Product Graph (\Cref{def: Graph Cartesian Product}).

\CartesianProductGraphAdj*

In this paper, our primary focus lies on a specific scenario where the Cartesian Product Graph is applied iteratively to the same graph. For this purpose, we introduce the Cartesian product operator. This operator, when applied to an adjacency matrix $A$ of a graph $G$, yields the adjacency matrix corresponding to $G^{\square^k}$ - the graph resulting from $k$-fold recursive Cartesian product of $G$ with itself.

\begin{definition}[Cartesian Product Operator]
\label{def: Cartesian operator}
    The Cartesian product operator is defined as:
    \begin{equation}
    \label{eq: Cartesian operator}
    \mcalC^k : \mathbb{R}^{n^k \times n^k} \rightarrow \mathbb{R}^{n^{k+1} \times n^{k+1}},
    \end{equation}
    applied to an adjacency matrix $A \in \mathbb{R}^{n \times n}$, recursively defined as:
    \begin{equation}
        \label{eq: Recursive Cartesian}
        \mcalC^k(A) = \mcalC^{k-1}(A) \otimes I_n + I_{n^{k-1}} \otimes A,
    \end{equation}
    with the base case being:
    \begin{equation}
        \label{eq: Cartesian Base Case}
        \mcalC^1(A) = A.
    \end{equation}
\end{definition}

It is apparent that for a graph $G$ with adjacency matrix $A$, applying the operator $\mathcal{C}^2$ to $A$ results in $\mathcal{C}^2(A) = A \otimes I_n + I_n \otimes A$, which is the adjacency matrix corresponding to the Cartesian Product of the graph with itself -- \Cref{eq: A Cartesian product}.

We propose the following proposition regarding the applicability of the Cartesian operator $\mcalC$ across any power $k$:

\begin{restatable}[Adjacency matrix of $G^{\square^k}$]{proposition}{AdjacencyOfGK}
\label{prop: A_G^k = C^k(A)}
    For a graph $G = (A, X)$ with adjacency matrix $A$, the adjacency matrix of $G^{\square^k}$—the graph formed by $k$-fold Cartesian product of $G$ with itself—is given by:
    \begin{equation}
        \mathcal{A}_{G^{\square^k}} = \mathcal{C}^k(A).    
    \end{equation}
\end{restatable}
The proof is given in \Cref{app: proofs}.

Utilizing \Cref{prop: A_G^k = C^k(A)}, we establish that $\mcalC^{k}$ effectively constructs the adjacency matrix corresponding to the graph $G^{\square^k}$. 

Notably, this operator functions recursively, therefore, to extend positional encoding to any $k$-tuple, our goal is to formulate a closed-form expression for the adjacency matrix of $G^{\square^k}$; we first introduce the necessary definition.

\begin{definition}[$\mcalA^k_{\mcalK}$]
\label{def: A_k^K}
    Given an adjacency matrix $A \in \mathbb{R}^{n \times n}$ and an index $k \in [\mcalK]$, define $\mcalA^k_{\mcalK}$ to be the tensor resulting from the Cartesian product of $\mcalK$ matrices, where each matrix is the identity matrix $I_n \in \mathbb{R}^{n \times n}$, except for the $k$-th factor which is $A$. That is, 
    \begin{equation}
        \label{eq: A^k_K}
        \mcalA^k_{\mcalK} = I_n \otimes I_n \otimes \cdots \otimes A \otimes \cdots \otimes I_n,
    \end{equation}

    where $A$ occupies the $k$-th position in the Cartesian product. %the positions are in $k \in \{0, 1, \ldots \mcalK-1\}$.
\end{definition}

Thus, we propose the following,

\begin{restatable}[Adjacency matrix of $G^{\square^k}$ -- closed form]{proposition}{AdjacencyOfGKClosedForm}
\label{prop: A_G^k = C^k}
    Given a graph $G = (A, X)$, with an adjacency matrix $A$, it holds that,
    \begin{equation}
    \label{eq: C^k closed form}
        \mcalC^{\mcalK}(A) = \sum_{k = 0}^{\mcalK-1} \mcalA^k_{\mcalK}.
    \end{equation}
\end{restatable}
The proof is given in \Cref{app: proofs}.

Having \Cref{prop: A_G^k = C^k(A)} and \Cref{prop: A_G^k = C^k}, we straight forwardly infer the following Corollary,

\begin{corollary}
    For a graph $G = (A, X)$, the adjacency matrix of $G^{\square^k}$—the graph formed by $k$-fold Cartesian product of $G$ with itself—is given by:
    \begin{equation}
    \label{eq: A_G^box k}
        \mathcal{A}_{G^{\square^\mcalK}} = \sum_{k = 0}^{\mathcal{K}-1} \mathcal{A}^k_{\mathcal{K}}.    
    \end{equation}
\end{corollary}
In what follows, we prove that the matrix $\mathcal{A}_{G^{\square^k}}$ is a valid, binary adjacency matrix.

\begin{restatable}[$\mathcal{A}_{G^{\square^k}}$ -- Valid binary adjacency matrix]{proposition}{AdjacencyOfGKIsValid}
    Given a graph $G= (A, X)$ with an adjacency matrix $A$, with no self loops, the matrix given by,
    \begin{equation}
        \mcalA_{G^{\square^{\mcalK}}} = \sum_{k = 0}^{\mcalK-1}  \mcalA^{k}_{\mcalK},
    \end{equation}
    is a binary adjacency matrix.
\end{restatable}
The proof is given in \Cref{app: proofs}.

Having developed this closed form expression for the (binary) adjacency matrix of $G^{\square^k}$, we have the following Proposition, which generalizes \Cref{prop: main prop} to any $k$-tuple,

\begin{restatable}[$k$-tuple -- Product graph eigendecomposition]{proposition}{KTupleSubgraphormerPE}
\label{prop: gen main prop}
Consider a graph $G = (A, X)$ without self-loops. Define $G^{\square^k}$ as the Cartesian product of $G$ repeated $k$ times, with its adjacency matrix denoted by $\mathcal{A}_{G^{\square^k}}$. 
The eigenvectors and eigenvalues of the Laplacian matrix for $G^{\square^k}$ can be characterized as follows: for each set of indices $\{i_1, i_2, \ldots, i_k\}$, where $i_j \in \{1, 2, \ldots, n\}$ for each $j$, there exists an eigenvector-eigenvalue pair given by $(v_{i_1} \otimes v_{i_2} \otimes \ldots \otimes v_{i_k}, \lambda_{i_1} + \lambda_{i_2} + \ldots + \lambda_{i_k})$. 
Here, $\{(v_i, \lambda_i) \}_{i=1}^{n}$ represent the eigenvectors and eigenvalues of the Laplacian matrix of the original graph $G$.
\end{restatable}
The proof is given in \Cref{app: proofs}.

The above Proposition enables the computation of positional encodings for any $k$-tuple, offering significantly more efficiency compared to the conventional method of directly diagonalizing a matrix in $ \mathbb{R}^{n^k \times n^k} $.

\begin{restatable}[$k$-tuple PE efficiency]{theorem}{KTupleSubgraphormerPEEfficiency}
    Consider a graph $G = (A, X)$ without self-loops. Define $ G^{\square^k} $ as the Cartesian product of $ G $ repeated $ k \geq 2 $ times.

    The time complexity of computing the eigendecomposition of  the Laplacian matrix of $ G^{\square^k} $ is $ \mathcal{O}(n^{2k}) $.
\end{restatable}

The proof is given in \Cref{app: proofs}.

We note that diagonalizing directly the Laplacian matrix of $G^{\square^k}$ would instead take $\mathcal{O}(n^{3k})$.

%\clearpage

\subsection{k-tuple SAB}
\label{app: k-tuple SAB}

Building upon the derivation of the adjacency matrix obtained from the Cartesian product of graph $G$ applied recursively $k$ times, as outlined in \Cref{eq: A_G^box k}, we can extend the node update mechanisms, initially conceptualized for 2-tuples in \Cref{sec: Subgraph Attention Block}, to support $k$-tuples. 

\textbf{Generalization of Internal/External updates.} We recall the $2$-tuple case, where the two matrices corresponding the Internal and External aggregations of \texttt{Subgraphormer} are given in \Cref{eq: A_G_G}. In our new formalism, using the Cartesian Product Operator, we can write,
\begin{equation}
    \mcalA_{G \square G} = \mcalA^0_2 + \mcalA^1_2,
\end{equation}
where $\mcalA^0_2 = \mcalA_{G^S}$ and $\mcalA^1_2 = \mcalA_G$. Thus the Internal and External updates are given as $\{ \mcalA^0_2, \mcalA^1_2 \}$.

Building on \Cref{eq: A_G^box k}, we can extend to a $\mcalK$-tuple scenario where the ``Internal/External'' updates correspond to a series of adjacencies, denoted by $\{ \mcalA^i_\mcalK \}_{i=0}^{\mcalK-1}$.

\textbf{Generalization of Point-wise update.}
As defined in \Cref{sec: From Subgraph GNNs to graph products}, the point-wise update should allow a node $v$ in subgraph $s$ to have access to the root representation, where the roots in the case of a $2$-tuple are defined to be, $\{ (v,v); v \in V \}$.

While the root nodes in the case of $k$-tuples can be defined as $\{(\overbrace{v, \ldots, v}^{\text{$k$ elements}}); v \in V \}$, the 
correspondence between nodes and subgraphs is not clear. 

Consequently, for a $k$-tuple scenario, we define the point-wise update by mapping the ``subgraph'' to a subset of $k-1$ nodes within the tuple, and designating the ``node'' as the remaining element in the tuple. Take, for instance, the tuple $(s_1, s_2, \ldots, s_i, v, s_{i+1}, \ldots, s_k)$: here, the ``subgraph'' comprises the indices $(s_1, s_2, \ldots, s_i, s_{i+1}, \ldots, s_k)$, while the ``node'' is represented by $v$.

Accordingly, we establish $k$ such point-wise updates. To be precise, for each $i \in \{ 1, 2, \ldots, k \}$, we define the update as follows:

\begin{align}
     \mcalA^{i}_{\text{point}}\big((v_1, \ldots, v_k), (v'_1, \ldots, v'_k)\big) = 
     \begin{cases} 
        1 & \text{if } v'_1 = \ldots = v'_k \text{ and } v'_1 = v_j \text{ for all } j \neq i; \\
        0 & \text{otherwise}.
    \end{cases} 
\end{align}

This equation defines the function $\mcalA^{i}_{\text{point}}$, which evaluates the adjacency between two $k$-tuples of vertices, $(v_1, \ldots, v_k)$ and $(v'_1, \ldots, v'_k)$. The function returns $1$, indicating a specific type of adjacency, if the second tuple is a root node -- all vertices in the second tuple are identical ($v'_1 = v'_2 = \ldots = v'_k$), and simultaneously, all vertices in the first tuple match this root node for all indices except $i$, denoted as $v'_1 = v_j$ for all $j \neq i$. In other instances, where these conditions are not satisfied, the function returns $0$, signifying the absence of such adjacency.

%\clearpage

\section{Extended Experimental Section}
\label{app: Extended Experimental Section}

\subsection{Dataset Description}
\label{app: Dataset Description}
In this section we overview the eight different datasets considered; this is summarized in \Cref{tab: dataset-overview}.
\begin{table*}[t]
    \centering
        \caption{Overview of the graph learning datasets.}
    \resizebox{\textwidth}{!}{%
        \begin{tabular}{l|c|c|c|c|c|c}
            \toprule
            \textbf{Dataset} & \textbf{\# Graphs} & \textbf{Avg. \# nodes} & \textbf{Avg. \# edges} & \textbf{Directed} &  \textbf{Prediction task} & \textbf{Metric} \\
            \midrule
            \textsc{ZINC-12k}~\cite{sterling2015zinc} & 12,000 & 23.2 & 24.9 & No  & Regression & Mean Abs. Error \\
            \textsc{ZINC-Full}~\cite{sterling2015zinc} & 249,456 & 23.2 & 49.8 & No  & Regression & Mean Abs. Error \\
            \textsc{Alchemy-12k}~\cite{chen2019alchemy} & 12,000 & 10.12 & 20.9 & No & 12-task Regression &  Mean Abs. Error \\
            \textsc{ogbg-molhiv}~\cite{hu2020open}  & 41,127 & 25.5 & 27.5 & No  & Binary Classification & AUROC \\
            \textsc{ogbg-molbace}~\cite{hu2020open}  & 1513 & 34.1 & 36.9  & No & Binary Classification & AUROC \\
            \textsc{ogbg-molesol}~\cite{hu2020open}  & 1,128 & 13.3 & 13.7  & No & Regression & Root Mean Squ. Error \\
            \midrule
            \textsc{Peptides-func}~\cite{dwivedi2022long} & 15,535 & 150.9 & 307.3 & No & 10-task Classification & Avg. Precision \\
            \textsc{Peptides-struct}~\cite{dwivedi2022long} & 15,535 & 150.9 & 307.3 & No & 11-task Regression & Mean Abs. Error \\
            \bottomrule
        \end{tabular}
    }
    \label{tab: dataset-overview}
\end{table*}

\textbf{\textsc{ZINC-12k} and \textsc{ZINC-Full} Datasets~\cite{sterling2015zinc, gomez2018automatic, dwivedi2023benchmarking}.} The \textsc{ZINC-12k} dataset comprises 12,000 molecular graphs, extracted from the ZINC database, which is a collection of commercially available chemical compounds. These molecular graphs vary in size, ranging from 9 to 37 nodes each. In these graphs, nodes correspond to heavy atoms, encompassing 28 distinct atom types. Edges in the graphs represent chemical bonds, with three possible bond types. The primary objective when using this dataset is to perform regression analysis on the constrained solubility (logP) of the molecules. The dataset is pre-partitioned into training, validation, and test sets, containing 10,000, 1,000, and 1,000 molecular graphs, respectively. In its full form, namely, \textsc{ZINC-Full}, it contains approximately 250,000 molecular graphs. These graphs vary in complexity, with each graph containing again between 9 to 37 nodes, and 16 to 84 edges. The nodes in these graphs also represent heavy atoms, and the dataset includes 28 different types of atoms. The edges, on the other hand, represent bonds between these atoms, and there are 4 distinct types of bonds in the dataset.

\textbf{\textsc{ogbg-molhiv}, \textsc{ogbg-molbace}, \textsc{ogbg-molesol} Datasets~\cite{hu2020open}.} Those datasets are molecular property prediction datasets, adopted by the Open Graph Benchmark (OGB) from MoleculeNet. These datasets employ a unified featurization for nodes (atoms) and edges (bonds), encapsulating various chemophysical properties.

\textbf{\textsc{Alchemy-12k} Dataset~\cite{chen2019alchemy}.} The \textsc{Alchemy-12k} dataset 
includes quantum mechanical properties of 12,000 organic molecules containing up to 14 heavy atoms like Carbon (C), Nitrogen (N), Oxygen (O), Fluorine (F), Sulfur (S), and Chlorine (Cl). These properties were computed using the Python-based Simulations of Chemistry Framework (PySCF).

\textbf{\textsc{Peptides-func} and \textsc{Peptides-struct} Datasets~\cite{dwivedi2022long}.} The \textsc{Peptides-func} and \textsc{Peptides-struct} datasets consist of atomic graphs representing peptides. In \textsc{Peptides-func}, the task involves multi-label graph classification into ten non-exclusive peptide functional classes. Conversely, \textsc{Peptides-struct} focuses on graph regression to predict eleven three-dimensional structural properties of the peptides.

\subsection{Experimental Details}
\label{app: Experimental Details}
Our experiments were conducted using the PyTorch \cite{paszke2019pytorch} and PyTorch Geometric \cite{fey2019fast} frameworks, using a single NVIDIA L40 GPU, and for every considered experiment, we show the mean $\pm$ std. of 3 runs with different random seeds. Hyperparameter tuning was performed utilizing the Weight and Biases framework \cite{wandb} -- see \Cref{app: HyperParameters}. All our \texttt{MLP}s feature a single hidden layer equipped with a ReLU non-linearity function. For the encoding of atom numbers and bonds, we utilized learnable embeddings indexed by their respective numbers. An exception was made for the \textsc{Alchemy-12k} dataset, where instead we employed a linear layer, as done in \citet{morris2022speqnets}. For all our datasets, we used a fixed value of $H=4$ attention heads.

In the case of the \textsc{ogbg-molhiv}, \textsc{ogbg-molesol}, \textsc{ogbg-molbace}, \textsc{Alchemy-12k} datasets, we follow \citet{frasca2022understanding}, therefore adding a residual connection between different layers. Additionally, for those datasets, we used linear layers instead of \texttt{MLP}s inside the GIN layers. Moreover, for these four datasets, and for the \textsc{Peptides} datasets, the following pooling mechanism was employed, instead of the one mentioned in~\Cref{eq: pooling} (which was used for \textsc{ZINC-12k} and \textsc{ZINC-Full}),

\begin{equation}
    \rho(\mathcal{X}) = \texttt{MLP} \left( \sum_{s=1}^n \left( \frac{1}{n} \sum_{v=1}^n  \mathcal{X}(s,v) \right) \right).
\end{equation}
For the \textsc{Peptides} datasets, we also used a residual connection between layers.

\subsection{Additional Results}
\label{app: Extended Results}

\textbf{\textsc{Alchemy-12k}.}
The results regarding the \textsc{Alchemy-12k} dataset~\citep{chen2019alchemy} is summarized in \Cref{tab: alchemy}, showcasing \texttt{Subgraphormer} outperforming Subgraph GNNs methods.
\begin{table}[t]
    \centering
    \caption{Test results over the \textsc{Alchemy-12k} dataset. 
    \colorbox{LightBlue}{Subgraph-based} baselines are highlighted in light blue. The top three results are reported as \textcolor{blue}{\textbf{First}}, \textcolor{red}{\textbf{Second}}, and \textcolor{orange}{\textbf{Third}}.}
    \label{tab: alchemy}
    \scriptsize
    % \resizebox{\columnwidth}{!}{%
    \begin{tabular}{l|c}
    \toprule
    \multirow{2}*{Model} & \textsc{Alchemy-12k} \\
    & (MAE $\downarrow$) \\
    \midrule
    GIN~\cite{xu2018powerful} & $0.180$\tiny $\pm0.006$ \\
    SignNet~\cite{lim2022sign} & $\textcolor{orange}{\mathbf{0.113}}$\tiny $\pm0.002$ \\
    $\delta$-2-GNN~\cite{morris2020weisfeiler} & $0.118$\tiny $\pm0.001$ \\
    $\delta$-2-LGNN~\cite{morris2020weisfeiler} & $0.122$\tiny $\pm0.003$ \\
    SpeqNet~\cite{morris2022speqnets} & $0.115$\tiny $\pm0.001$ \\
    GNN-IR~\cite{dupty2022graph} & $0.119$\tiny $\pm0.002$ \\
    PF-GNN~\cite{dupty2021pf} & $\textcolor{red}{\mathbf{0.111}}$\tiny $\pm0.010$ \\
    PPGN$++$(6)~\cite{puny2023equivariant} & $\textcolor{blue}{\mathbf{0.109}}$\tiny $\pm0.001$ \\
    \midrule
    \rowcolor{LightBlue} Recon. GNN~\cite{cotta2021reconstruction} & $0.125$\tiny $\pm0.001$ \\
    \rowcolor{LightBlue} DS-GNN (NM)~\cite{bevilacqua2023efficient} & $0.116$\tiny $\pm0.001$ \\
    \rowcolor{LightBlue} \revision{GNN-SSWL+~\cite{zhang2023complete}} & $0.116$\tiny $\pm0.002$ \\
    \midrule
    \texttt{Subgraphormer} & $0.114$\tiny $\pm0.001$ \\
    \texttt{Subgraphormer + PE} & $\textcolor{orange}{\mathbf{0.113}}$\tiny $\pm0.002$ \\
    \bottomrule
    \end{tabular}
    %}
\end{table}

\textbf{Stochastic Subgraph Sampling.}
The comparison of our stochastic sampling approach against the proposed method by~\citet{bevilacqua2021equivariant} over the \textsc{Zinc-12k} dataset is given in \Cref{tab: zinc-sampling}.
\begin{table}[t]
\centering
\scriptsize
\caption{Comparison of the stochastic variant of \texttt{Subgraphormer} and \texttt{Subgraphormer PE} to DSS-GNN, over the \textsc{ZINC-12k} dataset. Best result for each sampling ratio is in \textbf{Bold}.}
% \caption{Comparison between \colorbox{Gray}{Transformer-based} and \colorbox{LightBlue}{Subgraph-based} architectures.}
\label{tab: zinc-sampling}
% \resizebox{\columnwidth}{!}{%
\begin{tabular}{lr|c}
\toprule
\multirow{2}*{Model} &   & \textsc{ZINC-12k}  \\
&   &(MAE $\downarrow$) \\
\midrule
\multirow{4}{*}{DSS-GNN (EGO+)~\cite{bevilacqua2021equivariant}} & 100\%  & $0.102$\tiny $\pm0.003$ \\
 & 50\%  & $0.155$\tiny $\pm 0.007$  \\
 & 20\%  & $0.166$\tiny $\pm0.005$  \\
 & 5\%  & $0.179$\tiny $\pm0.001$ \\
 \midrule
% \multirow{4}{*}{\revision{GNN-SSWL+~\cite{zhang2023complete}}} & 100\%  &  \\
%  & 50\%  &  \\
%  & 20\%  &   \\
%  & 5\%  &  \\
%  \midrule
 \multirow{4}{*}{\texttt{Subgraphormer}} & 100\%  & $0.067$\tiny $\pm0.007$  \\
 & 50\%  & $\mathbf{0.084}$\tiny $\pm0.002$   \\
 & 20\%  & $0.121$\tiny $\pm0.007$  \\
 & 5\%  & $ 0.200$\tiny $\pm0.017$ \\
 \midrule
 \multirow{4}{*}{\texttt{Subgraphormer + PE}} & 100\%  & $\mathbf{0.063} $\tiny $\pm0.0003$ \\
 & 50\%  & $\mathbf{0.084}$\tiny $\pm0.002$ \\
 & 20\%  & $\mathbf{0.120}$\tiny $\pm0.002$ \\
 & 5\%  & $\mathbf{0.175}$\tiny $\pm0.006$ \\
\bottomrule
\end{tabular}
% }
\end{table}

% \begin{table}[h]
% \centering
% \scriptsize
% \caption{Best results are in \textbf{Bold}.}
% % \caption{Comparison between \colorbox{Gray}{Transformer-based} and \colorbox{LightBlue}{Subgraph-based} architectures.}
% \label{tab: zinc-sampling}
% % \resizebox{\columnwidth}{!}{%
% \begin{tabular}{lr|c|c}
% \toprule
% \multirow{2}*{Model} &   & \textsc{ZINC-12k} & \textsc{ogbg-molhiv} \\
% &   &(MAE $\downarrow$) & (ROC-AUC $\uparrow$) \\
% \midrule
% \multirow{4}{*}{DSS-GNN (EGO+)~\cite{bevilacqua2021equivariant}} & \textbf{100\%}  & $0.102 \pm 0.003$ &  $76.78 \pm 1.66 $\\
%  & \textbf{50\%}  & $0.155 \pm 0.007$  & $76.88 \pm 0.93$  \\
%  & \textbf{20\%}  & $0.166 \pm 0.005$  & $\mathbf{76.93 \pm 1.45}$ \\
%  & \textbf{5\%}  & $0.179 \pm 0.001$  & $\mathbf{75.97 \pm 0.80}$ \\
%  \midrule
%  \multirow{4}{*}{\texttt{Subgraphormer}} & \textbf{100\%}  & $0.067 \pm 0.007$ & $\mathbf{80.38 \pm 1.92}$ \\
%  & \textbf{50\%}  & $\mathbf{0.084 \pm 0.002}$ & $\mathbf{79.66 \pm 0.79}$  \\
%  & \textbf{20\%}  & $0.121 \pm 0.007$ &  $76.48 \pm 2.38$  \\
%  & \textbf{5\%}  & $ 0.200 \pm 0.017$ &  $ 70.86 \pm 0.9$ \\
%  \midrule
%  \multirow{4}{*}{\texttt{Subgraphormer + graph product PE}} & \textbf{100\%}  & $\mathbf{0.063 \pm 0.0003}$ & $79.48 \pm 1.28$  \\
%  & \textbf{50\%}  & $\mathbf{0.084 \pm 0.002}$ & $79.62 \pm 1.32$ \\
%  & \textbf{20\%}  & $\mathbf{0.120 \pm 0.002}$ & $76.68 \pm 1.07$ \\
%  & \textbf{5\%}  & $ \mathbf{0.175 \pm 0.006} $ & $71.02 \pm 0.79$ \\
% \bottomrule
% \end{tabular}
% % }
% \end{table}

In \Cref{tab: molbace-sampling,tab: molesol-sampling} we compare our stochastic sampling approach against a stochastic version of GNN-SSWL$+$. The table demonstrates that \texttt{Subgraphormer + PE} consistently surpasses the performance of GNN-SSWL$+$ across the datasets examined (7 out of 8 combinations of sampling ratios and datasets).

\begin{table}[t]
\centering
\scriptsize
\caption{Comparison of the stochastic variant of \texttt{Subgraphormer} and \texttt{Subgraphormer PE} to GNN-SSWL$+$, over the \textsc{ogbg-molbace} dataset. Best result for each sampling ratio is in \textbf{Bold}.}
\label{tab: molbace-sampling}
% \resizebox{\columnwidth}{!}{%
\begin{tabular}{lr|c}
\toprule
\multirow{2}*{Model} &   & \textsc{ogbg-molbace}  \\
&   &(ROC-AUC $\uparrow$) \\
\midrule
\multirow{4}{*}{GNN-SSWL$+$~\cite{zhang2023complete}} & 100\%  & $82.70$\tiny $\pm1.80$ \\
 & 50\%  & $79.99$\tiny $\pm 0.58$  \\
 & 20\%  & $78.04$\tiny $\pm5.98$  \\
 & 5\%  & $\mathbf{68.52}$\tiny $\pm5.73$ \\
 \midrule
 \multirow{4}{*}{\texttt{Subgraphormer}} & 100\%  & $81.62$\tiny $\pm3.55$  \\
 & 50\%  & $79.49$\tiny $\pm2.38$   \\
 & 20\%  & $75.27$\tiny $\pm5.63$  \\
 & 5\%  & $ 63.05$\tiny $\pm8.70$ \\
 \midrule
 \multirow{4}{*}{\texttt{Subgraphormer + PE}} & 100\%  & $\mathbf{84.35} $\tiny $\pm0.65$ \\
 & 50\%  & $\mathbf{83.82}$\tiny $\pm2.62$ \\
 & 20\%  & $\mathbf{78.77}$\tiny $\pm4.10$ \\
 & 5\%  & $67.73$\tiny $\pm5.50$ \\
\bottomrule
\end{tabular}
% }
\end{table}

\begin{table}[t]
\centering
\scriptsize
\caption{Comparison of the stochastic variant of \texttt{Subgraphormer} and \texttt{Subgraphormer PE} to GNN-SSWL$+$, over the \textsc{ogbg-molesol} dataset. Best result for each sampling ratio is in \textbf{Bold}.}
\label{tab: molesol-sampling}
% \resizebox{\columnwidth}{!}{%
\begin{tabular}{lr|c}
\toprule
\multirow{2}*{Model} &   & \textsc{ogbg-molesol}  \\
&   &(RMSE $\downarrow$) \\
\midrule
\multirow{4}{*}{GNN-SSWL$+$~\cite{zhang2023complete}} & 100\%  & $0.837$\tiny $\pm0.019$ \\
 & 50\%  & $0.886$\tiny $\pm 0.026$  \\
 & 20\%  & $1.180$\tiny $\pm0.036$  \\
 & 5\%  & $1.299$\tiny $\pm0.044$ \\
 \midrule
 \multirow{4}{*}{\texttt{Subgraphormer}} & 100\%  & $0.832$\tiny $\pm0.043$  \\
 & 50\%  & $0.829$\tiny $\pm0.013$   \\
 & 20\%  & $1.093$\tiny $\pm0.009$  \\
 & 5\%  & $ \mathbf{1.266}$\tiny $\pm0.019$ \\
 \midrule
 \multirow{4}{*}{\texttt{Subgraphormer + PE}} & 100\%  & $\mathbf{0.826} $\tiny $\pm0.010$ \\
 & 50\%  & $\mathbf{0.812}$\tiny $\pm0.001$ \\
 & 20\%  & $\mathbf{1.041}$\tiny $\pm0.030$ \\
 & 5\%  & $1.270$\tiny $\pm0.007$ \\
\bottomrule
\end{tabular}
% }
\end{table}

\subsubsection{Product Graph PE vs Concatenation PE}
\label{app: Graph Product PE vs Concatenation PE}
For completeness, we expand on the other valid choice of subgraph positional encodings, also introduced in \Cref{{sec: Subgraph Positional Encoding}} -- \emph{concatenation PE}, which is a more general subgraph positional encoding scheme. Specifically, given the Laplacian matrix of $G$, $L$, and its eigendecomposition, $L = U^T \Lambda U$, we define the \emph{concatenation PE} for node $(s,v)$ to be the output of an \texttt{MLP} acting on the concatenation of $\mathbf{p}_{s}^{:k} \triangleq [U_{s1}, \ldots, U_{sk}]$ and $\mathbf{p}_{v}^{:k} \triangleq [U_{v1}, \ldots, U_{vk}]$.

\begin{restatable}[\emph{Concatenation PE} can approximate \emph{product graph PE}]{proposition}{CatApproxProdPE}
    \label{prop: concat}
     The \emph{concatenation PE} can approximate (up to an ordering) the \emph{product graph PE} uniformly.
\end{restatable}
The proof is given in \Cref{app: proofs}.

Nevertheless, in this paper, we opt for the \emph{product graph PE}, as it demonstrates better performance in general. As an example, we provide a comparison between the \emph{product graph PE} and the \emph{Concatenation PE} over the \textsc{ZINC-12k} dataset in \Cref{tab: product vs concatenation}.

\begin{table}[t]
\centering
\scriptsize
\caption{A Comparison between different PE variants; including \texttt{Subgraphormer} alone. Best results are in \textbf{Bold}.}
\label{tab: product vs concatenation}
\resizebox{0.5\columnwidth}{!}{%
\begin{tabular}{l|c}
\toprule
\multirow{2}*{Model} &   \textsc{ZINC-12k}  \\
&   (MAE $\downarrow$) \\
\midrule
\texttt{Subgraphormer}  & $0.067$\tiny $\pm0.007$  \\
 \midrule

\texttt{Subgraphormer + concatenation PE} & $0.071$\tiny $\pm0.003$   \\
 \midrule
\texttt{Subgraphormer + product graph PE} & $\mathbf{0.063}$\tiny $\pm0.0003$  \\
\bottomrule
\end{tabular}
}
\end{table}
Notably, \texttt{Subgraphormer + product graph PE} exhibits superior performance. Interestingly, despite the theoretical capability of \texttt{concatenation PE} to implement \texttt{product graph PE} (as discussed in \Cref{prop: concat}), our empirical findings reveal that \texttt{concatenation PE} actually reduces the effectiveness of \texttt{Subgraphormer}.

\subsection{Runtime Comparison}
\label{app: Runtime Comparison}
We estimated the training time and inference time of a single epoch (measured in seconds) for our architecture, as well as for three representative baselines: GIN~\cite{xu2018powerful} (an MPNN), GNN-SSWL$+$~\cite{zhang2023complete} (a Subgraph GNN), and GPS~\cite{rampavsek2022recipe} (A Graph Transformer). For a fair comparison, we used an NVIDIA A100 GPU for all methods; the run-time of GPS was taken from their paper~\cite{rampavsek2022recipe}. The experiment was performed on the \textsc{ZINC-12k} dataset (using a batch size of 128 for all baselines, except GPS, which used a batch size of 32), and over the \textsc{ogbg-molhiv} dataset (using a batch size of 32 for all baselines). results are summarized in the \Cref{tab: run-time zinc,tab: run-time molhiv}.

As shown in the tables, and as expected, GIN (the MPNN baseline) has the fastest run-time, while exhibiting the lowest performance. Our method and GNN-SSWL$+$ both outperform GPS in terms of run-time across both datasets. Over the \textsc{ZINC-12k} dataset, our method’s run-time is roughly the same when compared to GNN-SSWL$+$ (our architecture offers a very modest speed advantage). On the \textsc{ogbg-molhiv} dataset, GNN-SSWL$+$ runs slightly faster. Finally, we note that \texttt{Subgraphormer} uses fewer parameters than all baselines, on both datasets.

\begin{table}[t]
\centering
% \scriptsize
\caption{Comparison of training and inference times (and parameter counts) on the \textsc{ZINC-12k} dataset using an NVIDIA A100 GPU. The table presents the time taken (in seconds) to train for one epoch and to perform inference on the test set.}
\label{tab: run-time zinc}
\resizebox{0.5\columnwidth}{!}{%
\begin{tabular}{l|c|c|c|c}
\toprule
Model &  \text{Param.} & Train time (s) & Test time (s) & MAE $\downarrow$  \\
\midrule
GIN~\cite{xu2018powerful} & $500k$ & $1.41 \pm 0.22$ & $0.36 \pm 0.02$ & $0.163 \pm 0.004$ \\
GPS~\cite{rampavsek2022recipe} & $424k$ & $21 \pm \text{N/A}$ & N/A & $0.070 \pm 0.004$ \\
GNN-SSWL$+$~\cite{zhang2023complete} & $387k$ & $9.65 \pm 0.19$ & $1.04 \pm 0.03$ & $0.070 \pm 0.005$ \\
\midrule
\texttt{Subgraphormer + PE} & $293k$ & $9.60 \pm 0.10$ & $0.95 \pm 0.03$ & $0.063 \pm 0.001$ \\
\bottomrule
\end{tabular}
}
\end{table}

\begin{table}[t]
\centering
% \scriptsize
\caption{Comparison of training and inference times (and parameter counts) on the \textsc{ogbg-molhiv} dataset using an NVIDIA A100 GPU. The table presents the time taken (in seconds) to train for one epoch and to perform inference on the test set.}
\label{tab: run-time molhiv}
\resizebox{0.5\columnwidth}{!}{%
\begin{tabular}{l|c|c|c|c}
\toprule
Model &  \text{Param.} & Train time (s) & Test time (s) & ROC-AUC $\downarrow$  \\
\midrule
GIN~\cite{xu2018powerful} & $1800k$ & $12.65 \pm 0.21$ & $1.05 \pm 0.08$ & $75.58 \pm 1.40$ \\
GPS~\cite{rampavsek2022recipe} & $558k$ & $96 \pm \text{N/A}$ & N/A & $78.80 \pm 1.01$ \\
GNN-SSWL$+$~\cite{zhang2023complete} & $46k$ & $51.02 \pm 0.25$ & $3.073 \pm 0.03$ & $79.58 \pm 0.35$ \\
\midrule
\texttt{Subgraphormer + PE} & $30k$ & $64.62 \pm 0.16$ & $3.200 \pm 0.1$ & $80.38 \pm 1.92$ \\
\bottomrule
\end{tabular}
}
\end{table}

\subsection{HyperParameters}
\label{app: HyperParameters}
In this section, we detail the hyperparameter search conducted for our experiments. We use the same hyperparameter grid for both \texttt{Subgraphormer} and \texttt{Subgraphormer + PE}. 
The hyperparameter search configurations for the full bag and stochastic bag settings are presented in \Cref{tab: Hyperparameters search for full bag setting,tab: Hyperparameters search for stochastic setting}, respectively. Additionaly, we report the final selected hyperparameters for \texttt{Subgraphormer + PE} in both settings in \Cref{tab: Hyperparameters for full bag setting,tab: Hyperparameters for stochastic setting}. Notably, in the stochastic bag setting with sampling ratios of 20\% and 5\% over the \textsc{Zinc-12k} dataset, our model failed to converge, leading us to extend the training to 800 epochs.

\textbf{Optimizers and Schedulers.} For the \textsc{Zinc-12k} and \textsc{Zinc-Full} datasets, we employ the Adam optimizer paired with a ReduceLROnPlateau scheduler (factor set to 0.5, patience at 20, and a minimum learning rate of 0). A similar setup was used for the \textsc{Alchemy-12k} dataset, except the minimum learning rate which was set to $1 \times 10^{-7}$. For the \textsc{ogbg-molhiv} dataset, we utilized the ASAM optimizer~\citep{kwon2021asam} without a scheduler. For both \textsc{ogbg-molesol} and \textsc{ogbg-molbace}, we employed a constant learning rate without any scheduler. Lastly, for the \textsc{Peptides-func} and \textsc{Peptides-struct} datasets, the AdamW optimizer was chosen in conjunction with a cosine annealing scheduler, incorporating 10 warmup epochs.

\begin{table}[ht]
    \centering
  \caption{Hyperparameters search for \texttt{Subgraphormer} and \texttt{Subgraphormer + PE} in full bag settings.}
    \label{tab: Hyperparameters search for full bag setting}
    \resizebox{1\textwidth}{!}{%
    \begin{tabular}{l|c|c|c|c|c|c|c} 
    \toprule
        Dataset & Num. layers & Learning rate & Embedding size & Epochs & Batch size & Dropout & Num. Eigenvectors \\  
        \midrule 
        \textsc{Zinc-12k} & $\{ 6 \}$ & $\{ 0.001, 0.0005, 0.0003, 0.0001 \}$ & $ \{ 96 \}$ & $ \{ 400 \}$ & $\{ 128 \} $ & $\{ 0 \}$ & $\{ 0, 1, 2, 8, 16 \}$ \\
        \textsc{Zinc-full} & $\{ 6 \}$ & $\{ 0.001, 0.0005, 0.0003, 0.0001 \}$ &  $ \{ 96 \}$ & $ \{ 400 \}$ & $\{ 128 \} $ & $\{ 0 \}$  & $\{ 0, 1, 2, 8, 16 \}$ \\
        \textsc{molhiv} & $\{ 2,3 \}$ & $\{ 0.1, 0.01, 0.001 \}$ & $\{ 60 \}$ & $\{ 100 \}$ & $\{ 32 \}$ & $\{ 0.3, 0.5 \}$ & $\{ 0, 1, 2, 8, 16 \}$ \\
        \textsc{molsol} & $\{ 2, 3 \}$ & $\{ 0.1, 0.01, 0.001 \}$ & $\{ 60 \}$ & $\{ 100 \}$ & $\{ 32 \}$ & $\{ 0.3, 0.5 \}$ & $\{ 0, 1, 2, 8, 16 \}$\\
        \textsc{molbace} & $\{2, 3\}$ & $\{ 0.1, 0.01, 0.001 \}$ & $\{ 60 \}$ & $\{ 100 \}$ & $\{ 32 \}$ & $\{ 0.3 \}$ & $\{ 0, 1, 2, 8, 16 \}$ \\
        \textsc{Alchemy-12k} & $\{ 5 \}$ & $\{ 0.01, 0.001\}$ & $\{ 96 \}$ & $\{ 400 \}$ & $\{ 128 \}$ & $\{ 0 \} $ & $\{ 0, 1, 2, 8, 16 \}$ \\
        \bottomrule
    \end{tabular}
    }
\end{table}

\begin{table}[ht]
    \centering
  \caption{Best hyperparameters for \texttt{Subgraphormer + PE} in full bag settings.}
    \label{tab: Hyperparameters for full bag setting}
    \resizebox{1\textwidth}{!}{%
    \begin{tabular}{l|c|c|c|c|c|c|c} 
    \toprule
        Dataset & Num. layers & Learning rate & Embedding size & Epochs & Batch size & Dropout & Num. Eigenvectors \\  
        \midrule 
        \textsc{Zinc-12k} & $6$ & $0.0005$ & $96$ & $400$ & $128$ & $0$ & $8$ \\
        \textsc{Zinc-full} & $6$ & $0.0005$ & $96$ & $400$ & $128$ & $0$ & $16$ \\
        \textsc{molhiv} & $2$ & $0.1$ & $60$ & $100$ & $32$ & $0.3$ & $8$ \\
        \textsc{molsol} & $3$ & $0.001$ & $60$ & $100$ & $32$ & $0.5$ & $2$ \\
        \textsc{molbace} & $3$ & $0.001$ & $60$ & $100$ & $32$ & $0.3$ & $16$ \\
        \textsc{Alchemy-12k} & $5$ & $0.01$ & $96$ & $400$ & $128$ & $0$ & $16$ \\
        \bottomrule
    \end{tabular}
    }
\end{table}

\begin{table}[ht]
    \centering
  \caption{Hyperparameters search for \texttt{Subgraphormer} and \texttt{Subgraphormer + PE} in  stochastic sampling settings.}
    \label{tab: Hyperparameters search for stochastic setting}
    \resizebox{1\textwidth}{!}{%
    \begin{tabular}{lr|c|c|c|c|c|c|c} 
    \toprule
        Dataset & Sampling Ratio & Num. layers & Learning rate & Embedding size & Epochs & Batch size & Dropout & Num. Eigenvectors \\  
        \midrule 
        \textsc{Zinc-12k} & 50\% & $\{ 6 \}$ &  $\{ 0.001, 0.0005, 0.0003, 0.0001 \}$  & $ \{ 96 \}$ &$ \{ 400 \}$& $\{ 128 \} $ & $\{ 0 \}$  & $\{ 0, 1, 2, 8, 16 \}$ \\
        \textsc{molhiv} & 50\% & $\{ 2,3 \}$ & $\{ 0.1, 0.01, 0.001 \}$ & $\{ 60 \}$ & $\{ 100 \}$ & $\{ 32 \}$ & $\{ 0.3, 0.5 \}$ & $\{ 0, 1, 2, 8, 16 \}$\\
        \textsc{molsol} & 50\% &  $\{3, 6, 12 \}$ & $\{ 0.1, 0.01, 0.001 \}$ & $\{ 60 \}$ & $\{ 100 \}$ & $\{ 32 \}$ & $\{ 0.3, 0.5 \}$ & $\{ 0, 1, 2, 8, 16 \}$ \\
        \textsc{molbace} & 50\% & $\{2, 3\}$ & $\{ 0.1, 0.01, 0.001 \}$ & $\{ 60 \}$ & $\{ 100 \}$ & $\{ 32 \}$ & $\{ 0.3 \}$ & $\{ 0, 1, 2, 8, 16 \}$  \\
        \specialrule{.2em}{.1em}{.1em} 
        \textsc{Peptides-func} & 30\% & $\{ 5 \}$ & $\{ 0.005, 0.003, 0.001 \} $ & $\{ 96 \}$ & $\{ 200 \}$ & $\{ 128 \}$ & $\{ 0 \}$ & $\{ 0, 1, 2, 8, 16 \}$  \\
        \textsc{Peptides-struct} & 30\% & $\{ 4 \}$ & $\{ 0.005, 0.003, 0.001 \} $ & $\{ 96 \}$ & $\{ 200 \}$ & $\{ 128 \}$ & $\{ 0 \}$ & $\{ 0, 1, 2, 8, 16 \}$ \\
        \specialrule{.2em}{.1em}{.1em} 
        \textsc{Zinc-12k} & 20\% & $\{ 6 \}$ &  $\{ 0.001, 0.0005, 0.0003, 0.0001 \}$  &  $ \{ 96 \}$  &$ \{ 400, 800 \}$& $\{ 128 \} $ & $\{ 0 \}$  & $\{ 0, 1, 2, 8, 16 \}$ \\
        \textsc{molhiv} & 20\% &$\{ 2,3 \}$ & $\{ 0.1, 0.01, 0.001 \}$ & $\{ 60 \}$ & $\{ 100 \}$ & $\{ 32 \}$ & $\{ 0.3, 0.5 \}$ & $\{ 0, 1, 2, 8, 16 \}$ \\
        \textsc{molsol} & 20\% &  $\{ 3, 6, 12 \}$ & $\{ 0.1, 0.01, 0.001 \}$ & $\{ 60 \}$ & $\{ 100 \}$ & $\{ 32 \}$ & $\{ 0.3, 0.5 \}$ & $\{ 0, 1, 2, 8, 16 \}$ \\
        \textsc{molbace} & 20\% & $\{2, 3\}$ & $\{ 0.1, 0.01, 0.001 \}$ & $\{ 60 \}$ & $\{ 100 \}$ & $\{ 32 \}$ & $\{ 0.3 \}$ & $\{ 0, 1, 2, 8, 16 \}$ \\
        \specialrule{.2em}{.1em}{.1em} 
        \textsc{Zinc-12k} & 5\% & $\{ 6 \}$ &  $\{ 0.001, 0.0005, 0.0003, 0.0001 \}$  &  $ \{ 96 \}$  & $ \{ 400, 800 \}$& $\{ 128 \} $ & $\{ 0 \}$  & $\{ 0, 1, 2, 8, 16 \}$ \\
        \textsc{molhiv} & 5\% & $\{ 2,3 \}$ & $\{ 0.1, 0.01, 0.001 \}$ & $\{ 60 \}$ & $\{ 100 \}$ & $\{ 32 \}$ & $\{ 0.3, 0.5 \}$ & $\{ 0, 1, 2, 8, 16 \}$\\
        \textsc{molsol} & 5\% &  $\{ 3, 6, 12 \}$ & $\{ 0.1, 0.01, 0.001 \}$ & $\{ 60 \}$ & $\{ 100 \}$ & $\{ 32 \}$ & $\{ 0.3, 0.5 \}$ & $\{ 0, 1, 2, 8, 16 \}$\\
        \textsc{molbace} & 5\% & $\{2, 3\}$ & $\{ 0.1, 0.01, 0.001 \}$ & $\{ 60 \}$ & $\{ 100 \}$ & $\{ 32 \}$ & $\{ 0.3 \}$ & $\{ 0, 1, 2, 8, 16 \}$  \\
        \bottomrule
    \end{tabular}
    }
\end{table}

\begin{table}[ht]
    \centering
  \caption{Best hyperparameters for \texttt{Subgraphormer + PE} in stochastic sampling settings.}
    \label{tab: Hyperparameters for stochastic setting}
    \resizebox{1\textwidth}{!}{%
    \begin{tabular}{lr|c|c|c|c|c|c|c} 
    \toprule
        Dataset & Sampling Ratio & Num. layers & Learning rate & Embedding size & Epochs & Batch size & Dropout & Num. Eigenvectors \\  
        \midrule 
        \textsc{Zinc-12k} & 50\% & $6$ & $0.0005$ & $96$ & $400$ & $128$ & $0$ & $8$ \\
        \textsc{molhiv} & 50\% & $2$ & $0.1$ & $60$ & $100$ & $32$ & $0.3$ & $8$ \\
        \textsc{molsol} & 50\% & $3$ & $0.001$ & $60$ & $100$ & $32$ & $0.5$ & $2$ \\
        \textsc{molbace} & 50\% & $3$ & $0.001$ & $60$ & $100$ & $32$ & $0.3$ & $8$ \\
        \specialrule{.2em}{.1em}{.1em} 
        \textsc{Peptides-func} & 30\% & $5$ & $0.003$ & $96$ & $200$ & $128$ & $0$ & $16$ \\
        \textsc{Peptides-struct} & 30\% & $4$ & $0.005$ & $96$ & $200$ & $128$ & $0$ & $16$ \\
        \specialrule{.2em}{.1em}{.1em} 
        \textsc{Zinc-12k} & 20\% & $6$ & $0.0005$ & $96$ & $800$ & $128$ & $0$ & $8$ \\
        \textsc{molhiv} & 20\% & $2$ & $0.1$ & $60$ & $100$ & $32$ & $0.3$ & $8$ \\
        \textsc{molsol} & 20\% & $12$ & $0.001$ & $60$ & $100$ & $32$ & $0.5$ & $1$ \\
        \textsc{molbace} & 20\% & $3$ & $0.001$ & $60$ & $100$ & $32$ & $0.3$ & $8$ \\
        \specialrule{.2em}{.1em}{.1em} 
        \textsc{Zinc-12k} & 5\% & $6$ & $0.0005$ & $96$ & $800$ & $128$ & $0$ & $8$ \\
        \textsc{molhiv} & 5\% & $2$ & $0.1$ & $60$ & $100$ & $32$ & $0.3$ & $8$ \\
        \textsc{molsol} & 5\% & $12$ & $0.001$ & $60$ & $100$ & $32$ & $0.5$ & $2$ \\
        \textsc{molbace} & 5\% & $3$ & $0.001$ & $60$ & $100$ & $32$ & $0.3$ & $2$ \\
        \bottomrule
    \end{tabular}
    }
\end{table}

\section{Complexity}
\label{app: Complexity}

This section provides an analysis of the computational complexity associated with each aggregation method we presented, as illustrated in~\Cref{fig: A_G,fig: A_G^S,fig: A_point}. Consider a graph $G = (A, X)$; with notes and edges denoted as $V$, $E$, respectively. Our product graph always encompasses $\vert V \vert^2$ nodes, and the number of edges is different for each considered aggregation as follows,

\begin{enumerate}
\item \textbf{Internal:} For $\mathcal{A}_G$, the number of edges is computed by $\sum_{s \in V} \sum_{v \in V} \big( 1 + d(v) \big)$, where $d(v)$ denotes the degree of vertex $v$ (in the original graph $G$). This simplifies to $\vert V \vert \cdot ( \vert V \vert + \vert E \vert ) = \vert V \vert^2 + \vert V \vert \cdot \vert E \vert $.
\item \textbf{External:} Similarly, for $\mathcal{A}_{G^S}$, the number of edges is given by $\sum_{v \in V} \sum_{s \in V} \big( 1 + d(s) \big)$, which also simplifies to $\vert V \vert \cdot ( \vert V \vert + \vert E \vert ) = \vert V \vert^2 + \vert V \vert \cdot \vert E \vert$.
\item \textbf{Point:} For $\mathcal{A}_\text{Point}$, the number of edges is $\vert V \vert^2$.
\end{enumerate}

The complexities are summarized in~\Cref{tab: number of edges}.
\begin{table}[t]
    \centering
  \caption{Complexity analysis of our aggregations.}
    \label{tab: number of edges}
    % \resizebox{0.7\textwidth}{!}{%
    \begin{tabular}{l|l} 
    \toprule
        Aggregation Type & Number of edges \\  
        \midrule 
        $\mcalA_G$ & $\mathcal{O} (\vert V \vert^2 + \vert V \vert \cdot \vert E \vert)$ \\
        $\mcalA_{G^S}$ & $\mathcal{O} (\vert V \vert^2 + \vert V \vert \cdot \vert E \vert)$  \\
        $\mcalA_\text{Point}$ & $\mathcal{O} ( \vert V \vert^2 )$ \\
        % \midrule
        % \emph{product graph PE} ($k$) & $\mathcal{O}(k \cdot  \vert V \vert^2) $ \\
        \bottomrule
    \end{tabular}
    % }
\end{table}

\textbf{\texttt{Subgraphormer} Complexity.} We utilized the GAT  convolution for our attention-based aggregations, as introduced by~\citet{velivckovic2017graph}.
Given a graph $G' = (A', X')$ , with nodes and edges denoted as $V'$, $E'$, respectively, the time complexity for a single GAT attention head\footnote{The feature dimension is considered a constant.} is $\mathcal{O}(\vert V' \vert +  \vert E' \vert)$. Thus, the overall complexity depends on the number of nodes $\vert V' \vert $ and the number of edges $\vert E' \vert$ in the graph. Specifically, the product graph we process always has $\vert V \vert^2$ nodes, while the number of edges are given in \Cref{tab: number of edges}. Thus, SAB exhibits the complexity of $\mathcal{O}(\vert V \vert^2 + \vert V \vert \cdot  \vert E \vert)$.

As elaborated in \Cref{sec: Subgraph Positional Encoding}, the computational complexity associated with calculating the product graph PE, which involves $k$ eigenvectors, is expressed as $\mathcal{O}(k \cdot \vert V \vert^2)$.

Therefore, the total complexity of \texttt{Subgraphormer} is $\mathcal{O}(k \cdot \vert V \vert^2 + \vert V\vert \cdot \vert E \vert)$.

\textbf{Run-time in Seconds.} The run-time performance was evaluated on the \textsc{Zinc-12k} dataset, where the average run-time, calculated over three different seeds, was 2 hours, 5 minutes, and 24 seconds, with a standard deviation of $\pm$29 seconds. It is important to highlight that the \emph{product graph PE} computation can be executed as a preprocessing step. Notably, this process required less than 10 minutes for the \textsc{Zinc-12k} dataset.

\section{Proofs}
\label{app: proofs}

\SubgraphGNNsasMPNNs*
\begin{proof}
    We consider an unnormalized variant of RGCN, defined as:
\begin{align}
    \mcalX^{t+1} &= \text{RGCN}^t\Big(\mcalX^t, \{ \mcalA_i \}_{i=1}^M \Big), \\
    \text{RGCN}^t\Big(\mcalX^t, \{ \mcalA_i \}_{i=1}^M \Big) &= \mcalX^t \mathbf{W}_0^t + \sum_{i=1}^M \mcalA_i \mcalX^t \mathbf{W}_i^t.
\end{align}
    We assume the input features at the first layer, $\mcalX^t$, are provided as 1-hot vectors, and that the parameterized function $f^t$, which is applied in \Cref{eq:bohang} also outputs 1-hot vectors.

    We aim to show that any given layer $t$ of a GNN-SSWL$+$, as per \Cref{eq:bohang}, can be implemented by this RGCN model. 
    %More specifically, the RGCN model will be responsible for encoding the input, and the \texttt{MLP}, will implement the parameterized function $f$.
%
    The proof involves the following steps:
    \begin{enumerate}
        \item Encode uniquely each of the inputs (individually) to the parameterized function $f^t$ in \Cref{eq:bohang},
        \begin{equation}
            \mcalX(s,v)^{t}, \mcalX(v,v)^{t}, \{\mcalX(s,v')^{t}\}_{v'\sim_G v}, \{\mcalX(s',v)^{t}\}_{s'\sim_G s}.
        \end{equation}
        \item Encode uniquely the input of $f^t$ as a whole,
        \begin{equation}
            \big( \mcalX(s,v)^{t}, \mcalX(v,v)^{t}, \{\mcalX(s,v')^{t}\}_{v'\sim_G v}, \{\mcalX(s',v)^{t}\}_{s'\sim_G s} \big).
        \end{equation}
        \item Implement the parameterized function $f^t$.
        \item Ensure the output features are 1-hot features -- to maintain generality across layers.
    \end{enumerate}

\textbf{Step 1:} The unique encoding of the inputs can be achieved as follows:
\begin{align}
    \mcalX(s,v)^{t} &\rightarrow \mcalX(s,v)^{t}, \label{eq: sv}\\
        \mcalX(v,v)^{t} &\rightarrow (\mcalA_{\text{point}} \mcalX^t)~(s,v) \label{eq: vv}, \\
    \{\mcalX(s,v')^{t}\}_{v'\sim_G v} &\rightarrow (\mcalA_G \mcalX^t)~(s,v), \label{eq: sv v}\\
    \{\mcalX(s',v)^{t}\}_{s'\sim_G s} &\rightarrow (\mcalA_{G^S} \mcalX^t)~(s,v), \label{eq: sv s}
\end{align}
For \Cref{eq: sv v,eq: sv s}, the right-hand sides effectively count the number of each element in the set, since the input features are 1-hot vectors. For \Cref{eq: sv,eq: vv}, the right hand side is actually equal to the left hand side.

\textbf{Step 2:} Assuming that $\mcalX \in \mathbb{R}^{n^2 \times d} $, the RGCN that uniquely encodes the input of $f^t$ as a whole is defined as follows,
\begin{align}
        &\text{RGCN}^t\Big(\mcalX^t, \{ \mcalA_G, \mcalA_{G^S},\mcalA_{\text{point}}\} \Big) \\
        &= \mcalX^t \mathbf{W}_0^t + \mcalA_\text{point} \mcalX^t \mathbf{W}_\text{point}^t + \mcalA_G \mcalX^t \mathbf{W}_G^t + \mcalA_{G^S} \mcalX^t \mathbf{W}_{G^S}^t , \nonumber
\end{align}
where the weight matrices are defined as:
\begin{align}
\mathbf{W}_0^t &= \begin{pmatrix}
I_{d \times d} & 0_{d \times d} & 0_{d \times d} & 0_{d \times d}
\end{pmatrix}, \\
\mathbf{W}_{\text{point}}^t &= \begin{pmatrix}
0_{d \times d} & I_{d \times d} & 0_{d \times d} & 0_{d \times d}
\end{pmatrix}, \\
\mathbf{W}_G^t &= \begin{pmatrix}
0_{d \times d}  & 0_{d \times d} & I_{d \times d} & 0_{d \times d}
\end{pmatrix}, \\
\mathbf{W}_{G^S}^t &= \begin{pmatrix}
0_{d \times d} & 0_{d \times d} & 0_{d \times d} & I_{d \times d}
\end{pmatrix},
\end{align}
More specifically, it holds that,
\begin{align}
\mcalX^t \mathbf{W}_0^t &= \begin{pmatrix}
\mcalX^t & 0_{d \times d} & 0_{d \times d} & 0_{d \times d}
\end{pmatrix}, \\
\mcalA_\text{point} \mcalX^t \mathbf{W}_{\text{point}}^t &= \begin{pmatrix}
0_{d \times d} & \mcalA_\text{point} \mcalX^t & 0_{d \times d} & 0_{d \times d}.
\end{pmatrix}, \\
\mcalA_G \mcalX^t \mathbf{W}_G^t &= \begin{pmatrix}
0_{d \times d} & 0_{d \times d} & \mcalA_G \mcalX^t & 0_{d \times d}
\end{pmatrix}, \\
\mcalA_{G^S} \mcalX^t \mathbf{W}_{G^S}^t &= \begin{pmatrix}
0_{d \times d} & 0_{d \times d} & 0_{d \times d} & \mcalA_{G^S} \mcalX^t
\end{pmatrix}.
\end{align}
Thus, since RGCN is the sum of those, it holds that,
\begin{align}
    &\text{RGCN}^t\Big(\mcalX^t, \{ \mcalA_G, \mcalA_{G^S},\mcalA_{\text{point}}\} \Big) \nonumber \\
    &= \begin{pmatrix}
\mcalX^t  & \mcalA_\text{point} \mcalX^t  & \mcalA_G \mcalX^t & \mcalA_{G^S} \mcalX^t \nonumber
\end{pmatrix}.
\end{align}
This concatenation is a unique encoding of the input for $f^t$.

\textbf{Step 3:} To prove this step, we start with the following theorem  from \citet{yun2019small} about the memorization power of ReLU networks:

\begin{restatable}{theorem}{}
\label{thm: memorization}
Consider a dataset $\{ x_i, y_i \}_{i=1}^N \in \mathbb{R}^d \times \mathbb{R}^{d_y}$, with each $x_i$ being distinct and every $y_i \in \{ 0, 1 \}^{d_y}$, for all $i$. There exists a 4-layer fully connected ReLU neural network $f_\theta: \mathbb{R}^d \rightarrow \mathbb{R}^{d_y}$ that perfectly maps each $x_i$ to its corresponding $y_i$, i.e., $f_\theta(x_i) = y_i$ for all $i$.
\end{restatable}

Employing this theorem, we define the $x_i$s as all possible (distinct) rows derived from the output of Step 2, represented as:
\begin{align}
    \begin{pmatrix}
        \mcalX^t  & \mcalA_\text{point} \mcalX^t  & \mcalA_G \mcalX^t & \mcalA_{G^S} \mcalX^t \nonumber
    \end{pmatrix},
\end{align}
with each corresponding $y_i$ being the output of the parameterized function $f^t$ over the original input that is represented by this $x_i$. Given the finite nature of our graph set, the input set is also finite, i.e., $N$ is finite.

Hence, in light of \Cref{thm: memorization}, a 4-layer fully connected ReLU neural network capable of implementing $f^t$ exists. This network can be equivalently realized using four layers of RGCN. Specifically, by setting $\{\mathbf{W}^{t}_i \}_{{i=1}^M}$ to zero and utilizing only $\mathbf{W}^{t}_0$ for $t \in \{ 0, 1, 2, 3\}$, we effectively mimic the operation of this 4-layer fully connected ReLU network.

\textbf{Step 4:} Finally, employing this same logic, we can again use \Cref{thm: memorization} and map those outputs of step 3 back to 1-hot vectors.
\end{proof}

%\clearpage

\AdjacenciesAsProducts*
\begin{proof}
    Without loss of generality, consider the matrix $(I \otimes A)$ indexed at $\big((s, v), (s', v')\big)$, we have,
    \begin{equation}
        (I \otimes A) \big((s, v), (s', v')\big) \triangleq I(s, s') \cdot A(v, v').
    \end{equation}
    In this context, edges exist if and only if $s=s'$ and $v, v'$ are neighbors in the original graph $G$. This is in direct correspondence with the definition of $\mathcal{A}_G$, as outlined in~\Cref{eq: A_G}. The proof for $A \otimes I = \mathcal{A}_{G^S}$ follows a similar line of reasoning.
\end{proof}

%\clearpage

\SubgraphormerPE*
\begin{proof}
    By definition,
    \begin{equation}
    \label{eq: L_GG}
        \mathcal{L}_{G \square G} = \mathcal{D}_{G \square G} - \mcalA_{G \square G},
    \end{equation}
    where $\mathcal{D}_{G \square G}$ is defined as
    \begin{equation}
        \mathcal{D}_{G \square G} \triangleq \diag(\mcalA_{G \square G}\vec{1}_{n^2}),
    \end{equation}
    and recalling~\Cref{eq: A_GG},
    \begin{equation}
       \mathcal{D}_{G \square G} = \diag\Big( (A \otimes I + I \otimes A) \vec{1}_{n^2} \Big).
    \end{equation}
    
    Using tensor product rules, we have,
    \begin{align}
        \mathcal{D}_{G \square G} &\triangleq  \diag\Big(A \otimes I + I \otimes A)\vec{1}_{n^2}\Big) \nonumber \\
        &= \diag\Big(A \otimes I + I \otimes A)(\vec{1}_{n}\otimes \vec{1}_{n})\Big) \nonumber \\
        &= \diag\Big( A\vec{1}_{n} \otimes I\vec{1}_{n} + I\vec{1}_{n} \otimes A\vec{1}_{n} \Big) \nonumber \\
        &= D \otimes I + I \otimes D \label{eq: D_GG},
    \end{align}
    where $D = A\vec{1}_n$.
    
    Substituting~\Cref{eq: D_GG} and~\Cref{eq: A_GG} in~\Cref{eq: L_GG} we have,
        \begin{equation}
        \mathcal{L}_{G \square G} = D \otimes I + I \otimes D - A \otimes I + I \otimes A.
    \end{equation}
    Therefore, 
    \begin{align}
        \mathcal{L}_{G \square G} &= D-A \otimes I + I \otimes D-A \nonumber \\
        &= L \otimes I + I \otimes L.
    \end{align}
    Thus, the eigenspace of $\mathcal{L}_{G \square G}$ is given by $\{(v_i \otimes v_j, \lambda_i + \lambda_j) \}_{i,j = 1}^{n^2}$.
\end{proof}

%\clearpage

\ComplexityPE*
\begin{proof}
The proof stems directly from the arguments established in \Cref{prop: main prop}, which states that the computation of the eigendecomposition of $\mathcal{L}_{G \square G}$ primarily involves diagonalizing the Laplacian matrix of the original graph $G$, which incurs a computational cost of $\mathcal{O}(n^3)$. Additionally, to obtain the complete set of $n^2$ eigenvectors for the graph $G \square G$, an extra $\mathcal{O}(n^4)$ operations are necessary, resulting in a total complexity of $\mathcal{O}(n^4)$.

It is important to note that when the objective is to compute only $k$ eigenvectors, where $k \leq n$, the computational complexity is reduced to $\mathcal{O}(k \cdot n^2)$, since in this case we only need to obtain $k$ eigenvectors. This is congruent with the complexity involved in calculating the same number of eigenvectors for the original graph $G$, which is also $\mathcal{O}(k \cdot n^2)$, for example by applying the Lanczos Algorithm~\citep{lanczos1950iteration} or the Implicitly Restarted Arnoldi Methods~\citep{lehoucq1998arpack,lee2009k}, as done by widely used libraries such as Scipy~\citep{2020SciPy-NMeth}. Therefore, our approach offers a computational advantage in scenarios where a subset of eigenvectors ($k \leq n$) suffices; which consist of most cases in practice.
\end{proof}

%\clearpage

\AdjacencyOfGK*
\begin{proof}
    The proof proceeds by induction.
    
    \textbf{Base Case:}
        For $k=2$, we must show that $\mathcal{A}_{G^{\square^2}} \triangleq \mathcal{C}^2(A)$.
        
        From \Cref{cor: adjacency of cartesian product}, we have:
        \begin{equation}
        \label{eq: base A_G2}
            \mathcal{A}_{G^{\square^2}} \triangleq \mathcal{A}_{G \square G} = A \otimes I_n + I_n \otimes A.
        \end{equation}

        Using \Cref{eq: Recursive Cartesian} and \Cref{eq: Cartesian Base Case}, we get:
        \begin{equation}
        \label{eq: base C2}
            \mathcal{C}^2(A) = \mcalC^1(A) \otimes I_n + I_n \otimes A = A \otimes I_n + I_n \otimes A.    
        \end{equation}
        Thus, from \Cref{eq: base A_G2} and \Cref{eq: base C2} we obtain,
        \begin{equation}
            \mathcal{A}_{G^{\square^2}} = \mathcal{C}^2(A), 
        \end{equation}
        this verifies the base case.
        
    \textbf{Inductive Assumption:}
    Assume the proposition holds for some integer $k$, i.e.,
    \begin{equation}
    \label{eq: Inductive Assumption}
    \mathcal{A}_{G^{\square^k}} \triangleq \mathcal{C}^k(A),
    \end{equation}
    then we must demonstrate its validity for $k+1$:
    \begin{equation}
        \mathcal{A}_{G^{\square^{k+1}}} \triangleq \mathcal{C}^{k+1}(A). 
    \end{equation}
    \textbf{Inductive Step:} 
    Let $G_1 = G^{\square^{k}}$ and $G_2 = G$. Given that $G$ has $n$ nodes, $G_1$ has $n^k$ nodes, and $G_2$ has $n$ nodes. Recalling \Cref{eq: A G1 G2}, we deduce:
    \begin{equation}
    \label{eq: Inductive Step Application}
        \mathcal{A}_{G^{\square^{k+1}}} = \mathcal{A}_{G^{\square^{k}} \square G} = \mathcal{A}_{G_1 \square G_2} =
        \mathcal{A}_{G^{\square^{k}}} \otimes I_n + I_{n^k} \otimes A,
    \end{equation}
    where $\mathcal{A}_{G^{\square^{k}}}$ is the adjacency of $G_1$ and $A$ is the adjacency of $G_2$.

    Applying the inductive assumption from \Cref{eq: Inductive Assumption} to \Cref{eq: Inductive Step Application}, we obtain:
    \begin{equation}
    \label{eq: Final Inductive Step}
        \mathcal{A}_{G_1 \square G_2} = \mathcal{C}^{k}(A) \otimes I_n + I_{n^k} \otimes A.
    \end{equation}
    The right-hand side of \Cref{eq: Final Inductive Step} corresponds exactly to the definition of $\mathcal{C}^{k+1}$ as specified in \Cref{eq: Recursive Cartesian}.

    Hence, we conclude that:
    \begin{equation}
        \mathcal{A}_{G^{\square^{k+1}}} = \mathcal{A}_{G^{\square^{k}} \square G} = \mathcal{A}_{G_1 \square G_2} = \mathcal{C}^{k+1}(A).
    \end{equation}
    This concludes the proof.
\end{proof}

%\clearpage

\AdjacencyOfGKClosedForm*
\begin{proof}
    We will prove by induction.
    
    \textbf{Base Case:} The base case, for $\mcalK=2$, is obtained straight-forwardly from the definition of $\mcalA^k_{\mcalK}$, recalling \Cref{eq: A^k_K}. Formally, by the definition of $\mcalC^k$ -- \Cref{def: Cartesian operator}, we have,
    \begin{equation}
        \mcalC^2(A) = A \otimes I_n + I_n \otimes A.
    \end{equation}
    We note that,
    \begin{align}
        \mcalA^0_2 &= A \otimes I_n, \\
        \mcalA^1_2 &= I_n \otimes A,
    \end{align}
thus, it holds that,
\begin{align}
    \mcalC^2(A) &= A \otimes I_n + I_n \otimes A \nonumber \\
    &= \mcalA^0_2 + \mcalA^1_2 \nonumber \\
    &= \sum_{k=0}^1 \mcalA^k_2,
\end{align}
which completes the base case.

    \textbf{Inductive Assumption:} Assume the proposition holds for some integer $\mcalK$, i.e.,
    \begin{equation}
    \label{eq: Inductive Assumption 2}
    \mcalC^{\mcalK}(A) =\sum_{k = 0}^{\mcalK-1} \mcalA^k_{\mcalK},
    \end{equation}
    then we must demonstrate its validity for $\mcalK+1$:
    \begin{equation}
    \mcalC^{\mcalK+1}(A) =\sum_{k = 0}^{\mcalK} \mcalA^k_{\mcalK+1}.
    \end{equation}

    \textbf{Inductive Step:} By the definition of the Cartesian operator -- \Cref{def: Cartesian operator}, it holds that,
    \begin{equation}
    \label{eq: C^K+1 proof}
        \mcalC^{\mcalK+1}(A) = \mathcal{C}^{\mcalK}(A) \otimes I_n + I_{n^\mcalK} \otimes A.
    \end{equation}
    Substituting \Cref{eq: Inductive Assumption 2} to \Cref{eq: C^K+1 proof} we obtain,
    \begin{equation}
        \mcalC^{\mcalK+1}(A) = \Big( \sum_{k = 0}^{\mcalK-1} \mcalA^k_{\mcalK} \Big) \otimes I_n + I_{n^\mcalK} \otimes A.
    \end{equation}
    Since the tensor product operator, $\otimes$, is linear, and recalling the definition of \Cref{eq: A^k_K} it holds that,
    \begin{align}
        \mathcal{C}^{\mathcal{K}+1}(A) &= \Big( \sum_{k = 0}^{\mathcal{K}-1} \overbrace{ \mathcal{A}^k_{\mathcal{K}}  \otimes I_n}^{ \triangleq \mathcal{A}^k_{\mathcal{K}+1}} \Big) + \overbrace{ \triangleq I_{n^\mathcal{K}} \otimes A }^{ \triangleq \mathcal{A}^\mathcal{K}_{\mathcal{K}+1}} \nonumber \\
        &= \sum_{k = 0}^{\mathcal{K}-1} \mathcal{A}^k_{\mathcal{K}+1} + \mathcal{A}^\mathcal{K}_{\mathcal{K}+1} \nonumber \\
        &= \sum_{k = 0}^{\mathcal{K}} \mathcal{A}^k_{\mathcal{K}+1}.
    \end{align}
    This completes the proof.
\end{proof}

%\clearpage

\AdjacencyOfGKIsValid*
\begin{proof}
Since for any $k$  the matrix given by $\mcalA^{k}_{\mcalK}$ is binary, it is sufficient to show that given $k, k'$, such that $k \neq k'$, it holds that,
    \begin{equation}
        \langle \mcalA^{k}_{\mcalK}, \mcalA^{k'}_{\mcalK} \rangle_F = 0,
    \end{equation}
where by $\langle \cdot, \cdot \rangle_F $ we denote the Frobenius inner product.

Recalling \Cref{def: A_k^K}, we have,
\begin{align}
    \langle \mcalA^{k}_{\mcalK}, \mcalA^{k'}_{\mcalK} \rangle_F &= \left\langle I \otimes I \otimes \cdots \otimes A \otimes \cdots \otimes I, I \otimes \cdots \otimes A \otimes \cdots \otimes I \otimes I \right\rangle_F \nonumber \\
    &= \text{Tr} \bigg( (I \otimes I \otimes \cdots \otimes A \otimes \cdots \otimes I ) \cdot (I \otimes \cdots \otimes A \otimes \cdots \otimes  I \otimes I)^T \bigg) \nonumber \\
    &=\text{Tr} \bigg( (I \otimes I \otimes \cdots \otimes A \otimes \cdots \otimes I ) \cdot (I^T \otimes \cdots \otimes A^T \otimes \cdots \otimes  I^T \otimes I^T) \bigg) \nonumber 
\end{align}
Using the the fact that $k \neq k'$, we obtain that,
\begin{align}
    \langle \mcalA^{k}_{\mcalK}, \mcalA^{k'}_{\mcalK} \rangle_F 
    &= \text{Tr} \bigg( II^T \otimes \cdots \otimes II^T \otimes \cdots \otimes II^T \otimes A I^T \otimes II^T \otimes \cdots \otimes II^T \otimes I A^T \otimes II^T \otimes \cdots \otimes II^T \bigg) \nonumber \\
    &=\text{Tr} \big( II^T \big) \cdot \text{Tr} \big( II^T \big) \cdot \ldots \cdot  \text{Tr} \big( A I^T \big) \cdot \text{Tr} \big( II^T \big) \cdot \ldots \cdot \text{Tr} \big( II^T \big) \cdot \text{Tr} \big( I A^T \big) \cdot \text{Tr} \big( II^T \big) \cdot \ldots \cdot \text{Tr} \big( II^T \big). \nonumber 
\end{align}
Since we assume no self loops in the original graph $G$, we know that the diagonal of $A$ and $A^T$ is zero, therefore, the Trace of both $ A I^T$ and $I A^T$ is zero. 
Thus, for any $k \neq k'$, it holds that,
\begin{equation}
    \langle \mcalA^{k}_{\mcalK}, \mcalA^{k'}_{\mcalK} \rangle_F = 0.
\end{equation}
\end{proof}

%\clearpage

\KTupleSubgraphormerPE*
\begin{proof}
 We are looking for the eigenvectors of $L_{G^{\square^k}}$, defined as
 \begin{equation}
 \label{eq: L_G box k}
     L_{G^{\square^k}} = D_{G^{\square^k}} - \mcalA_{G^{\square^k}},
 \end{equation}
 where $D_{G^{\square^k}}$ can be written as
\begin{align}
\label{eq: D_G^k}
    D_{G^{\square^k}} &\triangleq \diag{(\mcalA_{G^{\square^k}}\vec{1}_{n^k})} \nonumber \\
    &= \diag{ \Bigg( \sum_{k' = 0}^{\mcalK-1}  \mcalA^{k'}_{\mcalK} \vec{1}_{n^k} \Bigg) } \nonumber \\
    &= \sum_{k' = 0}^{\mcalK-1} \diag{\Bigg( \mcalA^{k'}_{\mcalK} \vec{1}_{n^k} \Bigg) },
\end{align}
where the last equality is obtained from the linearity of the $\diag$ operator.

Recalling~\Cref{def: A_k^K}, for a given $k'$, it holds that,
\begin{align*}
    \mcalA^{k'}_{\mcalK} \vec{1}_{n^k} &= (I_n \otimes I_n \otimes \ldots \otimes I_n \otimes A \otimes I_n \otimes \ldots \otimes I_n) \cdot ( \vec{1}_n  \otimes \vec{1}_n \otimes \ldots \otimes \vec{1}_n) \\
    &= I_n \vec{1}_n \otimes I_n \vec{1}_n \otimes \ldots I_n \vec{1}_n \otimes A \vec{1}_n \otimes I_n \vec{1}_n \otimes \ldots I_n \vec{1}_n \\
    &= \vec{1}_n \otimes \vec{1}_n \otimes \ldots \vec{1}_n \otimes A \vec{1}_n \otimes \vec{1}_n \otimes \ldots \otimes \vec{1}_n.
\end{align*}
Since the degree matrix $D$ of the original graph $G$ satisfies $D = \diag{A\vec{1}_n} $ we obtain,
\begin{equation}
\label{eq: diag D}
    \diag{ \bigg( \mcalA^{k'}_{\mcalK} \vec{1}_{n^k} \bigg)} = I_n \otimes I_n \otimes \ldots \otimes I_n \otimes D \otimes I_n \otimes I_n \otimes \ldots \otimes I_n.
\end{equation}
For simplicity, we will extend \Cref{def: A_k^K} to any calligraphic matrix; i.e.,
we define,
\begin{equation}
\label{eq: D_K^k}
    \mathcal{D}^{k'}_\mcalK \triangleq I_n \otimes I_n \otimes \ldots \otimes I_n \otimes D \otimes I_n \otimes I_n \otimes \ldots \otimes I_n,
\end{equation}
such that the matrix $D$ occupies the $k'$-th slot.

Substituting \Cref{eq: D_K^k} and \Cref{eq: diag D} to \Cref{eq: D_G^k}, we obtain,

\begin{equation}
\label{eq: D_G^k simplified}
    D_{G^{\square^k}} = \sum_{k' = 0}^{\mcalK-1} \mathcal{D}^{k'}_\mcalK.
\end{equation}

Plugging \Cref{eq: D_G^k simplified} and \Cref{eq: A_G^box k} to \Cref{eq: L_G box k}, 
\begin{equation}
    L_{G^{\square^k}} = \sum_{k' = 0}^{\mcalK-1} \bigg( \mathcal{D}^{k'}_\mcalK - \mathcal{A}^{k'}_{\mathcal{K}} \bigg).
\end{equation}

For a given $k'$, recalling \Cref{def: A_k^K} and \Cref{eq: D_K^k} it holds that,
\begin{align}
    \mathcal{D}^{k'}_\mcalK - \mathcal{A}^{k'}_{\mathcal{K}} &\triangleq  I_n \otimes I_n \otimes \ldots \otimes I_n \otimes D \otimes I_n \otimes I_n \otimes \ldots \otimes I_n - I_n \otimes I_n \otimes \ldots \otimes I_n \otimes A \otimes I_n \otimes I_n \otimes \ldots \otimes I_n \nonumber \\
    &= I_n \otimes I_n \otimes \ldots \otimes I_n \otimes L \otimes I_n \otimes I_n \otimes \ldots \otimes I_n \nonumber \\
    &\triangleq \mathcal{L}^{k'}_{\mathcal{K}}.
\end{align}
Thus, 
\begin{equation}
    L_{G^{\square^k}}  = \sum_{k' = 0}^{\mcalK-1} \mathcal{L}^{k'}_{\mathcal{K}}, 
\end{equation}
which states that the eigenvector-eigenvalue pairs are given by $(v_{i_1} \otimes v_{i_2} \otimes \ldots \otimes v_{i_k}, \lambda_{i_1} + \lambda_{i_2} + \ldots + \lambda_{i_k})$, for any $i_j \in [n]$, where $j \in [k]$; given that $\{(v_i, \lambda_i) \}_{i=1}^{n}$ represents the eigenvectors and eigenvalues of the Laplacian matrix of the original graph $G$.
\end{proof} 
%\clearpage

\KTupleSubgraphormerPEEfficiency*
\begin{proof}
    We begin by diagonalizing the Laplacian matrix of the graph $G$. This process requires $\mathcal{O}(n^3)$ time complexity.

    Following this, we utilize the results in \Cref{prop: gen main prop} to determine the eigenvectors and eigenvalues of the Cartesian power of the graph, denoted as $G^{\square^k}$. The computation of these eigenvectors and their corresponding eigenvalues requires $\mathcal{O}(n^{2k})$ operations.

    Therefore, the overall computational complexity for this procedure is $\mathcal{O}(n^{2k} + n^3)$. However, since $k \geq 2$ the dominating term is $\mathcal{O}(n^{2k})$. Thus, we can conclude that the total complexity is $\mathcal{O}(n^{2k})$.
\end{proof}

\CatApproxProdPE*
\begin{proof}
The full embedding of a node $ v $ in the graph is defined as
\begin{align}
    \mathbf{p}_{v} \triangleq [U_{v,0}, \ldots, U_{v,n-1}].
\end{align}
The function $ F^{\text{prod}} : \mathbb{R}^{2n} \rightarrow \mathbb{R}^{n^2}$ is defined as follows:
\begin{align}
    F^{\text{prod}} (\mathbf{p}_{s} \parallel \mathbf{p}_{v})_i = \texttt{flatten}(\mathbf{p}_{s} \cdot \mathbf{p}_{v}^T).
\end{align}

This function is defined over the compact set $ [-1,1]^{2n} $, and maps to the compact set $ [-1,1]^{n^2} $. Thus is due to the fact that the entries of $ \mathbf{p}_{v'} $ for any node $ v' $ are components of normalized eigenvectors.

Therefore, the function $ F^{\text{prod}} $ is continuous and defined over a compact set, implying that it can be approximated by a \texttt{MLP} via the Universal Approximation Theorem~\cite{hornik1991approximation,cybenko1989approximation} acting on the input $ (\mathbf{p}_{s} \parallel \mathbf{p}_{v}) $. This completes the proof.

\end{proof}

%\clearpage

\section{\texttt{Subgraphormer} Figure}
\label{app: Subgraphormer Figure}
We present a detailed figure illustrating the architecture of our \texttt{Subgraphormer} model -- \Cref{fig: architecture_full}. We begin with the construction of the Product Graph. Subsequently, we apply product graph PE and node-marking, as detailed in Section \ref{sec: Subgraph Positional Encoding}. This step is followed by stacking of $K$ Subgraph Attention Blocks (see Section \ref{sec: Subgraph Attention Block}). Finally, the process concludes with the integration of a pooling layer that returns a graph representation.

\begin{figure*}[t]
    \centering
    \includegraphics[width=\textwidth]{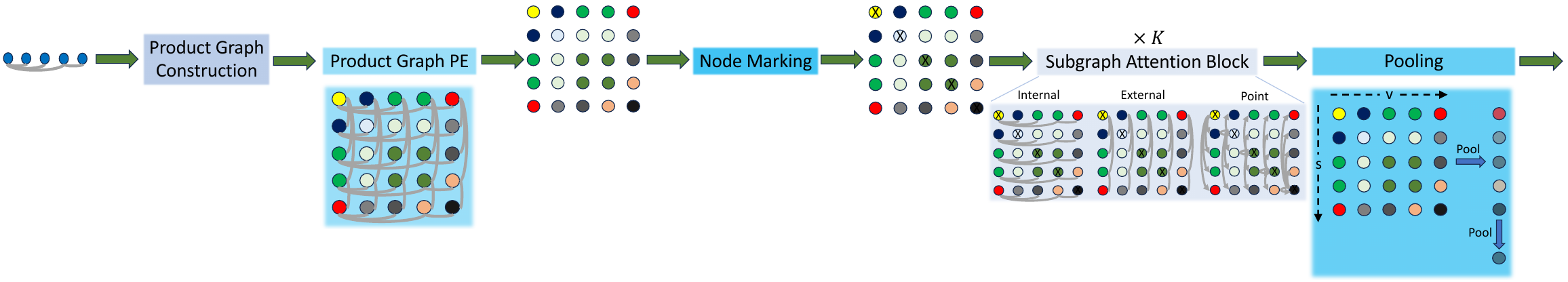}
    \caption{A deep overview of \texttt{Subgraphormer}. Given an input graph, the process begins with the construction of the product graph. This is followed by the computation of product graph PE, illustrated through varying colors on the nodes. Node-Marking is then implemented, depicted by black exes 
    (\begin{tikzpicture} 
    \draw[line width=0.3mm, rotate=45] (0,0.1) -- (0,-0.1);
    \draw[line width=0.3mm, rotate=45] (0.1,0) -- (-0.1,0);
    \end{tikzpicture})
    on diagonal nodes. The process continues with the application of $K$ Subgraph Attention Blocks (SABs), characterized by three distinct connectivities: Internal, External, and Point. The final stage involves the pooling layer, which initially aggregates data across the node dimension to form subgraph representations, and subsequently across the subgraph dimension.}
    \label{fig: architecture_full}
\end{figure*}

\end{document}